%% file: main.tex
\input{preamble}

\begin{document}
\runningtitle{Finite-Step Geometry of Marginal-Preserving Velocity Fields}
\runningauthor{Ewsey Obzherin, Valentin Khrulkov}

\twocolumn[
\aistatstitle{Finite-Step Geometry of Marginal-Preserving\\Velocity Fields in Generative Flows}
\aistatsauthor{Ewsey Obzherin$^{1,2}$, Valentin Khrulkov$^{1}$}
\aistatsaddress{$^{1}$FusionBrain Lab \qquad $^{2}$MIRIAI}
]

\input{sections/00_abstract}
\input{sections/01_problem}
\input{sections/02_response}
\input{sections/03_visibility}
\input{sections/04_reachability}
\input{sections/05_solver}
\input{sections/06_application}

\input{sections/07_positioning}
\input{sections/08_conclusion}

\clearpage
\nobalance
\input{sections/09_ai_use}

\bibliographystyle{plainnat}
\bibliography{references}

\clearpage
\nobalance
\input{sections/09_checklist}
\clearpage
\appendix
\thispagestyle{empty}
\onecolumn
\input{supplement_compact/appendix_layout}
\raggedbottom
\aistatstitle{Finite-Step Geometry of Marginal-Preserving Velocity Fields in Generative Flows: Supplementary Material}
\FloatBarrier
\renewcommand{\thesection}{\AlphAlph{\value{section}}}

% Compact supplement: each definition, proof and common protocol appears once.
\input{supplement_compact/overview}

\input{supplement_compact/theory}

\input{supplement_compact/experiments}

\FloatBarrier
\input{supplement_compact/provenance}

\end{document}

%% file: preamble.tex
\documentclass[twoside]{article}
\usepackage[preprint]{aistats2027}

\usepackage{microtype}
\usepackage{balance}
\usepackage{amsmath,amssymb,amsthm,mathtools,bm}
\usepackage{booktabs}
\usepackage{alphalph}
\usepackage{longtable}
\usepackage{enumitem}
\usepackage{tabularx}
\usepackage{graphicx}
\usepackage{tikz}
\usetikzlibrary{arrows.meta,positioning,calc,decorations.pathreplacing}
\usepackage{pgfplots}
\pgfplotsset{compat=1.18}
\usepackage{placeins}
\usepackage{float}
\usepackage[round]{natbib}

\renewcommand{\refname}{References}

\usepackage{xcolor}
\usepackage{xurl}
\usepackage[hidelinks]{hyperref}
\usepackage{aliascnt}
\usepackage[nameinlink,capitalise]{cleveref}

\crefname{theorem}{Theorem}{Theorems}
\crefname{proposition}{Proposition}{Propositions}
\crefname{lemma}{Lemma}{Lemmas}
\crefname{corollary}{Corollary}{Corollaries}
\crefname{definition}{Definition}{Definitions}

\definecolor{DeepBlue}{RGB}{37,78,122}
\definecolor{DeepOrange}{RGB}{196,92,34}
\definecolor{SoftGray}{RGB}{244,246,248}

\newtheorem{theorem}{Theorem}
\newaliascnt{proposition}{theorem}
\newtheorem{proposition}[proposition]{Proposition}
\aliascntresetthe{proposition}
\newaliascnt{lemma}{theorem}
\newtheorem{lemma}[lemma]{Lemma}
\aliascntresetthe{lemma}
\newaliascnt{corollary}{theorem}
\newtheorem{corollary}[corollary]{Corollary}
\aliascntresetthe{corollary}
\theoremstyle{definition}
\newaliascnt{definition}{theorem}
\newtheorem{definition}[definition]{Definition}
\aliascntresetthe{definition}
\theoremstyle{remark}
\newaliascnt{remark}{theorem}

\aliascntresetthe{remark}

\newcommand{\R}{\mathbb R}
\newcommand{\E}{\mathbb E}
\newcommand{\cS}{\mathcal S}
\newcommand{\cT}{\mathcal T}
\newcommand{\cF}{\mathcal F}

\newcommand{\cB}{\mathcal B}

\newcommand{\Proj}{\Pi}
\newcommand{\dd}{\mathrm d}
\newcommand{\Div}{\operatorname{div}}
\newcommand{\tr}{\operatorname{tr}}

\newcommand{\dist}{\operatorname{dist}}
\newcommand{\Response}{\mathfrak R}
\newcommand{\Reach}{\mathcal M}
\newcommand{\Uad}{\mathcal U_{\rm ad}}
\newcommand{\Id}{\mathrm{Id}}
\newcommand{\law}{\mathcal L}
\newcommand{\diag}{\operatorname{diag}}

\newcommand{\rank}{\operatorname{rank}}

\newif\ifshowexperimentplan
\ifdefined\paperworking
\showexperimentplantrue
\else
\showexperimentplanfalse
\fi

\hypersetup{
  pdftitle={Finite-Step Geometry of Marginal-Preserving Velocity Fields in Generative Flows},
  pdfauthor={Ewsey Obzherin, Valentin Khrulkov},
  pdfsubject={Numerical response geometry for generative flows}
}

\newcommand{\Cov}{\operatorname{Cov}}

\newcommand{\op}{\mathrm{op}}
\newcommand{\clip}{\operatorname{clip}}

%% file: sections/00_abstract.tex
\begin{abstract}
Generative flows can admit different velocity fields that induce the same continuous marginal evolution but behave differently after numerical discretization. We study this finite-step non-equivalence while keeping the marginal path and numerical method fixed. Projected modified equations characterize the component of local integration error that changes the marginal law and give sufficient conditions for higher marginal convergence order. For affine fixed-shape distributions, we characterize exact finite-step reachability and derive spectral obstructions and sharp nearest-reachable Wasserstein errors. For diffusion probability-flow ODEs, we construct skew-score perturbations that preserve the exact marginal path under exact scores and quantify their leakage under score error. Controlled Gaussian and mixture experiments recover the predicted order, reachability, and obstruction phenomena.
\end{abstract}

%% file: sections/01_problem.tex
\section{Why the Velocity Field Still Matters}\label{sec:problem}
A generative flow turns a simple random input into a sample by following a velocity field. Diffusion probability-flow ODEs and flow-matching models implement this idea with a neural network \citep{song2021score,lipman2023flow,albergo2023stochastic}:
\[
\dot X_t=v_t(X_t),\qquad X_0\sim\rho_0.
\]
Here $\rho_t$ denotes the distribution of particles at time $t$. Sampling requires a numerical approximation to this motion, and each velocity evaluation usually costs one network call. With only a few calls, the numerical distribution can differ substantially from the intended one.

Can we improve that approximation by changing the velocity \emph{without changing the intended continuous distributions}? A distribution records probability mass, not particle identities. Different motions can therefore produce the same continuous distributions but different numerical outputs.

\paragraph{A simple example: expansion with or without rotation.}
Imagine a two-dimensional Gaussian cloud whose radius grows exponentially. Let $\beta>0$ set its expansion rate, and let $J$ rotate a vector by a quarter turn. The fields
\begin{align}
\rho_t&=\mathcal N(0,e^{2\beta t}I_2),\label{eq:intro-gaussian}\\
w_\omega(x)&=(\beta I+\omega J)x,
\quad J=\begin{pmatrix}0&-1\\1&0\end{pmatrix}\label{eq:intro-family}
\end{align}
all produce this Gaussian path: the exact motion expands by $e^{\beta t}$ and rotates by $\omega t$, leaving the circular distributions unchanged.

Now take one Euler step of length $h$. Without rotation, its radial multiplier is $1+\beta h$, smaller than the exact multiplier $e^{\beta h}$. With rotation, the covariance multiplier becomes
\begin{equation}\label{eq:intro-euler-cov}
(1+\beta h)^2+\omega^2h^2.
\end{equation}
Choosing
\begin{equation}\label{eq:intro-omega}
\omega_h^2=\frac{e^{2\beta h}-(1+\beta h)^2}{h^2}
\end{equation}
therefore makes the numerical distribution exact. This choice is real for $\beta,h>0$. On a uniform grid it matches every grid marginal, although numerical and exact \emph{particle} endpoints still differ.

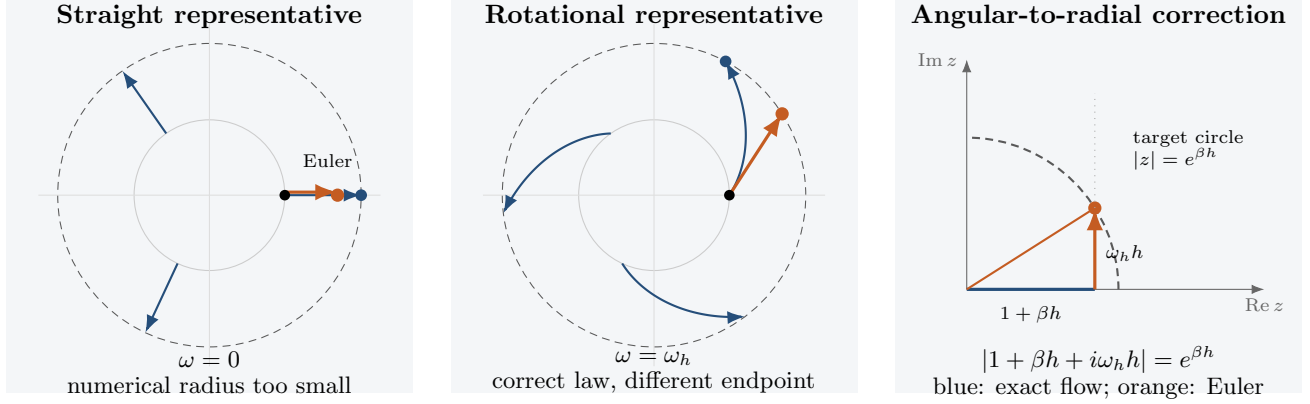
\begin{figure*}[t]
\centering
\input{figures/motivating_geometry}
\caption{\textbf{Rotation can repair an under-expanding Euler step.} In this exact two-dimensional Gaussian example ($\beta h=0.7$), the radial step falls short (left). Adding rotation preserves the continuous distribution and gives Euler the correct endpoint radius (middle), while individual particle endpoints remain different. The line--circle intersection on the right gives the required rotation. Circles are Gaussian density contours, not a bounded support. These are analytic illustrations, not benchmark measurements.}
\label{fig:intro-separation}
\end{figure*}

Euler follows a straight segment rather than the curved motion. The sideways displacement from rotation lengthens that segment and supplies the missing expansion (\cref{fig:intro-separation}). The straightest representative, which also has the least kinetic energy in this family, is therefore not the best representative for this numerical method.

\paragraph{What we hold fixed.}
The primary comparison keeps the continuous distribution path, solver, timesteps and network-call budget fixed. Other approaches change the learned transport, the update formula or the schedule \citep{liu2023rectified,karras2022edm,lu2022dpmsolver,sabour2024align,shaul2024bespoke}. Selecting the correction has an offline cost beyond sampling calls; \cref{app:calibration-algorithm,tab:recorded-cost} specify what is recorded and what remains unmeasured. Tasks that require a particular initial-to-final pairing, such as a prescribed inversion map, impose an additional constraint; improving a sampled distribution alone does not solve them.

\paragraph{Contributions.}
We answer three questions about a fixed distribution path.
\begin{enumerate}[leftmargin=*,itemsep=2pt,topsep=3pt]
\item \textbf{Marginal error and order.} We project modified-equation errors onto the component that changes the law. Cancelling successive components gives sufficient conditions for higher marginal convergence order, including quantitative tolerances for approximate cancellation.
\item \textbf{Exactness and unavoidable error.} For affine fixed-shape laws, we characterize finite-step reachability and its spectral obstructions. Rank, operator-norm and energy budgets give sharp nearest-reachable Wasserstein errors; the selected representative depends on the solver and grid.
\item \textbf{Constructions and empirical mechanisms.} Skew-score perturbations preserve exact diffusion marginals, with a separate leakage bound for score error. Controlled and pretrained-model tests examine solver transfer, native window fitting, local response and integration against each field's own dense flow. An exact Gaussian calibration analysis separates minibatch and penalty effects from these measured neural outcomes.
\end{enumerate}
Notation and detailed attribution appear in \cref{tab:notation,app:novelty-boundary}.

%% file: figures/motivating_geometry.tex
% Formula-based 2D illustration: beta*h = 0.7. No empirical data.
\resizebox{\textwidth}{!}{%
\begin{tikzpicture}[x=1cm,y=1cm,>=Latex,font=\small]
\pgfmathsetmacro{\rad}{exp(0.7)}
\pgfmathsetmacro{\b}{sqrt(exp(1.4)-1.7^2)}
\foreach \k in {0,1}{
  \begin{scope}[shift={(5.9*\k,0)}]
  \fill[SoftGray] (-2.7,-2.62) rectangle (2.7,2.62);
  \draw[black!20] (0,0) circle (1);
  \draw[black!65,densely dashed] (0,0) circle (\rad);
  \draw[black!13] (-2.27,0)--(2.27,0);
  \draw[black!13] (0,-2.27)--(0,2.27);
  \ifnum\k=0
    \node[font=\bfseries] at (0,2.37) {Straight representative};
    \foreach \ang in {0,125,245}{
      \draw[DeepBlue,thick,->] ({cos(\ang)},{sin(\ang)}) -- ({\rad*cos(\ang)},{\rad*sin(\ang)});
    }
    \draw[DeepOrange,very thick,->] (1,0.04)--(1.7,0.04);
    \fill[DeepOrange] (1.7,0) circle (2.5pt);
    \fill[DeepBlue] (\rad,0) circle (2.2pt);
    \fill[black] (1,0) circle (2pt);
    \node[anchor=south,font=\scriptsize] at (1.57,0.25) {Euler};
    \node[align=center] at (0,-2.35) {$\omega=0$\\numerical radius too small};
  \else
    \node[font=\bfseries] at (0,2.37) {Rotational representative};
    \foreach \ang in {0,125,245}{
      \draw[DeepBlue,thick,->,domain=0:1,samples=55,variable=\t]
       plot ({exp(0.7*\t)*cos(\ang+\b*\t*180/pi)},
             {exp(0.7*\t)*sin(\ang+\b*\t*180/pi)});
    }
    \draw[DeepOrange,very thick,->] (1,0)--(1.7,\b);
    \fill[DeepOrange] (1.7,\b) circle (2.5pt);
    \fill[DeepBlue] ({\rad*cos(\b*180/pi)},{\rad*sin(\b*180/pi)}) circle (2.2pt);
    \fill[black] (1,0) circle (2pt);
    \node[align=center] at (0,-2.35) {$\omega=\omega_h$\\correct law, different endpoint};
  \fi
  \end{scope}
}
\begin{scope}[shift={(11.8,0)}]
 \fill[SoftGray] (-2.7,-2.62) rectangle (2.7,2.62);
 \node[font=\bfseries] at (0,2.37) {Angular-to-radial correction};
 \begin{scope}[shift={(-1.75,-1.25)}]
  \draw[->,black!60] (0,0)--(3.95,0) node[below,font=\scriptsize] {$\operatorname{Re}z$};
  \draw[->,black!60] (0,0)--(0,3.05) node[left,font=\scriptsize] {$\operatorname{Im}z$};
  \draw[black!65,densely dashed,thick] (\rad,0) arc[start angle=0,end angle=90,radius=\rad];
  \draw[DeepBlue,very thick] (0,0)--(1.7,0);
  \draw[DeepOrange,very thick,->] (1.7,0)--(1.7,\b);
  \draw[DeepOrange,thick] (0,0)--(1.7,\b);
  \fill[DeepOrange] (1.7,\b) circle (2.5pt);
  \draw[black!35,dotted] (1.7,\b)--(1.7,2.6);
  \node[anchor=north,font=\scriptsize] at (0.85,-0.10) {$1+\beta h$};
  \node[right,font=\scriptsize] at (1.72,0.45*\b) {$\omega_hh$};
  \node[font=\scriptsize,align=left,anchor=west] at (2.08,1.85) {target circle\\$|z|=e^{\beta h}$};
 \end{scope}
 \node[align=center] at (0,-2.35) {$|1+\beta h+i\omega_hh|=e^{\beta h}$\\blue: exact flow; orange: Euler};
\end{scope}
\end{tikzpicture}%
}

%% file: sections/02_response.tex
\section{What It Means to Keep the Distribution Fixed}\label{sec:response}
The Gaussian rotation illustrates a general constraint. The continuity equation describes how a velocity redistributes probability mass:
\begin{equation}\label{eq:ce}
\partial_t\rho_t+\Div(\rho_tv_t)=0.
\end{equation}
Adding $u_t$ leaves this equation unchanged exactly when
\begin{equation}\label{eq:gauge}
\Div(\rho_tu_t)=0.
\end{equation}
The quantity $\rho_tu_t$ is a probability current; zero divergence means that this additional motion creates no net local gain or loss of mass. We call it a \emph{marginal gauge}. Under the regularity and transport-uniqueness assumptions stated in \cref{app:response-guide}, $v+u$ then produces the same continuous distributions as $v$. The constraint must hold throughout the interval, not just at its first node. A correction need not move along density contours to satisfy it.

\paragraph{The numerical response is simply the output distribution.}
Start a step with particles distributed as $\rho_t$, apply the chosen numerical map $\Psi_{t,h}^{v+u}$, and inspect their new distribution:
\begin{equation}\label{eq:response}
\Response_{t,h}(u)=(\Psi_{t,h}^{v+u})_\#\rho_t.
\end{equation}
The symbol $\#\rho_t$ means the distribution obtained by applying the map to a draw from $\rho_t$. Exact integration always returns the prescribed $\rho_{t+h}$ for an admissible correction. One numerical step need not. A low-rank restriction or an amplitude cap limits which responses are available.

We use this single output map in three ways. At small $h$, we ask what its leading distributional error is. At a fixed $h$, we ask whether any allowed correction reaches $\rho_{t+h}$ exactly. If it cannot, we minimize the distance to that target. We use $W_2$, the Wasserstein distance that measures the least root-mean-square displacement needed to transport one distribution to another. The formal reachable-set and minimization definitions are in \cref{app:admissible-response}.

\paragraph{Why local success is not enough.}
Later steps can amplify local error. Under Lipschitz stability, the final transport error is bounded by local errors weighted by their subsequent amplification (\cref{app:superconv}). This motivates separate local, own-field integration and image-quality checks.

%% file: sections/03_visibility.tex
\section{Which Integration Errors Change the Law?}\label{sec:visibility}
For general smooth fields, we test whether the omitted motion changes probability mass.

\paragraph{Euler misses acceleration, but not all acceleration matters.}
For $w=v+u$, the exact trajectory begins as $x+hw+h^2a_w/2+O(h^3)$, where
\begin{equation}\label{eq:material-acc}
a_w=\partial_tw+(\nabla w)w
\end{equation}
is the acceleration along the trajectory. Euler keeps only $x+hw$. Nevertheless, its leading distributional error depends only on the weighted divergence of the omitted acceleration:
\begin{equation}\label{eq:euler-density-defect}
(\Psi_{h,\mathrm E}^w)_\#\rho_t-\rho_{t+h}
=\frac{h^2}{2}\Div(\rho_ta_w)+O(h^3).
\end{equation}
This expansion is weak: it compares expectations of smooth test functions. If $\Div(\rho_ta_w)=0$, the leading particle error merely rearranges mass and the leading density error disappears.

We can separate any field into a component that changes the distribution and an orthogonal component that only rearranges it. Write $\Proj_\rho$ for the projection onto the former, the closure of gradient fields in $L^2(\rho)$. For Euler the relevant component is
\begin{equation}\label{eq:euler-active}
\Proj_{\rho_t}a_w.
\end{equation}
It can vanish even when the full acceleration is nonzero. The weighted projection also supplies the local Wasserstein error coefficient \citep{benamou2000computational,ambrosio2008gradient}; its construction and norm identity appear in \cref{app:hodge}.

\paragraph{The same idea applies beyond Euler.}
A modified equation represents a numerical method as the exact flow of a nearby velocity, expanded in powers of $h$ \citep{hairer2006geometric}. Its successive coefficients $E_{p_{\rm meth}},E_{{p_{\rm meth}}+1},\ldots$ describe errors introduced by a method of classical order ${p_{\rm meth}}$. We test each coefficient with the same mass-conservation condition as in \cref{eq:gauge}.

\textbf{Order result.} If the coefficients $E_{p_{\rm meth}},\ldots,E_{q-1}$ all satisfy $\Div(\rho_tE_j)=0$, the local law error is $O(h^{q+1})$ and stable repeated steps have terminal $W_2$ error $O(h^q)$. This statement holds for a smooth, uniformly positive density path on a flat torus, with a uniformly controlled modified expansion, sufficient map agreement and Lipschitz stability. \Cref{app:modified-order} states these conditions and the leading error coefficient precisely; \cref{cor:tolerance} quantifies the required accuracy of approximate cancellation.

These are sufficient conditions, not a general existence theorem for a correction. For the expanding Gaussian, a fixed rotation attains local $h^3$ law error, and a moment lower bound excludes local $O(h^4)$ for any fixed finite-energy lift (\cref{prop:fixed-euler-barrier}). This sharp local result does not forbid further terminal cancellation. The method's classical order is unchanged.

\paragraph{Controlled check: state error versus law error.}
Expansion--rotation tests give trajectory and marginal tail slopes $0.984$ and $1.968$, matching first-order particle and second-order law error. A separate fixture combines large full acceleration with negligible projected acceleration, showing why curvature is not a law-error proxy (\cref{app:controlled-order,app:controlled-curvature}).

\paragraph{Two questions must remain separate.}
An order claim requires estimates uniform in $h$, whereas fixed-step exactness may select a separate field for each budget. An $h$-dependent family yields an order guarantee when it satisfies the uniform hypotheses and cancellation budgets of \cref{cor:tolerance}. A budget-specific exact correction alone does not supply those estimates.

One can select a correction by minimizing the squared size of the projected leading error (\cref{eq:Br}). A finite dictionary may only observe or cancel part of this quantity, however; an unobserved remainder can remain. The optimization, approximation and unresolved-projection conditions are given in \cref{app:general-opt}. Next we study a family where the \emph{complete} finite-step error, rather than just its leading term, can be computed.

%% file: sections/04_reachability.tex
\section{When Can One Step Be Exact?}\label{sec:reachability}
Rotation repaired the two-dimensional expanding Gaussian. It cannot repair every target: rotations act in planes, and this puts a concrete restriction on the changes they can make. We derive exact feasibility and the smallest remaining error for affine fixed-shape distributions.

\paragraph{A family whose shape can be tracked exactly.}
Write $\rho_t=\law(m_t+\Sigma_t^{1/2}Z)$, where a single spherical random vector $Z$ has mean zero and covariance $I$. The center $m_t$ and positive covariance $\Sigma_t$ change with time, while the underlying radial shape is fixed. Gaussian and normalized fixed-degree Student families with finite covariance are examples. Arbitrarily changing a distribution's radial shape is outside this family.

An affine velocity must satisfy two requirements: its instantaneous covariance change must agree with the prescribed path, and its Euler endpoint must have the target covariance. Matching one does not guarantee the other. \Cref{thm:finite-gaussian} gives a necessary-and-sufficient matrix criterion for the general case, including the separate mean condition $m_{t+h}=m_t+h\dot m_t$. This is a feasibility equivalence, not a general efficient anisotropic solver. We specialize to a covariance-$I$ source.

\paragraph{Rotation adds expansion in pairs.}
If the instantaneous covariance grows equally in every direction, the affine generator has the form
\begin{equation}\label{eq:normal-form}
B=\beta I+K,\qquad K^\top=-K.
\end{equation}
The symmetric part $\beta I$ specifies the growth; the skew part $K$ specifies rotations. Put $a=1+h\beta$. The complete Euler covariance is
\begin{equation}\label{eq:euler-normal-form}
\Gamma(K)=a^2I-h^2K^2.
\end{equation}
In a rotation plane with frequency $\omega$, \emph{both} directions have radius $\sqrt{a^2+h^2\omega^2}$. Thus a rotation can enlarge the uncorrected radius $|a|$, but cannot shrink it or enlarge only one member of a pair.

\textbf{Exactness result.} A positive target covariance ${\Sigma_\star}$ is reachable exactly if and only if every target variance is at least $a^2$ and every variance strictly above $a^2$ occurs with even multiplicity (\cref{app:spectral-pairing}). When an isotropic target radius exceeds $|a|$, all directions can be corrected exactly in even dimension; odd dimension leaves one direction unpaired. At equality no rotation is needed in either dimension. A target smaller than the Euler radius cannot be repaired by adding rotation. This is a single-step obstruction: for odd $d\geq3$, feasible affine fields $\beta I+K(t)$ can match an expanding spherical terminal law in three Euler steps, but not one or two (\cref{prop:odd-three-step}). Intermediate laws need not match.

\begin{figure*}[t]
\centering
\input{figures/response_geometry_2d}
\caption{\textbf{A correction can change the output without reaching every target.} The first two panels show all Euler covariance responses in two dimensions as the rotation $K=\omega J$ varies ($h=1$). With unequal growth rates the response is a parabola; with equal growth rates it is a ray. Targets on these curves are reachable; the displayed targets off them are not. The coordinates are ${\xi_{\rm mix}}=\Gamma_{12}$ and ${\zeta_{\rm exp}}=\tfrac12\tr(\Gamma-\Gamma(0))$, not an isometric Wasserstein chart. Right: target radii $|a|\le y_1<y_2$ must share the reachable radius $\varrho_\star=(y_1+y_2)/2$; the displayed residual cost is exact.}
\label{fig:response-2d}
\end{figure*}
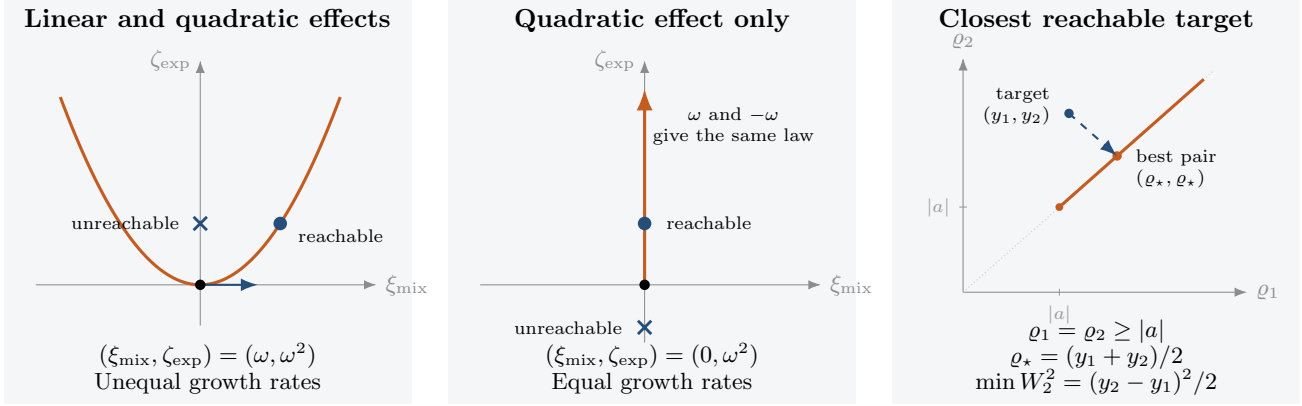

\paragraph{How much error is unavoidable?}
Suppose $d$ is even and all target variances lie above the floor $a^2$. Order their square roots as $y_1\le\cdots\le y_d$. The exact minimum is
\begin{equation}\label{eq:sharp-pairing-main}
\begin{aligned}
&\inf_{K^\top=-K}W_2^2\bigl(\law((aI+hK)Z),\law({\Sigma_\star}^{1/2}Z)\bigr)\\
&\hspace{3em}=\frac12\sum_{j=1}^{d/2}(y_{2j}-y_{2j-1})^2.
\end{aligned}
\end{equation}
Each pair must share a radius, so it chooses the mean of its two target radii. Pairing adjacent ordered radii minimizes the total loss (\cref{fig:response-2d}). Rank limits restrict how many planes can be corrected; an operator-norm cap limits how far their radii can grow. For example, target radii $2$ and $3$ in one plane must share the reachable radius $2.5$, leaving squared error $0.5$ when the Euler floor is at most $2$. A cap restricting that radius to $2.2$ instead leaves $(2-2.2)^2+(3-2.2)^2=0.68$. In higher dimension, the budget decides which pairs receive a correction. These are exact law-distance calculations within the stated spherical family, not particle-coupling errors or empirical feature distances. \Cref{thm:rank,cor:budgets} give the clipped-pair optimum; \cref{cor:energy-spectral-optimum} adds an exact energy budget.

\paragraph{Controlled check: exactness and its boundary.}
Gaussian tests recover even-dimensional exactness, odd-dimensional residuals and clipped spectral optima predicted by the formulas (\cref{app:controlled-reachability}). Mixtures provide a less symmetric test (\cref{app:controlled-mixtures}): the analytic finite-step correction reduces component-scale error in the separated regime, while additional calibration gives no improvement in the overlap stress test. Exact spherical-family feasibility does not guarantee calibration improvement in mixtures.

Unequal growth rates add a first-order mixing effect, which can vanish along symmetry directions without making the full finite response vanish (\cref{prop:symmetry-strata}). This separates sensitivity to the correction from exact target reachability (\cref{app:anisotropic-proof}). The general affine criterion, symmetry classification and proofs are in \cref{app:elliptical,app:solver-strata}. The transport-distance identity used here is classical \citep{gelbrich1990formula,bhatia2019bures}; the constrained reachable-set calculation is the result. Normalizing a covariance preserves feasibility, but a nonorthogonal change of coordinates changes the Euclidean $W_2$ metric.

The expanding-cloud mechanism also has a useful limit. For a stationary Euclidean law with finite variance, one Euler step along a finite-energy pure gauge adds $h^2\E[uu^\top]$ to its covariance. Any nonzero such correction creates error rather than repairing it. The moment identity and the contrasting torus setting are proved in \cref{app:moment-budget,app:nonlinear-atlas}.

%% file: figures/response_geometry_2d.tex
% Exact two-dimensional covariance response and aligned Bures geometry.
% Drawing coordinates are schematic. The stated mean optimum assumes
% target radii |a| <= y_1 < y_2, as specified in the accompanying caption.
\resizebox{\textwidth}{!}{%
\begin{tikzpicture}[x=1cm,y=1cm,>=Latex,font=\small]
\foreach \k in {0,1,2}{
\begin{scope}[shift={(6*\k,0)}]
\fill[SoftGray] (-2.65,-1.6) rectangle (2.85,3.9);
\ifnum\k=0
 \node[font=\bfseries] at (0.1,3.57) {Linear and quadratic effects};
 \draw[black!45,->] (-2.22,0)--(2.38,0) node[right] {$\xi_{\rm mix}$};
 \draw[black!45,->] (0,-0.55)--(0,3.03) node[left] {$\zeta_{\rm exp}$};
 \draw[DeepOrange,very thick,domain=-1.05:1.05,samples=80,variable=\t]
 plot ({1.8*\t},{2.3*\t*\t});
 \draw[DeepBlue,thick,->] (0,0)--(0.78,0);
 \fill[black] (0,0) circle (2pt);
 \fill[DeepBlue] (1.08,0.828) circle (2.6pt);
 \node[font=\scriptsize,anchor=west] at (1.2,0.67) {reachable};
 \draw[DeepBlue,line width=1.1pt] (-0.09,0.738)--(0.09,0.918) (-0.09,0.918)--(0.09,0.738);
 \node[font=\scriptsize,anchor=east] at (-0.16,0.84) {unreachable};
 \node[align=center] at (0.1,-1.12) {$(\xi_{\rm mix},\zeta_{\rm exp})=(\omega,\omega^2)$\\Unequal growth rates};
\else\ifnum\k=1
 \node[font=\bfseries] at (0.1,3.57) {Quadratic effect only};
 \draw[black!45,->] (-2.22,0)--(2.38,0) node[right] {$\xi_{\rm mix}$};
 \draw[black!45,->] (0,-0.78)--(0,3.03) node[left] {$\zeta_{\rm exp}$};
 \draw[DeepOrange,very thick,->] (0,0)--(0,2.65);
 \fill[black] (0,0) circle (2pt);
 \fill[DeepBlue] (0,0.828) circle (2.6pt);
 \node[font=\scriptsize,anchor=west] at (0.17,0.84) {reachable};
 \draw[DeepBlue,line width=1.1pt] (-0.09,-0.665)--(0.09,-0.485) (-0.09,-0.485)--(0.09,-0.665);
 \node[font=\scriptsize,anchor=east] at (-0.18,-0.58) {unreachable};
 \node[font=\scriptsize,align=center] at (1.2,2.15) {$\omega$ and $-\omega$\\give the same law};
 \node[align=center] at (0.1,-1.12) {$(\xi_{\rm mix},\zeta_{\rm exp})=(0,\omega^2)$\\Equal growth rates};
\else
 \node[font=\bfseries] at (0.1,3.57) {Closest reachable target};
 \begin{scope}[shift={(-1.70,-0.10)},x=1.3cm,y=1.15cm]
  \draw[black!45,->] (0,0)--(2.95,0) node[right] {$\varrho_1$};
  \draw[black!45,->] (0,0)--(0,2.75) node[above] {$\varrho_2$};
  \draw[black!40] (1,0.04)--(1,-0.04) node[below,font=\scriptsize] {$|a|$};
  \draw[black!40] (0.04,1)--(-0.04,1) node[left,font=\scriptsize] {$|a|$};
  \draw[black!22,densely dotted] (0,0)--(2.6,2.6);
  \draw[DeepOrange,very thick] (1,1)--(2.5,2.5);
  \fill[DeepOrange] (1,1) circle (1.5pt);
  \fill[DeepBlue] (1.1,2.1) circle (1.8pt);
  \fill[DeepOrange] (1.6,1.6) circle (1.8pt);
  \draw[DeepBlue,thick,dashed,->] (1.1,2.1)--(1.6,1.6);
  \node[font=\scriptsize,anchor=east,align=right] at (1.0,2.18) {target\\$(y_1,y_2)$};
  \node[font=\scriptsize,anchor=west,align=left] at (1.7,1.45) {best pair\\$(\varrho_\star,\varrho_\star)$};
 \end{scope}
 \node[align=center] at (0.1,-0.98) {$\varrho_1=\varrho_2\ge |a|$\\$\varrho_\star=(y_1+y_2)/2$\\$\min W_2^2=(y_2-y_1)^2/2$};
\fi\fi
\end{scope}
}
\end{tikzpicture}%
}

%% file: sections/05_solver.tex
\section{Why the Choice of Solver Matters}\label{sec:solver}
The continuous path fixes admissibility, but the solver fixes the response: a correction that repairs Euler need not repair Heun on the same timesteps.

In a constant two-dimensional expansion--rotation block, write the generator as $\lambda I+\omega J$. The exact flow multiplies the radius by $e^{h\lambda}$. A numerical method instead multiplies it by $|R(h\lambda+ih\omega)|$, where $R$ is its scalar amplification function: $R(z)=1+z$ for Euler and $R(z)=1+z+z^2/2$ for Heun. Exact marginal matching is precisely
\begin{equation}\label{eq:solver-root-main}
|R(h\lambda+ih\omega)|=e^{h\lambda}.
\end{equation}
Changing the method changes the useful frequencies. Bounded exact branches exist for small $h$, with limits $|\lambda|$ for Euler and $|\lambda|/\sqrt3$ for Heun (\cref{cor:bounded-frequency-branches}); they still vary with $h$. The comparison is with the intended expansion $e^{h\lambda}$, rather than with unit amplitude as in an absolute-stability test.

This argument extends to constant generators of the whole path $\law(e^{tS}Z)$ when the rotations commute with the symmetric growth matrix $S$. Within each growth eigenspace, planar rotations must satisfy \cref{eq:solver-root-main}; a leftover zero-frequency direction must already have the correct scalar multiplier. Consequently an odd-dimensional growth block can be obstructed even when the total dimension is even. \Cref{thm:block-classification} gives the complete classification. The response formula and derivative-kernel analysis are in \cref{app:solver-strata,thm:response-kernel}.

These conclusions concern frozen autonomous linear fields. Nonlinear, time-varying fields require the actual multistage map, not a Jacobian in $R$. Frequencies of order $1/h$ fall outside the fixed-field order theorem (\cref{sec:visibility}).

\paragraph{Controlled check: native fitting versus solver transfer.}
Across ten independently refit Digits latent PCA/class-Gaussian models ($d=16$, Euler with four field calls), finite-step calibration reduces prescribed class-scale RMSE to $0.0238\pm0.0010$, from uncorrected Base $0.461\pm0.002$ and leading-order calibration $0.174\pm0.002$ (mean $\pm$ standard error across refits); the independent sampling floor is $0.0262\pm0.0013$. Under two-step Heun, the transferred Euler correction gives $0.773\pm0.005$, versus native Heun $0.0252\pm0.0019$ and Base $0.134\pm0.002$. This is a controlled latent-distribution experiment, not an image-generation benchmark (\cref{app:controlled-digits}).

%% file: sections/06_application.tex
\section{From Exact Gauges to Pretrained Models}\label{sec:application}
Skew-score corrections reuse a pretrained denoiser call and add a low-rank matrix operation. Their path-preservation status depends on score accuracy.

\paragraph{Why a skew-score field preserves the exact path.}
The score $s_t=\nabla\log\rho_t$ points in the direction of increasing log density. A forward SDE with drift $f_t$ and scalar diffusion coefficient $g_t$ has probability-flow velocity $f_t-g_t^2s_t/2$ \citep{song2021score}. For a spatially constant skew matrix $A_t$, the additional current has zero divergence:
\[
\Div(\rho_tA_ts_t)=\Div(A_t\nabla\rho_t)=0.
\]
The last equality pairs a skew matrix with a symmetric density Hessian. Therefore
\begin{equation}\label{eq:diffusion-gauge}
w_t^A=f_t-\tfrac12g_t^2(I-A_t)s_t
\end{equation}
has the same exact marginal path. Antisymmetric currents are established \citep{ma2015recipe}; skew-score additions were already tested in generative sampling \citep{ma2022preconditioned,ma2025preconditioned}. We analyze their finite-step response.

\paragraph{How it enters a sampler.}
In the EDM parameterization, $D_\theta(x,\sigma)$ is the denoised prediction at noise level $\sigma$ \citep{karras2022edm}. The corrected update uses
\begin{equation}\label{eq:edm-gauge}
\frac{\dd x}{\dd\sigma}
=(I-A_\sigma)\frac{x-D_\theta(x,\sigma)}{\sigma}.
\end{equation}
The same denoiser output supplies both terms. A skew operator with $r$ rank pairs (rank at most $2r$) adds $O(dr)$ algebra in state dimension $d$, with no extra denoiser evaluation. If the matrix varies with position, the exact correction additionally needs its divergence; dropping that term generally loses preservation.

\paragraph{What changes when the correction is learned.}
We freeze checkpoint weights, fit a compact correction using calibration clouds, select its amplitude on an independent validation bank, and evaluate on fresh test banks. The empirical endpoint objective, specified in \cref{app:calibration-algorithm}, is a surrogate, not final FID. If the learned score differs from the exact score by $e_t$, a correction $\alpha_tA_ts_\theta$ adds at most
\begin{equation}\label{eq:learned-score-leak-main}
\|\Delta\mathcal R_t\|_{\dot H^{-1}(\rho_t)}
\le |\alpha_t|\|A_t\|_{\rm op}\|e_t\|_{L^2(\rho_t)}
\end{equation}
to the continuity residual. Here the negative norm measures mass-changing error through smooth test functions. Smaller amplitude and operator norm limit this residual; they do not certify unchanged marginals for a pretrained network. \Cref{app:diffusion,thm:leak} give the construction, leakage proof and finite-time bound.

The choice of path also matters when moving to another model family. For $X_t=(1-t)Z+tY$, with $Z\sim\mathcal N(0,I)$ independent of the data $Y$, the standard Gaussian interpolation identity gives $s_t=(tv_t-x)/(1-t)$ for $0<t<1$ \citep{albergo2023stochastic,vinograd2026eddy}. A spatially constant skew matrix then gives $\delta v_t=t(1-t){K^{\rm RF}}_t(tv_t-x)$, which vanishes at the endpoints. This is a path-specific construction, not an arbitrary rotation of a flow-matching velocity. A pretrained velocity again provides only a score proxy (\cref{prop:rf-score-gauge}).

\paragraph{Calibration itself has a finite-step limitation.}
The implemented objective compares paired feature means from common starting particles. With fixed features and iid input pairs, its expected squared discrepancy contains both the population mean error and a batch-variance term. In the expanding Gaussian of \cref{sec:problem}, the exact useful rotation can therefore lose to zero under a finite batch. At $h=.1$, batch size 1000 alone selects frequency $.68280$, while a population energy penalty $10^{-5}$ alone selects $.68214$; combining them selects zero. \Cref{app:calibration-limits} gives the exact joint threshold (\cref{prop:gaussian-joint-calibration}) and compact limit (\cref{cor:joint-calibration-limit}). This is a controlled calculation for fixed features, not a reconstruction of a neural fit. A low logged loss need not recover the best law-preserving correction.

The relevant scales differ. For bounded rotation frequencies, the squared population feature discrepancy is $O(h^4)$, whereas the extra paired-minibatch contribution at a useful rotation is of order $h^2/n_{\rm b}$. A fixed energy penalty can dominate both as the step shrinks. The compact limit in \cref{cor:joint-calibration-limit} expresses the joint effect through $1/(n_{\rm b}h^2)$ and $\lambda_h/h^4$ (\cref{fig:joint-calibration}). Its hypotheses include uniform convergence and a global optimization tolerance of $o(h^4)$. Thus the calibration analysis explains an objective-level obstruction without equating a training history, its selected-amplitude field and an asymptotic optimum.

\begin{figure*}[t]
\centering
\includegraphics[width=\textwidth]{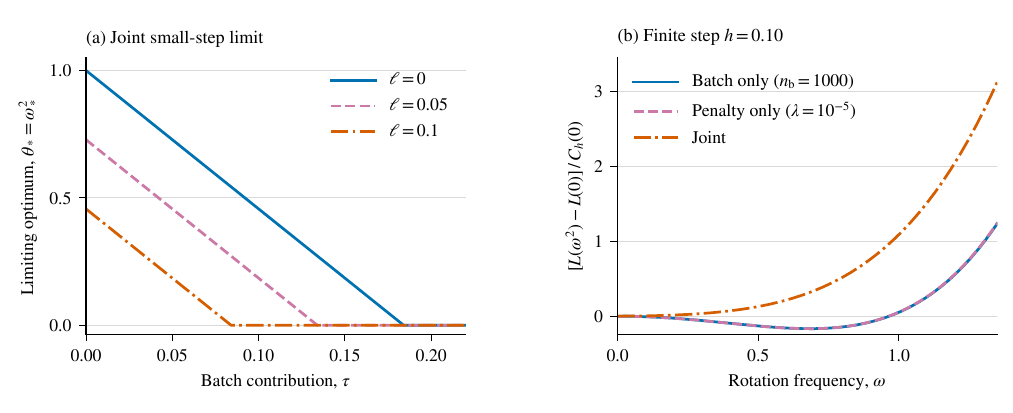}
\caption{\textbf{Sampling variance and an energy penalty can jointly suppress a useful correction.} Left: the limiting squared Gaussian rotation frequency versus $\tau_{\rm batch}=\lim(n_{\rm b}h^2)^{-1}$ for three $\ell_{\rm pen}=\lim\lambda_h/h^4$ values, with cap $\Omega=2$. Right: at $h=.1$, either effect alone prefers a nonzero correction, while their combination prefers zero. Curves subtract the objective at zero and divide by unpenalized population $C_h(0)$. These are exact fixed-feature calculations, with no neural measurements or sampling intervals.}\label{fig:joint-calibration}
\end{figure*}

%% file: sections/07_positioning.tex
\setcounter{dbltopnumber}{2}
\renewcommand{\dbltopfraction}{.9}
\section{Pretrained-Model Evidence}\label{sec:positioning}\label{sec:empirical}
We distinguish local state/feature response, integration against an approximate dense reference of the \emph{same} field, and image quality relative to real data. None certifies the others. Full frozen-model protocols, adverse/null outcomes and sample panels are in \cref{app:experiments}.

\noindent\begin{minipage}{\columnwidth}
\emph{Base} uses the uncorrected model; \emph{FH} is a correction fitted to the finite-step update. A native fit is trained for the window or solver where it is evaluated; transfer reuses it elsewhere without tuning. FID compares generated and real-image features, with lower values preferred. N4/N8 denote four/eight partition intervals; Euler uses one network call per interval, while the eight-interval Heun sampler uses 15 calls.
\end{minipage}\par

\input{sections/08_cifar_summary}

\paragraph{Native window fitting differs from redistributing a profile.}
\Cref{tab:main-cifar} summarizes CIFAR-10 gains and an adverse regime. Validation-selected corrections improve conditional-CIFAR four-call FID by $4.413$ relative to Base, with all eight independent N4 pipelines improving; unconditional-CIFAR N4 and FFHQ N8 include harmful outcomes (\cref{app:neural-checkpoints}). Relative to the original FH profile, moving that frozen profile to the late half helps both CIFAR N8 checkpoints but harms those latter cells (\cref{app:neural-temporal-profile}). We therefore train early and late corrections separately, masking the inactive half \emph{inside} the objective, then transfer the selected amplitude without receiver tuning.

For unconditional CIFAR N4, all five native late fits select $\alpha=1$ and improve FID by $1.742$--$4.364$ relative to Base. Transfer to early raises FID by $0.590$--$2.449$ relative to each corresponding native late fit; all five conditional pointwise 95\% intervals are positive (\cref{fig:window-fitting,app:neural-window}). Other cells are mixed, including FFHQ late fits that improve after transfer. Actual profiles differ, as does the original knot layout, so this supports window dependence in a specific regime rather than a universal causal rule about late-time corrections.

\input{sections/07_window_figure}

\paragraph{Better distributional integration need not mean smaller state error.}
All active rank/cap fields reduce empirical sliced-$W_2$ distance to their own dense flows, yet FID can worsen; calibration-objective--own-error associations vary by checkpoint/cap (\cref{app:neural-correlations}). In 112 local groups, one Euler step is compared with each field's own 32-substep Heun flow from common initial states. None of 84 paired intervals for the full-state RMS \emph{difference from Base} lies wholly below zero; 55 lie above. A zero correction at the Euler node need not remain zero inside the dense interval (\cref{app:neural-local}); tiny wholly inactive residuals are not learned effects. The same-bank 64-substep audit checks numerical convergence, not an independent floor. Finite-feature calibration therefore does not certify state or real-data improvement.

\paragraph{Competitiveness depends on the call budget.}
On fresh 20k banks (\cref{app:neural-cost}), Euler8+FH has lower point FID than Base Heun at 7/9 calls for all three checkpoints. Euler16 with frozen FH8 loses to Heun at 15/17 calls. These are nearby call budgets, not equal wall time; N16 tests transfer, not native N16 fitting. The independent-Gaussian rectified-flow (RF) pilot has $\Delta$FID $-0.0091$ at N8 and $-0.0084$ at N16, with both paired intervals crossing zero; its two own-field changes also differ in sign. This single-checkpoint test yields a path-specific adapter but no measurable quality gain.

\noindent\begin{minipage}{\columnwidth}
\paragraph{Scope and relation to earlier work.}
Gauge freedom and marginal observability are established \citep{horvat2024gauge,khelifa2026gradients}; Flux Matching varies fields at a fixed marginal path \citep{pao2026flux}. Weak-error and Gaussian Euler analyses are close neighbors \citep{hurault2026geometry,claustre2026gaussian}. \Cref{app:novelty-boundary} distinguishes our higher-order marginal hierarchy and constrained finite-step reachability from solver/schedule optimization \citep{obzherin2026density}.
\end{minipage}\par

%% file: sections/08_cifar_summary.tex
\begin{table*}[t]\centering\small
\caption{\textbf{CIFAR-10 gains depend on the field.} Point $\Delta\mathrm{FID}=$field--Base; bold marks lower FID within a paired bank. The native range spans all five refits, not a confidence interval. Rows come from separate studies/banks. Full conditions, adverse/null outcomes and comparator-specific intervals: \cref{app:neural-checkpoints,app:neural-temporal-profile,app:neural-window}.}\label{tab:main-cifar}
\setlength{\tabcolsep}{3pt}
\begin{tabular*}{\textwidth}{@{\extracolsep{\fill}}llrrrr@{}}\toprule
CIFAR-10 / calls & Correction & Test $n$ & Base FID & Field FID & $\Delta$FID\\\midrule
Conditional VP / 4 & FH, $\alpha=.75$ & 50k & 67.4914 & $\mathbf{63.0786}$ & $\mathbf{-4.4128}$\\
Conditional VP / 8 & FH, $\alpha=1$ & 50k & 22.2009 & $\mathbf{21.4159}$ & $\mathbf{-0.7850}$\\
Conditional VE / 8 & Frozen late & 50k & 21.4683 & $\mathbf{19.1866}$ & $\mathbf{-2.2817}$\\
Unconditional VP / 8 & Frozen late & 50k & 28.9744 & $\mathbf{26.5503}$ & $\mathbf{-2.4241}$\\
Unconditional VP / 4 & Frozen late & 50k & 95.3793 & 97.2882 & $+1.9089$\\
Unconditional VP / 4 & Native late, five refits & 20k & 96.2418 & $\mathbf{91.8783\text{--}94.4997}$ & $\mathbf{-4.3635\ldots-1.7421}$\\
\bottomrule\end{tabular*}
\end{table*}

%% file: sections/07_window_figure.tex
% Native-window evidence follows its defining paragraph.
\begin{figure*}[t]
\centering
\includegraphics[width=\textwidth]{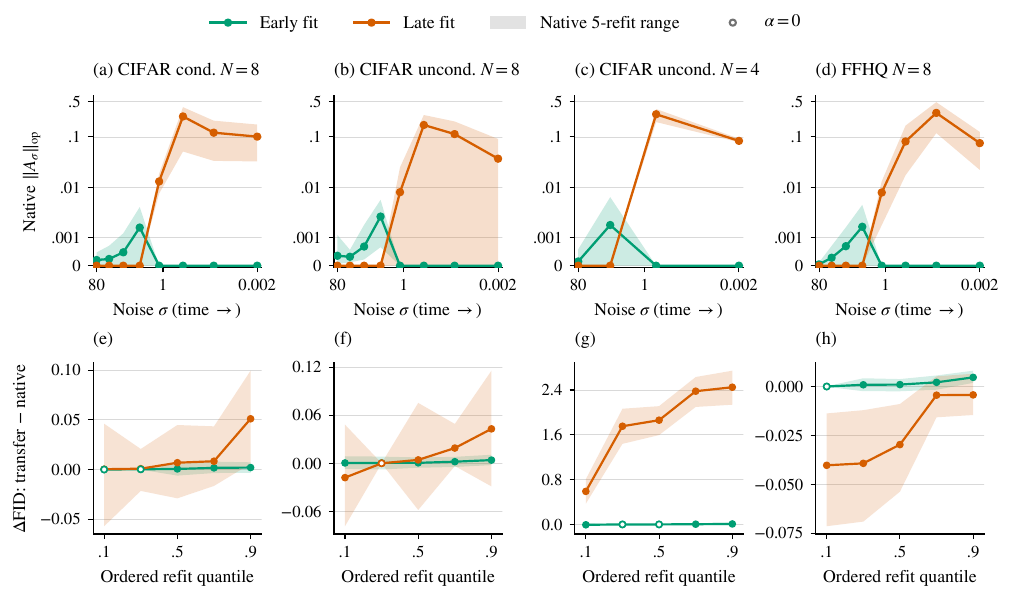}
\caption{\textbf{Fitting in one time window, then transferring to the other.} Top: mean selected $A_\sigma$ norm at Euler nodes and minimum--maximum range of five independent native refits, including zero amplitudes. Bottom: transferred minus native FID by refit quantile, with conditional pointwise 95\% paired-bootstrap bands for each fitted field. Top bands describe calibration/refit variation; bottom bands exclude calibration uncertainty. Negative differences favor transfer; positive differences favor the native window.}
\label{fig:window-fitting}
\end{figure*}

%% file: sections/08_conclusion.tex
\noindent\begin{minipage}{\columnwidth}
\section{Conclusions}\label{sec:conclusion}
With the marginal path and solver fixed, our results give sufficient order conditions, exact affine reachability criteria and sharp nearest-reachable errors. On pretrained models, gains depend on the solver, time window and call budget. Independent native-window fits can succeed in regimes where redistributing an existing profile does not; improved own-field integration still need not improve state error or FID.
\end{minipage}\par

%% file: sections/09_ai_use.tex
\subsection*{AI Use Statement}
Generative AI tools assisted with literature discovery, mathematical drafting and proof review, software implementation and debugging, experiment orchestration, manuscript editing, and figure/layout preparation. The authors are responsible for the claims, sources, code and interpretation. Analytical fixtures, pretrained-model evaluations, and posthoc summaries are identified separately; AI assistance is not treated as an independent verification of mathematical correctness or scientific conclusions.

%% file: sections/09_checklist.tex
% Official AISTATS 2027 questions are preserved verbatim.
% Only answer placeholders and brief justifications have been replaced.
\section*{Checklist}
\begin{enumerate}

  \item For all models and algorithms presented, check if you include:
  \begin{enumerate}
    \item A clear description of the mathematical setting, assumptions, algorithm, and/or model. [Yes] The paper and supplement define the fixed probability path, admissible corrections, solver response, exact structured models, learned-score approximation and independent-Gaussian RF specialization.
    \item An analysis of the properties and complexity (time, space, sample size) of any algorithm. [Yes] Theoretical exactness and obstruction results state their assumptions. Low-rank corrections add $O(dr)$ algebra per field call; the neural appendix reports actual network-call counts, sample sizes, fit updates and reference costs.
    \item (Optional) Anonymized source code, with specification of all dependencies, including external libraries. [No] The lightweight manuscript package supplies figure-generation code and compact plotting inputs, but not all sampling/fitting code and dependencies. These remain in the separate research archive; pretrained checkpoints and bulk feature/shard arrays are not bundled.
  \end{enumerate}

  \item For any theoretical claim, check if you include:
  \begin{enumerate}
    \item Statements of the full set of assumptions of all theoretical results. [Yes] Each result states its regularity, transport uniqueness, path, moment, topology, metric, rank and norm assumptions.
    \item Complete proofs of all theoretical results. [Yes] The supplement provides proofs under the stated hypotheses. Numerical regression checks complement, rather than replace, the proofs.
    \item Clear explanations of any assumptions. [Yes] The paper distinguishes local/asymptotic, fixed-step, whole-path, exact-score and neural-proxy claims and explains the restrictions underlying them.
  \end{enumerate}

  \item For all figures and tables that present empirical results, check if you include:
  \begin{enumerate}
    \item The code, data, and instructions needed to reproduce the main experimental results (either in the supplemental material or as a URL). [No] Protocols, frozen choices, scalar records and figure/galleries inputs are supplied, but a complete deposit for reproducing the main image-generation experiments is not included. Calibration clouds, fitted fields, feature arrays and sampling code remain in the separate research archive.
    \item All the training details (e.g., data splits, hyperparameters, how they were chosen). [Yes] Checkpoint weights are frozen; the neural appendix records calibration, validation/test separation, optimization settings and amplitude selection. Transfers inherit frozen choices without receiver tuning, and adverse/zero-amplitude outcomes are retained.
    \item A clear definition of the specific measure or statistics and error bars (e.g., with respect to the random seed after running experiments multiple times). [Yes] The appendix distinguishes analytic errors, empirical own-field distances, FID/KID, local state/feature diagnostics and descriptive diversity measures. Bootstrap intervals are conditional and pointwise; refit ranges, standard errors and small-family correlations have separately stated meanings.
    \item A description of the computing infrastructure used. (e.g., type of GPUs, internal cluster, or cloud provider). [Yes] Controlled/classification tests use CPU, while pretrained sampling uses eight NVIDIA H100 GPUs with resident workers. The supplement records numerical precision and network-call accounting; CPU analysis uses cached outputs.
  \end{enumerate}

  \item If you are using existing assets (e.g., code, data, models) or curating/releasing new assets, check if you include:
  \begin{enumerate}
    \item Citations of the creator If your work uses existing assets. [Yes] The paper cites Digits/scikit-learn, CIFAR-10, FFHQ, AFHQv2, EDM and Rectified Flow creators, together with the FID/KID/PRDC methods.
    \item The license information of the assets, if applicable. [No] A complete asset-license inventory is not supplied. Original dataset/model repositories and their terms govern use; the lightweight package redistributes no pretrained weights or real images, and any public research-code/data release must retain applicable licenses.
    \item New assets either in the supplemental material or as a URL, if applicable. [Yes] The package includes updated TeX, vector curves, scalar plotting inputs, cached generated-image gallery inputs and source/figure manifests. The research record distinguishes completed, deadline-partial, cancelled and deferred cases.
    \item Information about consent from data providers/curators. [Not Applicable] No new participant data or human-subject study was collected; existing datasets were used under their original access terms.
    \item Discussion of sensible content if applicable, e.g., personally identifiable information or offensive content. [Yes] FFHQ contains identifiable face imagery. No real-face images are redistributed in this package; fixed-seed generated samples are identified as qualitative illustrations.
  \end{enumerate}

  \item If you used crowdsourcing or conducted research with human subjects, check if you include:
  \begin{enumerate}
    \item The full text of instructions given to participants and screenshots. [Not Applicable] No crowdsourcing or human-subject research was conducted.
    \item Descriptions of potential participant risks, with links to Institutional Review Board (IRB) approvals if applicable. [Not Applicable] No crowdsourcing or human-subject research was conducted.
    \item The estimated hourly wage paid to participants and the total amount spent on participant compensation. [Not Applicable] No participant compensation was incurred.
  \end{enumerate}

\end{enumerate}

%% file: supplement_compact/appendix_layout.tex
% Appendix-only placement: full-size figures flow beside their related prose.
\setcounter{topnumber}{3}
\setcounter{bottomnumber}{3}
\setcounter{totalnumber}{6}
\renewcommand{\topfraction}{.95}
\renewcommand{\bottomfraction}{.95}
\renewcommand{\textfraction}{.05}
\renewcommand{\floatpagefraction}{.8}
\setlength{\floatsep}{8pt plus 2pt minus 2pt}
\setlength{\textfloatsep}{10pt plus 2pt minus 2pt}
\setlength{\intextsep}{8pt plus 2pt minus 2pt}
\setlength{\abovecaptionskip}{5pt}
\setlength{\belowcaptionskip}{0pt}

\makeatletter
\setlength{\@fptop}{0pt}
\setlength{\@fpsep}{8pt}
\setlength{\@fpbot}{0pt plus 1fil}
\makeatother

%% file: supplement_compact/overview.tex
\phantomsection\label{app:notation}\paragraph{Notation and scope.}
\Cref{tab:notation} distinguishes recurring field, density, matrix and objective roles. Dummy indices/constants retain local scope; family subscripts do not make operators interchangeable.
\begin{table}[H]\centering\small
\caption{Recurring notation. Definitions and assumptions remain in the indicated text; the table is a reading aid.}\label{tab:notation}
\begin{tabularx}{\textwidth}{>{\raggedright\arraybackslash}p{4.4cm}X}\toprule
Symbols & Meaning\\\midrule
$\rho_t$, $\widetilde\rho_t$, $\rho_t^{\rm B},\rho_t^{\rm G}$ & Prescribed density, comparison-flow density, and actual learned Base/Gauge densities.\\
$v,u,w=v+u$ & Base velocity, admissible correction, and selected representative.\\
$\Pi_\rho$, $\Psi_h^w$, $\Response_{t,h}$, $\Reach_{t,h}$ & Weighted gradient projection, numerical step map, response law, and reachable set.\\
$m_t$, $\Sigma_t$, $\Sigma_\star$ & Center (mean when integrable), scatter (covariance when normalized), and whitened target scatter.\\
$Z$, $P_{\rm PCA}$ & Spherical/Gaussian random input and fixed calibration PCA basis.\\
$B$, $\widehat B=L^{-1}BL$, $S$, $K$, $J$ & Affine generator, whitened generator, symmetric strain, skew control, and orthogonal complex structure.\\
$H_\Sigma$, $G_\Sigma$, $T_{\rm opt}$ & Whitened covariance tangent/curvature and positive optimal-transport map.\\
$A_\sigma$, $M_{\rm sk}$, $L_\sigma^{\rm res}$, $K_t^{\rm RF}$ & Dimensionless EDM skew operator, score mobility, residual factor $A_\sigma/\sigma$, and independent-path RF skew operator.\\
$U_{\rm fit},V_{\rm fit},b_\sigma$ & Latent fit factors and their noise-dependent profile coefficients.\\
$\mathcal C_{\rm cal}$, $\mathcal B_{p_{\rm meth}}$ & Finite empirical calibration objective and population projected modified-error objective.\\
$p_{\rm meth}$, $p_{\rm lin}$, $q$, $\ell_{\rm mod}$, $E_j$, $\mathcal Q_h$ & Classical method order, scalar linear order, first active marginal index, expansion cutoff, modified coefficients and remainder.\\
$r$, $\varrho$, $\upsilon$, $\ell_{\rm rad},\ell_{\rm jet}$ & Rank-pair budget, radial coordinate, source quantile, radial power and jet degree.\\
$R(z)$, $N$, $h$, $\sigma$, $\kappa_{\rm EDM}$ & Stability function, step count, signed step, noise level and EDM grid exponent (7).\\
$s=\nabla\log\rho$, $\varsigma$, $f,g$ & Score, isotropic scale, and SDE drift/diffusion coefficient.\\
$D_\theta$, $D_{\rm diff}$, $\Xi$ & Denoiser, diffusion covariance, and antisymmetric current potential.\\
$\mathcal V_\rho$, $\mathcal E_u$, $C_{\rm post}$, $\chi_m$ & State variance, gauge energy, posterior velocity covariance and mixture responsibilities.\\
$\mathcal D_j$, $\mathcal I_F$, $G_{\rm Gram}$ & Linearized projected residual operator, dictionary right inverse and gradient Gram matrix.\\
$\xi_{\rm mix}$, $\zeta_{\rm exp}$ & Mixing and expansion coordinates in \cref{fig:response-2d}; neither is the mean.\\\bottomrule
\end{tabularx}
\end{table}

\paragraph{Reading the supplement.}
\Cref{app:hodge,app:elliptical,app:solver-strata} develop marginal order, reachability and solver dependence; \cref{app:nonlinear-atlas} gives nonlinear boundaries. Controlled checks follow the corresponding proofs in \cref{app:controlled-order,app:controlled-reachability,app:controlled-solver,app:controlled-digits,app:controlled-radial,app:controlled-mixtures,app:controlled-curvature}; \cref{app:controlled-checks} collects protocol and regression scope. Neural adapters and experiments follow in \cref{app:diffusion,app:experiments}, with finite-feature calibration limits in \cref{app:calibration-limits}. Attribution is in \cref{app:novelty-boundary}; sample selection is specified below.
\input{supplement_compact/checks/roadmap}
\input{supplement_compact/sample_galleries}

%% file: supplement_compact/checks/roadmap.tex
\begin{table}[H]\centering\small
\caption{From theoretical predictions to controlled checks. Each proof cluster is followed by its corresponding verification; pretrained evidence remains separate in \cref{app:experiments}.}\label{tab:compact-controlled}\label{tab:toy-summary}
\begin{tabularx}{\linewidth}{@{}p{3.0cm}Xp{2.3cm}@{}}\toprule Theory & Prediction and retained check & Location\\\midrule
Projected marginal order & E0: trajectory versus law convergence & \cref{app:controlled-order}\\
Affine reachability & E0/E1: parity, spectral and rank/cap obstructions & \cref{app:controlled-reachability}\\
Solver response & E0/E3: exact roots, native fitting and solver/grid transfer & \cref{app:controlled-solver,app:controlled-digits}\\
Euclidean and radial limits & Stationary moment defect, torus contrast and shellwise regularity & \cref{app:controlled-moments,app:controlled-radial}\\
Mixture response & E2: separated/overlap outcomes and curvature counterexample & \cref{app:controlled-mixtures,app:controlled-curvature}\\
\bottomrule\end{tabularx}\end{table}

%% file: supplement_compact/sample_galleries.tex
\phantomsection\label{app:sample-galleries}\paragraph{Selection of sample panels.}
All sample panels use cached outputs and the first eight stored samples (per-image seed IDs for neural models; array indices for Digits). Matched rows share the declared noise/label bank. Window panels use refit 0; AFHQ retains all five refits and disabled corrections. Digits is an exact latent surrogate, separate from pretrained neural sampling. These qualitative examples add no replications, equivalence tests or evidence against mode collapse. Pixel hashes and source tags accompany the compact plotting inputs.

%% file: supplement_compact/theory.tex
% Compact submission theory. Full historical derivations remain in supplement_sections/.
% Input once, after the main paper's theorem statements and before compact experiments.
\input{supplement_compact/visuals/controlled}
\input{supplement_compact/theory/01_foundations}

\input{supplement_compact/theory/02_affine}

\input{supplement_compact/theory/03_solvers}

\input{supplement_compact/theory/04_nonlinear}
\input{supplement_compact/controlled_checks}
\input{supplement_compact/theory/05_models}

%% file: supplement_compact/visuals/controlled.tex
% Individually placed blocks keep each unchanged plot/caption next to its protocol.
\newcommand{\ControlledOrderFigure}{%
\begin{figure}[!htbp]\centering
\begin{minipage}{.34\linewidth}\centering
\includegraphics[width=\linewidth]{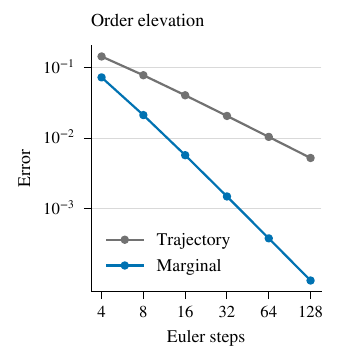}
\end{minipage}\hfill
\begin{minipage}{.62\linewidth}\raggedright
E0 gives trajectory/marginal tail slopes 0.984/1.968 over six Euler resolutions. This confirms the separation between first-order particle error and second-order law error for the analytic expansion--rotation fixture. It checks this example, not the existence of a correction for every path.
\end{minipage}
\caption{Controlled marginal-order check. Trajectory and law-error curves isolate the marginal-order prediction. Solver-root and curvature diagnostics accompany their corresponding proof blocks. This deterministic fixture has no sampling interval.}\label{fig:controlled-order-response}
\end{figure}
}
\newcommand{\ControlledSolverRootFigure}{%
\begin{figure}[!htbp]\centering
\begin{minipage}{.34\linewidth}\centering
\includegraphics[width=\linewidth]{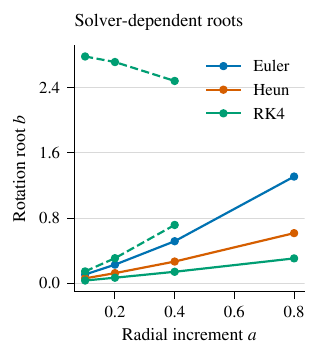}
\end{minipage}\hfill
\begin{minipage}{.62\linewidth}\raggedright
All 18 recorded Euler/Heun/RK4 root rows are retained; secondary RK4 branches are dashed. Solver-root covariance residuals are $\leq 2.9\!\times\!10^{-13}$. Different useful roots implement the same intended radial increment with different numerical response maps.
\end{minipage}
\caption{Controlled solver-response check. The recorded roots accompany \cref{cor:bounded-frequency-branches}. They are deterministic roots, not sampling intervals or matched-cost comparisons.}\label{fig:controlled-solver-roots}
\end{figure}
}
\newcommand{\ControlledCurvatureFigure}{%
\begin{figure}[!htbp]\centering
\begin{minipage}{.34\linewidth}\centering
\includegraphics[width=\linewidth]{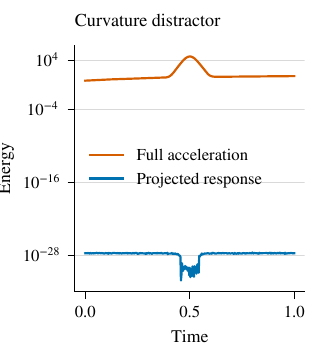}
\end{minipage}\hfill
\begin{minipage}{.62\linewidth}\raggedright
E2 combines full acceleration energy $4.9\!\times\!10^4$ with projected energy $3.0\!\times\!10^{-28}$. All 401 recorded times remain. This verifies the falsification example in \cref{prop:curvature-distractor}: full curvature is not intrinsically law-visible error.
\end{minipage}
\caption{Controlled curvature counterexample. Large full acceleration coexists with negligible projected response. This deterministic fixture has no sampling interval.}\label{fig:controlled-curvature}
\end{figure}
}

\newcommand{\ControlledObstructionFigure}{%
\begin{figure}[H]\centering
\includegraphics[width=\linewidth]{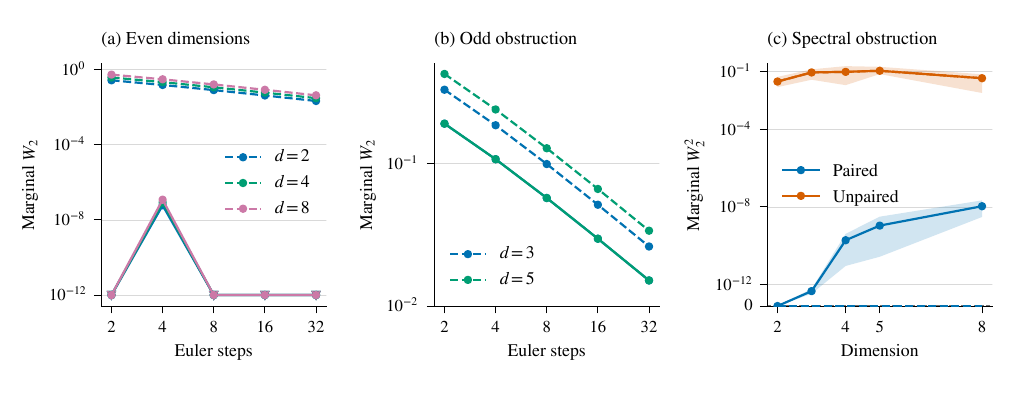}
\caption{Finite-step reachability and spectral obstruction across the retained dimensions and fixtures. Even exactness, odd residuals and paired spectral costs are distinct tests. Floating-point display floors do not turn numerical zeros into positive physical errors; fixture ranges are descriptive.}\label{fig:finite-h-obstruction}
\end{figure}
}

\newcommand{\ControlledMixtureFigure}{%
\begin{figure}[H]\centering
\begin{minipage}{\linewidth}\raggedright\normalsize
\paragraph{Mixture protocol.}\label{app:e2b-executed}\label{app:e2c-stress}
The hard GMM has eight unequal-weight components, linear centers and scales $\varsigma_m(t)=\varsigma_{m,0}e^{c_mt^2}$; its exact responsibility-weighted velocity/component rotations define the path. Random/uniform/curvature controls match per-step empirical gauge energy. Overlap stress uses scale 5, 256 calibration particles and a four-function time basis, frozen before eight held-out seeds.
\end{minipage}\par\smallskip
\includegraphics[width=\linewidth]{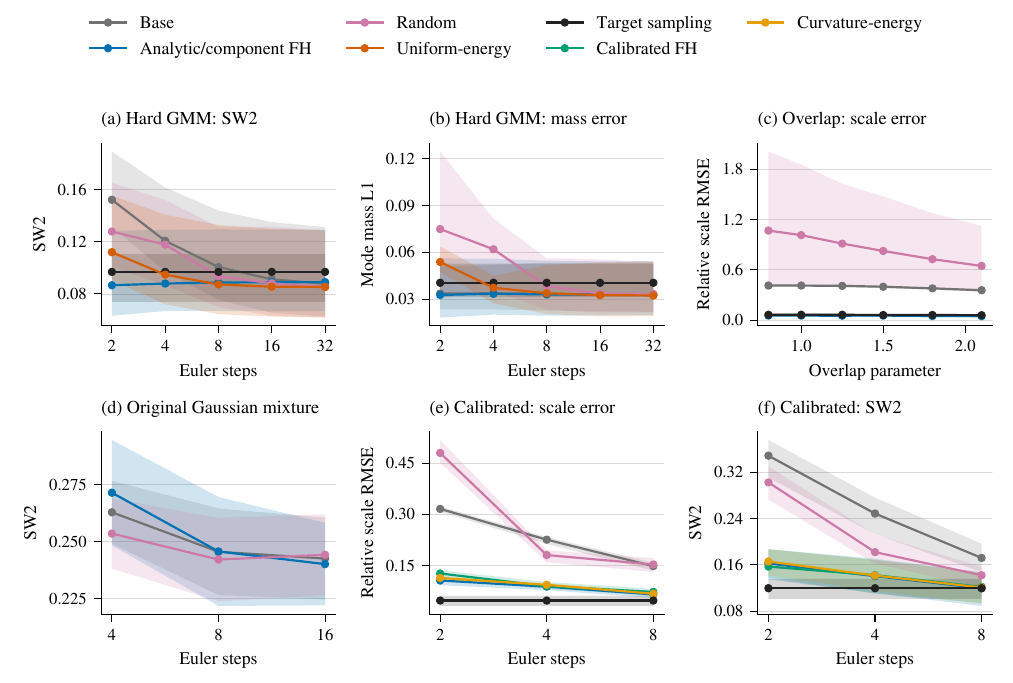}
\caption{Held-out mixture checks across actual integration budgets. Hard and overlapping mixtures retain all matched controls and negative outcomes, including competitive curvature and the failure of extra calibration to improve on analytic FH. Panels (a)--(c) show descriptive ranges over four seeds; (d) shows supplied mean $\pm$ SE over five seeds; (e)--(f) show ranges over five seeds. These bands are separate from neural paired-FID intervals.}\label{fig:controlled-mixture-response}
\end{figure}
}

\newcommand{\ControlledDigitsFigure}{%
\begin{figure}[H]\centering
\begin{minipage}{\linewidth}\raggedright\normalsize
E3 fits PCA/class-isotropic Gaussians to public Digits \citep{alpaydin1998digits,pedregosa2011scikit}: an exact specified surrogate, not the true image law. Ten stratified refits each average three particle seeds before split-level uncertainty. Dimensions 8/16/32 and all controls/negative outcomes remain (Table~\ref{tab:controlled-digits}).\label{app:e3a-digits}\label{app:e3b-transfer}
\end{minipage}\par\smallskip
\includegraphics[width=\linewidth]{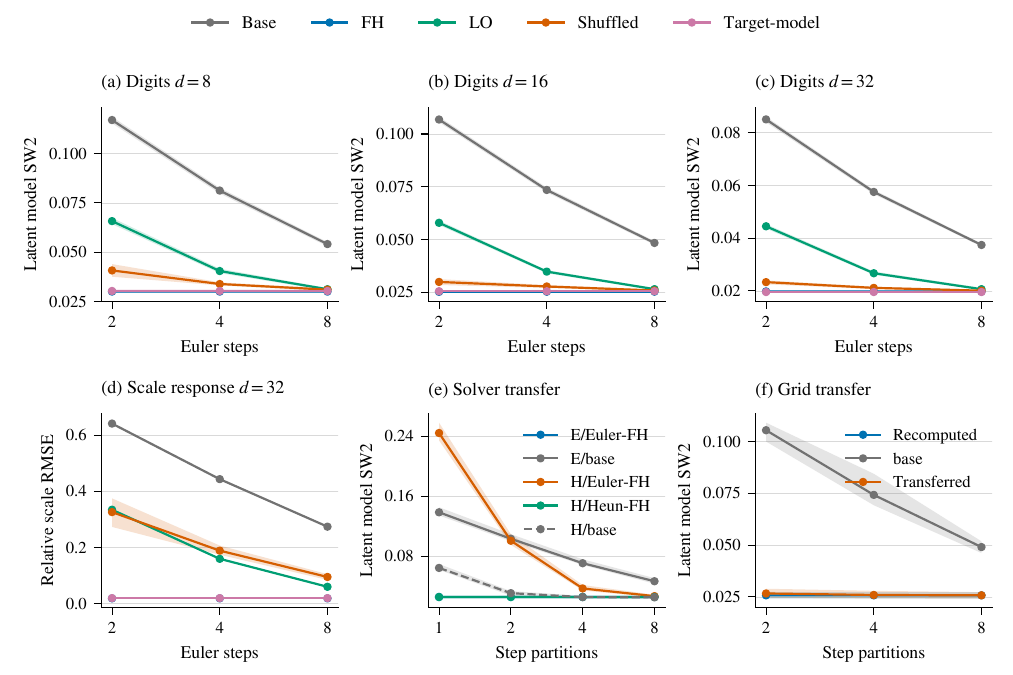}
\caption{Data-derived exact latent Digits models: dimensions 8/16/32 over integration budgets, followed by receiving-solver and grid transfer. Panels (a)--(d) use supplied 95\% split-bootstrap intervals after averaging three particle seeds within each of ten independent refits; solver/grid panels use descriptive ranges over 12 fixed seeds. These are surrogate-model integration errors, not true-image-law guarantees or matched-call solver comparisons.}\label{fig:digits-latent-robustness}
\end{figure}
}

%% file: supplement_compact/theory/01_foundations.tex
\section{Projection, Stability, and Solver-Aware Cancellation}
\label{app:hodge}\label{app:response-guide}
The state space is a flat torus, or $\R^d$ under the stated moment, decay, and complete-flow assumptions. Fix a positive smooth path ${\rho}_t$ and write $\mathcal H=L^2({\rho}_t\,\dd t)$. Integrations by parts require integrability; the formulas do not cover curved manifolds or unspecified retractions. \Cref{fig:controlled-order-response} checks the predicted numerical order.

\subsection{Admissible families and one-step response}\label{app:admissible-response}

At a fixed time, define
\begin{equation}\label{eq:fiber}
\begin{aligned}
\cS_\rho&=\{u\in L^2(\rho):\Div(\rho u)=0\},\\
\cF_{\rho,\dot\rho}&=\{w:-\Div(\rho w)=\dot\rho\}=v+\cS_\rho.
\end{aligned}
\end{equation}
For an interval $I$, choose an admissible family containing $u=0$
\[
\Uad\subset\{u=(u_s)_{s\in I}:u_s\in\cS_{\rho_s}\},
\]
with unique flows and finite second moments. Each member is a \emph{time-dependent field}: Euler uses $u_t$, while multistage methods may evaluate other times and positions. Low-rank or affine restrictions belong to $\Uad$.

\begin{definition}[Numerical response and reachable marginals]\label{def:response}
Fix a solver, $t$, and $h$ with $[t,t+h]\subset I$. Its response to a control $u\in\Uad$ is given by \cref{eq:response} in the main text.
The corresponding reachable set and best local error are
\begin{align}
\Reach_{t,h}&=\{\Response_{t,h}(u):u\in\Uad\},\label{eq:reachable}\\
\varepsilon_{t,h}&=\inf_{u\in\Uad}W_2(\Response_{t,h}(u),\rho_{t+h}).\label{eq:best-local}
\end{align}
\end{definition}
The numerical image need not be smooth, convex, or closed; no global Wasserstein logarithm is used. \Cref{prop:response-compact} gives an attainment criterion.

For the numerical law $\rho_N^u$ on a mesh $\{t_k\}$, the explicit propagation bound is
\begin{equation}\label{eq:response-telescope}
W_2(\rho_N^u,\rho_T)\le
\sum_{k=0}^{N-1}\left(\prod_{j=k+1}^{N-1}\ell_j\right)
W_2(\Response_{t_k,h_k}(u),\rho_{t_{k+1}}).
\end{equation}
Step $k$ is $\ell_k$-Lipschitz; the bound is for transport error.

\begin{equation}\label{eq:hodge}
 L^2(\rho;\R^d)=\cT_\rho\oplus\cS_\rho,
 \qquad \cT_\rho=\overline{\{\nabla\phi\}}^{L^2(\rho)}.
\end{equation}
Write $\Pi_\rho$ for the orthogonal projection onto $\mathcal T_\rho$. 
\begin{proposition}[Projection and conservation]
\label{prop:poisson-existence}\label{prop:gauge-invariance-app}\label{lem:negative-norm}\label{prop:one-dim-no-gauge}
On a connected flat torus, if $0<c\le {\rho}\le C$, each $z\in L^2({\rho})$ has a unique zero-mean $H^1$ potential satisfying
\begin{equation}\label{eq:poisson-weak-app}
 \int {\rho}\nabla\phi\cdot\nabla\psi=\int {\rho}z\cdot\nabla\psi\quad(\psi\in H^1).
\end{equation}
Its gradient is $\Pi_{\rho}z$, $z-\nabla\phi\in\mathcal S_{\rho}$, and
\begin{equation}\label{eq:negative-norm-identity}
 \|\Div({\rho}z)\|_{\dot H^{-1}({\rho})}=\|\Pi_{\rho}z\|_{L^2({\rho})},\qquad
 \|{\mathfrak f}\|_{\dot H^{-1}({\rho})}=\sup_{\|\nabla\psi\|_{L^2({\rho})}\le1}|\langle {\mathfrak f},\psi\rangle|.
\end{equation}
The potential result also holds on $\R^d$ under a weighted Poincar\'e inequality and appropriate decay. If $v$ transports ${\rho}$ and $u\in\mathcal S_{\rho}$, then $v+u$ satisfies the same continuity equation; complete unique flows and continuity-equation uniqueness identify its pushforward with ${\rho}$. In dimension one, a gauge with ${\rho}u\to0$ at both infinities is zero.
\end{proposition}
\begin{proof}
The weighted Dirichlet form is coercive on zero-mean $H^1$ by Poincar\'e, and its right side is bounded by Cauchy--Schwarz. Lax--Milgram gives~\eqref{eq:poisson-weak-app}; its residual is orthogonal to every gradient, proving the decomposition. Conversely gradient orthogonality is exactly weak zero weighted divergence. The dual norm in~\eqref{eq:negative-norm-identity} is bounded above by $\|\Pi_{\rho}z\|$; normalized potentials, or smooth gradient approximations, attain the bound. Linearity gives $\partial_t{\rho}+\Div({\rho}(v+u))=0$. Finally $\partial_x({\rho}u)=0$ implies ${\rho}u$ constant, and decay makes this constant zero.
\end{proof}

\begin{lemma}[Infinitesimal transport norm]\label{lem:w2-tangent}
On the torus let ${\rho}\ge c>0$ be smooth. If probability densities satisfy ${\rho}_\epsilon-{\rho}=\epsilon {\mathfrak f}+r_\epsilon$, $\|r_\epsilon\|_{\dot H^{-1}({\rho})}=o(|\epsilon|)$ and $\|{\rho}_\epsilon-{\rho}\|_\infty=o(1)$, with ${\mathfrak f}=-\Div({\rho}z)$, then
\begin{equation}\label{eq:w2-chart}
 W_2({\rho}_\epsilon,{\rho})=|\epsilon|\,\|\Pi_{\rho}z\|_{L^2({\rho})}+o(|\epsilon|).
\end{equation}
\end{lemma}
\begin{proof}
Set $\delta={\rho}_\epsilon-{\rho}$ and solve $-\Div({\rho}\nabla\phi)=\delta$. The interpolation ${\rho}+s\delta$ with constant flux ${\rho}\nabla\phi$ has dynamic transport action at most $(1+o(1))\int {\rho}|\nabla\phi|^2$. Thus $W_2\le(1+o(1))\|\delta\|_{\dot H^{-1}({\rho})}=O(|\epsilon|)$ \citep{benamou2000computational,ambrosio2008gradient}. For a smooth periodic $\psi$ and an optimal coupling, Taylor expansion along a shortest geodesic yields
$|\int\psi\delta|\le\|\nabla\psi\|_{L^2({\rho})}W_2+\frac12\|D^2\psi\|_\infty W_2^2$.
Divide by $|\epsilon|$, take the supremum over smooth unit test gradients, and apply~\eqref{eq:negative-norm-identity}. The lower and upper limiting coefficients agree.
\end{proof}

\subsection{Finite-time stability and order}\label{sec:leak}\label{app:superconv}
Let ${\rho}_t\in\mathcal P_2$ solve the continuity equation for $v$ on $\R^d$. Put ${z^{\rm vis}}_t=\Pi_{{\rho}_t}(w_t-v_t)$ and $\epsilon_t=\|{z^{\rm vis}}_t\|_{L^2({\rho}_t)}$. Assume ${z^{\rm vis}}$ has sufficient smoothness/growth that $w-{z^{\rm vis}}$ has a complete unique flow representing its continuity equation, $w$ has a complete flow with finite second moments, and $\epsilon_t$ is integrable. Assume an integrable one-sided bound $(w_t(x)-w_t(y))\cdot(x-y)\le L_t|x-y|^2$.
\begin{theorem}[Projected-residual stability]\label{thm:leak}
The exact law ${\widetilde\rho}_t$ of $w$ started at ${\rho}_0$ obeys
\begin{equation}\label{eq:leakfull}
 W_2({\widetilde\rho}_T,{\rho}_T)\le\int_0^T\exp\!\left(\int_t^TL_s\dd s\right)\epsilon_t\dd t.
\end{equation}
On a torus the same conclusion holds with a nonnegative Lipschitz bound on periodic lifts.
\end{theorem}
\begin{proof}
Projection preserves weighted divergence, so $\bar w=w-{z^{\rm vis}}$ transports ${\rho}$. Couple $X$ following $w$ and $Y$ following $\bar w$ from the same initial draw. Writing ${d_{\rm coup}}^2=\E|X-Y|^2$, differentiation and Cauchy--Schwarz give $\frac12({d_{\rm coup}}^2)'\le L_t{d_{\rm coup}}^2+{d_{\rm coup}}\epsilon_t$. Apply scalar Gronwall to $\sqrt{{d_{\rm coup}}^2+\delta^2}$; for negative $L_t$ regularization adds at most $|L_t|\delta$. Let $\delta\downarrow0$, then $W_2\le {d_{\rm coup}}$. On the torus geodesic displacement is bounded by displacement of the coupled periodic lifts.
\end{proof}

\begin{proposition}[Finite-step image and error propagation]
\label{prop:response-compact}\label{prop:response-global}
If $\Theta$ is nonempty compact and $\theta\mapsto T_\theta$ is continuous in $L^2(\rho)$, its image of pushforward laws is $W_2$-compact. Distance to any target attains its minimum, which is zero exactly on this image. Without compactness, zero infimum means membership in its closure. For numerical maps with Lipschitz constants $\ell_k$ and exact-path local errors $\delta_k$, the terminal error from the exact initial law is at most $\sum_k\delta_k\prod_{j>k}\ell_j$. If $\ell_k\le1+Lh_k$, $h_k>0$, and $\delta_k\le Ch_k^{q+1}$, it is at most $CTe^{L_+T}|\pi|^q$, where $L_+=\max(L,0)$ and $|\pi|=\max h_k$.
\end{proposition}
\begin{proof}
A common-source coupling gives $W_2(T_{\theta\#}\rho,T_{\vartheta\#}\rho)\le\|T_\theta-T_\vartheta\|_{L^2(\rho)}$. Continuity and compactness give the first claims; distance zero characterizes closure. Splitting through the numerical evolution of the exact marginal gives $e_{k+1}\le\ell_ke_k+\delta_k$. Iterate, bound products by $e^{L_+T}$, and use $\sum h_k^{q+1}\le T|\pi|^q$.
\end{proof}

\subsection{Full modified-equation order statement}\label{app:modified-order}
For a method of classical order ${p_{\rm meth}}$, the finite modified-field expansion is
\begin{equation}\label{eq:modified}
w_h=w+\sum_{j={p_{\rm meth}}}^{{\ell_{\rm mod}}}h^jE_j[w]+h^{{\ell_{\rm mod}}+1}\mathcal Q_h[w].
\end{equation}
\begin{definition}[First active modified index]\label{def:marginal-order}
For a fixed path and representative, set
\begin{equation}\label{eq:active-index}
q(w)=\min\{{p_{\rm meth}}\le j\le {\ell_{\rm mod}}:\Proj_{\rho_s}E_j[w](s)\not\equiv0\text{ on }I\}.
\end{equation}
If no coefficient is active, record only $q(w)>{\ell_{\rm mod}}$.
\end{definition}
\begin{theorem}[Gauge-induced marginal order elevation]\label{thm:superconv}
Work on the flat torus with a smooth density path uniformly bounded below and a fixed representative $w$, independent of $h$. For $0<h\le h_0$ assume \cref{eq:modified} with ${\ell_{\rm mod}}\ge q$ and, on a common time--space neighborhood,
\[
 \|w\|_{C^{q+4}}+\sum_{j=p_{\rm meth}}^{\ell_{\rm mod}}\|E_j[w]\|_{C^{q+4}}
 +\sup_{0<h\le h_0}\|\mathcal Q_h[w]\|_{C^{q+4}}\le C.
\]
Assume the numerical map agrees in $C^1$ with the modified flow up to $O(h^{q+2})$, with a constant independent of $h$. All asymptotics below are as $h\downarrow0$, with $[t,t+h]$ in this neighborhood. If, for some $q>p_{\rm meth}$,
\begin{equation}\label{eq:vertical-conditions}
\Div\!\left(\rho_tE_j[w](t)\right)=0,
\qquad j={p_{\rm meth}},\ldots,q-1,
\end{equation}
throughout a neighborhood of the step, then
\begin{align}
(\Psi_h^w)_\#\rho_t-\rho_{t+h}
&=-h^{q+1}\Div(\rho_tE_q[w])+o(h^{q+1}),\label{eq:density-q}\\
W_2((\Psi_h^w)_\#\rho_t,\rho_{t+h})
&=h^{q+1}\|\Proj_{\rho_t}E_q[w]\|_{L^2(\rho_t)}\notag\\
&\quad+o(h^{q+1}).\label{eq:w2-q}
\end{align}
If the numerical maps are uniformly \((1+Lh)\)-Lipschitz and these local bounds hold along the exact path, then for \(Nh=T\),
\begin{equation}
W_2(\rho_N^{\rm num},\rho_T)=O(N^{-q}).
\end{equation}
\end{theorem}
\begin{proof}[Proof of \cref{thm:superconv}]
Write $w_h=w+\sum_{j={p_{\rm meth}}}^{\ell_{\rm mod}} h^jE_j+h^{{\ell_{\rm mod}}+1}\mathcal Q_h$. Since $\Div({\rho}_sE_j)=0$ for $j<q$, the prescribed density has modified-field residual $h^q\Div({\rho}_sE_q)+O(h^{q+1})$. If $\widetilde {\rho}$ is the modified-flow law and $\eta=\widetilde {\rho}-{\rho}$, then
$\partial_s\eta+\Div(w_h\eta)=-h^q\Div({\rho}_sE_q)+O(h^{q+1})$, $\eta_t=0$.
With $U_h$ the homogeneous transport propagator, variation of constants gives
\begin{equation}\label{eq:duhamel-superconv}
 \eta_{t+h}=-h^q\int_t^{t+h}U_h(t+h,s)\Div({\rho}_sE_q(s))\dd s+O(h^{q+2})
 =-h^{q+1}\Div({\rho}_tE_q(t))+O(h^{q+2}).
\end{equation}
The second equality uses the uniform smooth coefficient and flow bounds; these also give the required $L^\infty$ and negative-norm remainders. A $C^1$ map mismatch $O(h^{q+2})$ gives a law mismatch of that order by common-source coupling and smooth inverse-map change of variables. Apply~\eqref{eq:w2-chart}, based at ${\rho}_{t+h}$; smoothness changes the tangent norm from ${\rho}_t$ by $O(h)$. \Cref{prop:response-global} gives global order from local errors on the prescribed path alone.
\end{proof}
A \emph{single} feasible $w$ must solve all of \cref{eq:vertical-conditions}; general solvability is not asserted. \Cref{prop:fixed-euler-barrier,prop:acceleration-moment-barrier} delimit Euclidean feasibility without contradicting this torus result.

\begin{corollary}[Approximate cancellation budgets]\label{cor:tolerance}
Let $w^{(h)}$ be admissible representatives of the same prescribed path $\rho_t$. Assume the main order theorem's smooth coefficient/remainder and flow bounds through order $q$, its $O(h^{q+2})$ map mismatch, the stability theorem's projected-field regularity, and $(1+Lh)$-Lipschitz numerical maps, all with constants uniform over this family as $h\downarrow0$. If $\sup_t\|\Pi_{{\rho}_t}E_j[w^{(h)}](t)\|\le\epsilon_j(h)$ for ${p_{\rm meth}}\le j<q$, then
\begin{equation}\label{eq:tolorder}
 W_2(\mu_N,{\rho}_T)\le C_T\left(h^q+\sum_{j={p_{\rm meth}}}^{q-1}h^j\epsilon_j(h)\right).
\end{equation}
Thus $\epsilon_j(h)=O(h^{q-j})$ suffices for order $q$; fixed small tolerances do not generally elevate asymptotic order. For a fixed $w$, a vanishing tolerance is possible only if the corresponding projected coefficient is zero; a nonzero solenoidal coefficient can have zero projection.
\end{corollary}
\begin{proof}
The projected modified-field difference is bounded by $\sum_{j<q}h^j\epsilon_j+Ch^q$. Equation~\eqref{eq:leakfull} on a length-$h$ interval bounds the local law error by $Ch$ times this quantity; add map mismatch and propagate over $T/h$ steps.
\end{proof}
Frequencies of size $1/h$ need not satisfy the uniform bounds. Unlike classical invariant-measure results \citep{debussche2012weak,abdulle2014invariant,laurent2022invariant,bronasco2025exotic}, cancellation here concerns a prescribed time-dependent path and feasible representatives.

\subsection{Optimization and finite dictionaries}\label{app:general-opt}\label{sec:galerkin}\label{app:connection}\label{app:interpretation}
Choose a dense smooth/Sobolev domain $\mathcal X\subset\mathcal H$ for $E_j:\mathcal X\to\mathcal H$, with $\mathcal U=\mathcal X\cap\mathcal S_{\rho}$ dense in the closed feasible subspace. Differentiate in the $\mathcal X$ norm. Tested residuals must lie in the adjoint domain with an $\mathcal H$ representative; projected $L^2$ descent is smooth admissible only in $\mathcal X$. The leading objective is
\begin{equation}\label{eq:Br}
\cB_{p_{\rm meth}}(w)=\frac12\int_0^T\|\Proj_{\rho_t}E_{p_{\rm meth}}[w](t)\|_{L^2(\rho_t)}^2\,\dd t
\end{equation}
\begin{proposition}[Stationarity and linearized reachability]
\label{prop:generic-stationarity}\label{prop:linearized-cancel}\label{thm:stacked-range}\label{prop:euler-adjoint}
Let $b_j=\Pi_{\rho}E_j[w]$ and $\mathcal D_j=\Pi_{\rho}DE_j[w]|_{\mathcal U}$. For $\mathcal B_{p_{\rm meth}}=\frac12\|b_{p_{\rm meth}}\|_{\mathcal H}^2$,
\begin{equation}\label{eq:generic-firstvar-app}
 D\mathcal B_{p_{\rm meth}}[w]u=\langle DE_{p_{\rm meth}}[w]^*b_{p_{\rm meth}},u\rangle_{\mathcal H},\qquad
 w\text{ stationary}\ \Longleftrightarrow\ P_{\mathcal S_{\rho}}DE_{p_{\rm meth}}[w]^*b_{p_{\rm meth}}=0.
\end{equation}
At the linearized level only, one correction cancels the residuals ${p_{\rm meth}},\ldots,q-1$ iff $-b=(-b_{p_{\rm meth}},\ldots,-b_{q-1})$ belongs to the range of the stacked operator $\mathcal D=(\mathcal D_{p_{\rm meth}},\ldots,\mathcal D_{q-1})$; the infimal squared residual is $\dist^2(-b,\overline{\operatorname{Range}\mathcal D})$. If the range is closed the least-squares infimum is attained. For Euler's acceleration action $\frac12\|\Pi_{\rho}a_w\|_{\mathcal H}^2$, endpoint-zero variations have adjoint representative
\begin{equation}\label{eq:euler-firstvar-explicit}
 z=-D_t^wb+(\nabla w)^\top b,\qquad b=\Pi_{\rho}a_w,\qquad D_t^w=\partial_t+w\cdot\nabla;
\end{equation}
stationarity is $P_{\mathcal S_{\rho}}z=0$, assuming $z\in\mathcal H$.
\end{proposition}
\begin{proof}
The projection is fixed along the gauge fiber. Differentiate its squared norm, use its self-adjointness and the adjoint definition; density of $\mathcal U$ gives stationarity on the closure. The product first-order expansion is $b+\mathcal D u+o(\|u\|_{\mathcal X})$. Solvability is range membership and least squares is distance to its closure; a closed range contains its orthogonal projection. For Euler, $Da_w[u]=D_t^wu+(\nabla w)u$. The continuity equation gives
$\frac\dd{\dd t}\int {\rho}\,b\cdot u=\int {\rho}(D_t^wb\cdot u+b\cdot D_t^wu)$.
Integrate in time, use zero endpoint values, and transpose the remaining Jacobian term to obtain~\eqref{eq:euler-firstvar-explicit}.
\end{proof}

\begin{proposition}[Observed and unresolved residuals; nonlinear correction]
\label{prop:galerkin}\label{prop:dictionary}
For independent test gradients $\nabla\phi_i$, let $(G_{\rm Gram})_{ij}=\langle \nabla\phi_i,\nabla\phi_j\rangle_{\rho}$, $b_i=\langle z,\nabla\phi_i\rangle_{\rho}$ and $P_M$ project onto their span. Then
\begin{equation}
 \|\Pi_{\rho}z\|^2=b^\top {G_{\rm Gram}}^{-1}b+\|\Pi_{\rho}z-P_Mz\|^2.
\end{equation}
For a $C^1$ finite residual $F:\R^m\to\R^n$, suppose $DF$ is $L$-Lipschitz and $DF(\theta_0)$ has a right inverse of norm at most $\beta$. Set $\epsilon=\|F(\theta_0)\|$. If its domain contains the radius-$2\beta\epsilon$ ball and $4L\beta^2\epsilon\le1$, some $\theta_*$ satisfies $F(\theta_*)=0$ and $\|\theta_*-\theta_0\|\le2\beta\epsilon$.
\end{proposition}
\begin{proof}
Normal equations give $P_Mz=\sum_i({G_{\rm Gram}}^{-1}b)_i\nabla\phi_i$. Since $P_Mz=P_M\Pi_{\rho}z$, Pythagoras proves the identity. For the second claim write the right inverse as ${\mathcal I_F}$ and use $T(y)=y-F(\theta_0+{\mathcal I_F}y)$ on $\|y\|\le2\epsilon$. Its derivative has norm at most $L\beta^2\|y\|\le1/2$, while $\|T(0)\|=\epsilon$. It is a contraction of this ball into itself; its fixed point gives $\theta_*$. If $\epsilon=0$, take $\theta_0$.
\end{proof}
Redundant gradients use a pseudoinverse. Unobserved orthogonal gradients and empirical Gram error prevent a full Hodge certificate. Nonlinear correction covers only the specified residuals in a feasible dictionary; stationarity alone does not define an Ehresmann connection.

\input{supplement_compact/checks/order}

%% file: supplement_compact/checks/order.tex
\FloatBarrier
\subsection{Controlled order verification}\label{app:controlled-order}
The analytic expansion--rotation fixture checks the predicted separation between particle and law convergence.
\ControlledOrderFigure
\FloatBarrier

%% file: supplement_compact/theory/02_affine.tex
\section{Affine Elliptical Response and Sharp Error}
\label{app:elliptical}\label{app:finite-gaussian-proof}\label{app:complementary-statements}
Fix one spherical generator $Z$. Scatter identifies laws without moments; transport costs require the stated moments. Nonorthogonal whitening preserves feasibility but changes Euclidean cost. \Cref{fig:finite-h-obstruction,app:controlled-reachability} give numerical Gaussian/Student counterparts.

\begin{lemma}[Factors, moments, and transport distance]
\label{lem:elliptic-factor}\label{prop:momentfree}\label{thm:elliptical-w2}\label{sec:momentfree}
If $Z$ is spherical and not almost surely zero, and $L_1,L_2$ are invertible, then $\law(m_1+L_1Z)=\law(m_2+L_2Z)$ iff $m_1=m_2$ and $L_1L_1^\top=L_2L_2^\top$. If additionally $\E Z=0$ and $\E ZZ^\top=I$, the same-generator laws with positive scatters $\Sigma_1,\Sigma_2$ satisfy
\begin{equation}\label{eq:elliptical-w2}
 W_2^2=|m_1-m_2|^2+\tr\Sigma_1+\tr\Sigma_2
 -2\tr(\Sigma_1^{1/2}\Sigma_2\Sigma_1^{1/2})^{1/2}.
\end{equation}
For ${\varsigma},{\widetilde\varsigma}\ge0$ and ${\zeta}\ge1$, $W_{\zeta}(\law({\varsigma}Z),\law({\widetilde\varsigma}Z))=|{\varsigma}-{\widetilde\varsigma}|(\E|Z|^{\zeta})^{1/{\zeta}}$ whenever this moment is finite. For a standard $d$-dimensional Student $Z_\nu$,
\begin{equation}\label{eq:student-moment}
 \E|Z_\nu|^{\zeta}=\nu^{{\zeta}/2}\frac{\Gamma((d+{\zeta})/2)\Gamma((\nu-{\zeta})/2)}{\Gamma(d/2)\Gamma(\nu/2)},\quad 0<{\zeta}<\nu.
\end{equation}
\end{lemma}
\begin{proof}
Spherical laws are centrally symmetric. Two distinct symmetry centers would imply invariance under a nonzero translation (compose the reflections), impossible for a probability law: its characteristic function is nonzero near zero, whereas translation invariance would force $e^{i\xi\cdot a}=1$ there. Thus centers agree. Every scalar projection $\xi^\top LZ$ has law $\|L^\top\xi\|Z_1$. A nondegenerate scalar law cannot be invariant under a positive dilation other than one: iterating its tail probabilities forces concentration at zero. Spherical $Z\not\equiv0$ has nondegenerate $Z_1$, so equality implies equality of all Gram quadratic forms. Conversely equal Grams give $L_2=L_1Q$ with orthogonal $Q$, hence equal laws.

For the cost formula set ${T_{\rm opt}}=\Sigma_1^{-1/2}(\Sigma_1^{1/2}\Sigma_2\Sigma_1^{1/2})^{1/2}\Sigma_1^{-1/2}$. It is positive and ${T_{\rm opt}}\Sigma_1{T_{\rm opt}}=\Sigma_2$. The coupling $Y-m_2={T_{\rm opt}}(X-m_1)$ has the target law by the factor identity. For any coupling of centered variables, $\E|{T_{\rm opt}}^{1/2}X-{T_{\rm opt}}^{-1/2}Y|^2\ge0$ bounds its cross moment by this coupling's value; expansion proves optimality and~\eqref{eq:elliptical-w2} \citep{gelbrich1990formula}. The scale formula follows from the common-$Z$ upper bound and reverse Minkowski lower bound. Finally integrate the Student radial density proportional to ${\varrho}^{d-1}(1+{\varrho}^2/\nu)^{-(d+\nu)/2}$, substitute ${\varrho}^2/\nu$, and evaluate the beta integral.
\end{proof}
For an invertible linear coordinate map $L$, common-coupling bounds give $\sigma_{\min}(L)W_2(\mu,\eta)\le W_2(L_\#\mu,L_\#\eta)\le\sigma_{\max}(L)W_2(\mu,\eta)$; the lower bound applies the upper one to $L^{-1}$. Student scatter is not covariance when $\nu\le2$: affine feasibility remains meaningful, but $W_2$ requires $\nu>2$ and fourth moments require $\nu>4$.

The fixed-shape path is $\rho_t=\law(m_t+\Sigma_t^{1/2}Z)$ with $\Sigma_t\succ0$; for covariance statements $\E Z=0$ and $\E ZZ^\top=I$. Fix $h\ne0$. All affine lifts have $w_t(x)=\dot m_t+B_t(x-m_t)$ and satisfy
\begin{equation}\label{eq:lyap}
B_t\Sigma_t+\Sigma_tB_t^\top=\dot\Sigma_t.
\end{equation}
At a selected time, write $\Sigma=\Sigma_t$, $\Sigma_+=\Sigma_{t+h}$. Euler reaches the next center only if $m_{t+h}=m_t+h\dot m_t$, and reaches the next centered law only if
\begin{equation}\label{eq:exactcov}
(I+hB)\Sigma(I+hB)^\top=\Sigma_+.
\end{equation}
Exactness requires both the tangent \cref{eq:lyap} and finite chord \cref{eq:exactcov}.

\begin{theorem}[Affine finite-step reachability]\label{thm:finite-gaussian}
In the fixed-shape family above, put $L=\Sigma^{1/2}$, ${\Sigma_\star}=L^{-1}\Sigma_+L^{-1}$ and ${H_\Sigma}=L^{-1}\dot\Sigma L^{-1}$. An Euler-exact affine tangent lift exists if and only if the center condition holds and
\begin{equation}\label{eq:orthogonal-equation}
\exists Q\in O(d):\quad {\Sigma_\star}^{1/2}Q+Q^\top {\Sigma_\star}^{1/2}=2I+h{H_\Sigma}.
\end{equation}
All its linear parts are
\begin{equation}\label{eq:BQ}
B_Q=h^{-1}L({\Sigma_\star}^{1/2}Q-I)L^{-1}.
\end{equation}
\end{theorem}
\begin{proof}[Proof of \cref{thm:finite-gaussian}]
An affine field has center velocity $\dot m$ and scatter tangent $B\Sigma+\Sigma B^\top=\dot\Sigma$. Let $L=\Sigma^{1/2}$, ${\widehat B}=L^{-1}BL$ and ${M_{\rm step}}=I+h{\widehat B}$. The tangent is ${\widehat B}+{\widehat B}^\top={H_\Sigma}$ and Euler's target scatter condition is ${M_{\rm step}}{M_{\rm step}}^\top={\Sigma_\star}$. Since ${\Sigma_\star}\succ0$, this is equivalent to ${M_{\rm step}}={\Sigma_\star}^{1/2}Q$, $Q\in O(d)$. Substitution gives precisely the main orbit--slice equation and $B_Q$. Euler's center is $m+h\dot m$; the factor lemma makes these necessary and sufficient for the complete fixed-generator law. The tangent remains necessary without moments: the linear ODE's scatter evolves by its matrix fundamental solution, and the factor lemma identifies that scatter uniquely.
\end{proof}
\Cref{eq:orthogonal-equation,thm:anisotropic} are \emph{feasibility equivalences}, not efficient solvers for arbitrary anisotropy. The following directional tests are necessary; block and planar cases admit explicit reconstructions.

\begin{proposition}[Additional finite-step tests and nodewise extension]
\label{prop:finite-directional}\label{prop:block-isotropic}\label{cor:grid-exact}\label{app:2d}
Every orbit--slice solution obeys $|x^\top(2I+h{H_\Sigma})x|\le2\sqrt{x^\top {\Sigma_\star}x}$ for unit $x$, and $\|2I+h{H_\Sigma}\|_F\le2\sqrt{\tr {\Sigma_\star}}$. If ${\Sigma_\star}=\lambda_\ell^2I$ and ${H_\Sigma}=2\beta_\ell I$ on common orthogonal blocks, block-diagonal solutions exist whenever $|c_\ell|\le1$, $c_\ell=(1+h\beta_\ell)/\lambda_\ell$, and strict-inequality blocks are even-dimensional. Exact nodewise lifts on a fixed grid give exact marginal laws at every grid point and can be extended to a smooth feasible continuous field.

In dimension two, diagonalize $S={H_\Sigma}/2=\diag(\beta_1,\beta_2)$, put $a_i=1+h\beta_i$, and express ${\Sigma_\star}$ in this basis. Exactness is equivalent to
\begin{equation}\label{eq:2d}
 z=(\Sigma_\star)_{11}-a_1^2=(\Sigma_\star)_{22}-a_2^2\ge0,\qquad (\Sigma_\star)_{12}^2=(a_1-a_2)^2z.
\end{equation}
If $a_1\ne a_2$, $K=\omega J$ has $\omega=(\Sigma_\star)_{12}/[h(a_1-a_2)]$; otherwise $(\Sigma_\star)_{12}=0$ and $\omega=\pm\sqrt z/|h|$.
\end{proposition}
\begin{proof}
Write $2I+h{H_\Sigma}={P_\Sigma}Q+Q^\top {P_\Sigma}$, ${P_\Sigma}={\Sigma_\star}^{1/2}$; Cauchy--Schwarz gives the directional bound, and the triangle inequality gives the Frobenius bound. On each common block take $Q=cI+\sqrt{1-c^2}J$ for an orthogonal complex structure, or $Q=\operatorname{sign}(c)I$ at equality. This proves sufficiency, not necessity among maps mixing blocks. Induction proves grid exactness. At a node ${\mathsf C_{\rm node}}_k=B_k\Sigma_k-\dot\Sigma_k/2$ is skew; smoothly interpolate these ${\mathsf C_{\rm node}}_k$ and set $B_t=(\dot\Sigma_t/2+{\mathsf C_{\rm node}}_t)\Sigma_t^{-1}$. Center velocities extend similarly. This does not promise uniform derivatives as grids vary. Finally direct multiplication of $\diag(a_1,a_2)+h\omega J$ gives diagonals $a_i^2+h^2\omega^2$ and off-diagonal $h\omega(a_1-a_2)$, proving the two-dimensional criterion and reconstruction.
\end{proof}

\subsection{Isotropic and spectral classifications}
\label{app:isotropic-proof}\label{app:spectral-pairing}\label{sec:spectral}
\begin{lemma}[Orthogonal complex structures]\label{lem:complex-structure-manifold}
The skew matrices $J$ satisfying $J^2=-I$ exist exactly in even dimension $2m$ and form $O(2m)/U(m)$, of dimension $m(m-1)$.
\end{lemma}
\begin{proof}
The pairs $(e,Je)$ are orthonormal and their complements are $J$-invariant. Iteration gives the canonical planar blocks, hence even dimension and transitive orthogonal conjugation. Orthogonal matrices commuting with the canonical $J$ are precisely complex-linear isometries, $U(m)$. Subtract its dimension $m^2$ from $\dim O(2m)=m(2m-1)$.
\end{proof}

\begin{theorem}[All-sign isotropic optimum]\label{thm:all-signs}\label{thm:parity-finite}
Let $h\ne0$, ${\varsigma},{\varsigma}_+>0$, $\beta=\dot {\varsigma}/{\varsigma}$, $a=|1+h\beta|$, $\lambda={\varsigma}_+/{\varsigma}$, and $d=2{n_{\rm pair}}+\varepsilon$, $\varepsilon\in\{0,1\}$. For a spherical source with covariance $I$, the minimum squared affine-tangent Euler error is
\begin{equation}\label{eq:odd-w2}
 |m_+-m-h\dot m|^2+{\varsigma}^2\{2{n_{\rm pair}}(a-\lambda)_+^2+\varepsilon(a-\lambda)^2\}.
\end{equation}
Exactness requires the mean chord, and either $\lambda=a$ with $K=0$, or $\lambda>a$ in even dimension with
\begin{equation}\label{eq:finite-exact-gauge}
 B=\beta I+K,\qquad K=\frac{\sqrt{\lambda^2-a^2}}{|h|}J,\quad J^\top=-J,\quad J^2=-I.
\end{equation}
For $\lambda<a$, no correction can reduce the excess radius.
\end{theorem}
\begin{proof}
The tangent condition forces $B=\beta I+K$ with $K$ skew, so Euler's scatter is ${\varsigma}^2(a^2I-h^2K^2)$. Canonical skew blocks give paired radii $\sqrt{a^2+h^2\omega_j^2}\ge a$ and an unavoidable radius $a$ in odd dimension. Formula~\eqref{eq:elliptical-w2} reduces the isotropic-target cost to squared radius differences; each pair is optimally assigned $\max(a,\lambda)$. Equality forces the stated complex structures or zero. The formula includes negative chords and backward steps. Also $[\lambda^2-(1+h\beta)^2]/h^2=\ddot {\varsigma}/{\varsigma}+O(h)$ for a smooth scale. The same canonical blocks show $\min_{K^\top=-K}\|K^2+\kappa I\|_F^2$ equals zero for $\kappa>0$ in even dimension, $\kappa^2$ for $\kappa>0$ in odd dimension, and $d\kappa^2$ for $\kappa\le0$.
\end{proof}

\begin{proposition}[Three steps remove the odd terminal obstruction]\label{prop:odd-three-step}
Let $d\ge3$ be odd, $\beta,h>0$, and $Z$ a nonzero spherical random vector in $\mathcal P_2$. For $\rho_t=\law(e^{\beta t}Z)$, allow time-dependent feasible affine fields $(\beta I+K(t))x$, $K(t)^\top=-K(t)$. Exact \emph{terminal} law on a uniform Euler grid is impossible with one or two steps, but possible with three. The exact continuous path is preserved; the intermediate numerical marginals need not be.
\end{proposition}
\begin{proof}
Put $a=1+\beta h$ and $\mathsf F_j=aI+hK_j$. Its Gram has minimum eigenvalue $a^2$ of odd multiplicity and every eigenvalue above $a^2$ of even multiplicity. One-step exactness is therefore impossible by \cref{thm:all-signs}. For two-step exactness, the factor lemma implies $(\mathsf F_2\mathsf F_1)(\mathsf F_2\mathsf F_1)^\top=\lambda_2^2I$, $\lambda_2=e^{2\beta h}>a^2$. Thus
\[
 \mathsf F_2^\top\mathsf F_2=\lambda_2^2(\mathsf F_1\mathsf F_1^\top)^{-1}.
\]
The left minimum $a^2$ has odd multiplicity, while its right counterpart comes from the first Gram's maximum $\lambda_2^2/a^2>a^2$ of even multiplicity, a contradiction.

For three steps set $\lambda_3=e^{3\beta h}$, $b=\sqrt{\lambda_3/a}>a$, and $\omega=\sqrt{b^2-a^2}/h$. Partition the coordinates into three odd nonempty blocks $\mathcal I_j$ of sizes $1,1,d-2$. Let $\mathsf P_j$ project onto the complement of $\mathcal I_j$, and choose skew $J_j$ with $J_j^2=-\mathsf P_j$. Define
\[
\begin{aligned}
 \mathsf D_j&=a(I-\mathsf P_j)+b\mathsf P_j,\quad
 \mathsf U_j=(aI+h\omega J_j)\mathsf D_j^{-1},\\
 \mathsf Q_0&=I,\quad \mathsf Q_j=\mathsf Q_{j-1}\mathsf U_j,\quad
 K_j=\omega\mathsf Q_{j-1}J_j\mathsf Q_{j-1}^\top.
\end{aligned}
\]
Each $\mathsf U_j$ is orthogonal and commutes with $\mathsf D_j$. Induction gives $\mathsf F_j\cdots\mathsf F_1=\mathsf Q_j\mathsf D_j\cdots\mathsf D_1$. Every coordinate receives one factor $a$ and two factors $b$, so the three-step product is $ab^2\mathsf Q_3=\lambda_3\mathsf Q_3$, giving the target law by spherical symmetry. Smooth skew interpolation through zero realizes the three nodal fields without increasing their maximum rank or norm. Its exact fundamental matrix is $e^{\beta t}O(t)$ with $O(t)$ orthogonal, hence it preserves the entire prescribed path.
\end{proof}
The nodal ranks are $d-1,d-1,2$, and $\omega^2\to3\beta^2/2$ as $h\downarrow0$. Derivatives may grow as $1/h$, excluding a uniform fixed-field order claim. More generally an $N$-step exact product with $\rank K_j\le2r$ necessarily has $2rN\ge d$: telescoping gives $\rank(\mathsf F_N\cdots\mathsf F_1-a^NI)\le2rN$, whereas exactness makes this difference $e^{N\beta h}Q-a^NI$ invertible. The rank count is only necessary.

At a covariance-normalized spherical source with isotropic covariance tangent $2\beta I$, let $B=\beta I+K$, $K^\top=-K$, $h\ne0$, $a=1+h\beta$, and let ${\Sigma_\star}\succ0$ be the target covariance. The centered Euler response is $a^2I-h^2K^2$.
\begin{theorem}[Spectral reachability]\label{thm:spectral-pairing}
An exact skew control exists if and only if
\begin{equation}\label{eq:spectral-condition}
{\Sigma_\star}\succeq a^2I,\qquad
{n}_{\Sigma_\star}(\lambda)\text{ is even for each }\lambda>a^2.
\end{equation}
On a $2{n}_\ell$-dimensional eigenspace of ${\Sigma_\star}$ with $\lambda_\ell>a^2$, all solutions have the form
\[
K_\ell=\frac{\sqrt{\lambda_\ell-a^2}}{|h|}J_\ell,
\quad J_\ell^\top=-J_\ell,\quad J_\ell^2=-I.
\]
On the eigenspace at $a^2$, $K=0$. Each nonzero-frequency choice is parameterized by $O(2{n}_\ell)/U({n}_\ell)$.
\end{theorem}
\begin{proof}[Proof of \cref{thm:spectral-pairing}]
The reachable scatter is ${\Sigma_\star}=a^2I-h^2K^2$. Every positive eigenvalue of $-K^2$ occurs in even-dimensional planar blocks. Conversely assign a scaled complex structure to each even-dimensional positive-excess eigenspace of ${\Sigma_\star}$. Since ${\Sigma_\star}$ is a polynomial in $K$, every solution preserves these eigenspaces; $K^2=0$ on the floor forces $K=0$ there. This proves the equivalence and product of $O(2{n})/U({n})$ parameterizations. Generic simple-spectrum perturbations above the floor are obstructed.
\label{cor:generic-aniso}
\end{proof}

\begin{theorem}[Sharp projection with rank and norm budgets]\label{thm:rank}
For spherical $Z$ with covariance $I$ and $h\ne0$, set $c=1+h\beta$, $a=|c|$, and constrain $\rank K\le2r$, $\|K\|_{\op}\le\Omega$, where $r$ is an integer, $0\le r\le\lfloor d/2\rfloor$, and $0\le\Omega\le\infty$. Let $y_1\le\cdots\le y_d$ be the positive target scatter radii, $n_0=d-2r$, $b=\sqrt{a^2+h^2\Omega^2}$ (possibly infinite), and ${\widehat y}_j=\clip((y_{n_0+2j-1}+y_{n_0+2j})/2,[a,b])$. Then
\begin{equation}\label{eq:rankcost}
 \min_K W_2^2(\law((cI+hK)Z),\law({\Sigma_\star}^{1/2}Z))
 =\sum_{i=1}^{n_0}(y_i-a)^2+\sum_{j=1}^r[(y_{n_0+2j-1}-{\widehat y}_j)^2+(y_{n_0+2j}-{\widehat y}_j)^2].
\end{equation}
An optimizer aligns target eigenspaces, uses zero on the first $n_0$ directions, and successive planar frequencies $\sqrt{{\widehat y}_j^2-a^2}/|h|$.
\end{theorem}
\begin{proof}
The sorted reachable radii are $(a,\ldots,a,{\widetilde y}_1,{\widetilde y}_1,\ldots,{\widetilde y}_r,{\widetilde y}_r)$ with $a\le {\widetilde y}_1\le\cdots\le {\widetilde y}_r\le b$; ${\widetilde y}_j=a$ covers smaller ranks. For fixed output radii $z_i$, Bures fidelity is at most $\sum_i y_iz_i$, attained by alignment. Indeed write it as $\|\diag(y)W\diag(z)\|_*$; its variational expression over orthogonal $Q$ has terms $y_iz_jW_{ij}Q_{ji}$. The inequality $uv\le(u^2+v^2)/2$ bounds this by two sums with doubly stochastic weights. Rearrangement bounds each by $\sum y_iz_i$. Thus the minimum is squared radius distance. Each pair minimizes independently at its clipped arithmetic mean; these means are already ordered, so they jointly attain the minimum. The displayed blocks realize them. Singular outputs at $a=0$ follow by continuity, or common-generator Procrustes coupling.
\end{proof}

\begin{corollary}[Sharp additional energy budget]\label{cor:energy-spectral-optimum}
Under \cref{thm:rank}, also constrain $\|K\|_F^2\le\mathcal E_{\max}$, $\mathcal E_{\max}\ge0$. Put $m_j=(y_{n_0+2j-1}+y_{n_0+2j})/2$. The same minimum formula \cref{eq:rankcost} holds with pair radii
\begin{equation}\label{eq:energy-pair-radii}
 z_j(\eta)=\clip\bigl(m_j/(1+\eta),[a,b]\bigr).
\end{equation}
Use $\eta=0$ if $\sum_j(z_j(0)^2-a^2)\le h^2\mathcal E_{\max}/2$; otherwise choose $\eta>0$ so that equality holds. If $\mathcal E_{\max}=0$, $r=0$, or $b=a$, all permitted pair radii are $a$. Exact reachability holds if and only if the conditions of \cref{cor:budgets} and the following energy inequality hold:
\begin{equation}\label{eq:energy-minimum-exact}
 \mathcal E_{\max}\ge\mathcal E_{\min}
 =\frac{\tr\Sigma_\star-da^2}{h^2}.
\end{equation}
\end{corollary}
\begin{proof}
Alignment does not change the skew frequencies or their energy, $\|K\|_F^2=2h^{-2}\sum_j(z_j^2-a^2)$. After alignment the convex problem minimizes $2\sum_j(z_j-m_j)^2$ over $[a,b]^r$ with this quadratic budget. A multiplier $2\eta$ gives derivative $4[(1+\eta)z_j-m_j]$, hence \cref{eq:energy-pair-radii}. The radii remain sorted. For a positive active budget the sum is continuous and nonincreasing in $\eta$, from its unconstrained value to zero; complementary slackness proves global optimality. Degenerate zero budgets give the stated floor directly. Taking the trace of exact scatter $a^2I-h^2K^2$ proves necessity of \cref{eq:energy-minimum-exact}; the canonical blocks from \cref{cor:budgets} attain this energy, proving sufficiency.
\end{proof}
The identity $\|K\|_F^2=\E|KZ|^2$ requires covariance-normalized isotropy; neural score energy and nonorthogonal whitening require their covariance factors.

\begin{corollary}[Minimal budgets and unrestricted pairing]
\label{cor:budgets}\label{thm:sharp-pairing}
Zero error requires and suffices that ${\Sigma_\star}\succeq a^2I$, $\lambda_{\max}({\Sigma_\star})\le a^2+h^2\Omega^2$, $\rank({\Sigma_\star}-a^2I)\le2r$, and every positive-excess eigenvalue has even multiplicity. The minimal budgets are $r_{\min}=\rank({\Sigma_\star}-a^2I)/2$ and $\Omega_{\min}=\sqrt{\lambda_{\max}({\Sigma_\star})-a^2}/|h|$. Without budgets, each pair costs
\begin{equation}\label{eq:pair-cost}
 g_a(y,z)=\tfrac12(y-z)^2+2(a-(y+z)/2)_+^2;
\end{equation}
ordered adjacent pairing is optimal, leaving $y_1$ unmatched in odd dimension. Above the floor in even dimension this is $\frac12\sum_j(y_{2j}-y_{2j-1})^2$.
\end{corollary}
\begin{proof}
Canonical skew form proves the zero-error conditions and realizes the minimal rank/largest frequency. Minimizing $({\varrho}-y)^2+({\varrho}-z)^2$ over ${\varrho}\ge a$ gives $g_a$. All pairings are allowed by conjugation; the preceding theorem supplies a global sorted optimizer, including below-floor targets. The above-floor formula follows when clipping is inactive.
\end{proof}

\subsection{Leading cancellation and covariance surrogates}
\label{app:anisotropic-proof}\label{app:anis-opt}\label{app:singular}
\begin{theorem}[Leading anisotropic feasibility]\label{thm:anisotropic}
Put ${H_\Sigma}=\Sigma^{-1/2}\dot\Sigma\Sigma^{-1/2}$ and ${G_\Sigma}=\Sigma^{-1/2}\ddot\Sigma\Sigma^{-1/2}$. Leading covariance cancellation is equivalent to ${\widehat B}+{\widehat B}^\top={H_\Sigma}$, $2{\widehat B}{\widehat B}^\top={G_\Sigma}$, hence to ${G_\Sigma}\succeq0$ and
\begin{equation}\label{eq:anis-orth}
 {P_G}Q+Q^\top {P_G}={H_\Sigma},\qquad {P_G}=({G_\Sigma}/2)^{1/2},\quad Q\in O(d).
\end{equation}
Necessarily $(x^\top {H_\Sigma}x)^2\le2x^\top {G_\Sigma}x$ for unit $x$, and $\tr {G_\Sigma}\ge\|{H_\Sigma}\|_F^2/2$. In the range/kernel splitting of ${G_\Sigma}$, write $\widehat B_{++},\widehat B_{+0}$ for the indicated blocks of $\widehat B$. Feasibility is equivalent to $(H_\Sigma)_{00}=0$, $\widehat B=\left(\begin{smallmatrix}\widehat B_{++}&\widehat B_{+0}\\0&0\end{smallmatrix}\right)$, $(H_\Sigma)_{+0}=\widehat B_{+0}$, $\widehat B_{++}+\widehat B_{++}^\top=(H_\Sigma)_{++}$, and $\widehat B_{++}\widehat B_{++}^\top=(G_\Sigma)_{++}/2-\widehat B_{+0}\widehat B_{+0}^\top$.
\end{theorem}
\begin{proof}
Whitening gives $B\Sigma+\Sigma B^\top=\dot\Sigma$ and $2B\Sigma B^\top=\ddot\Sigma$ exactly as the two equations for ${\widehat B}$. If ${\widehat B}=U\Lambda V^\top$ is an SVD, then ${P_G}=U\Lambda U^\top$ and $Q=UV^\top$ give ${\widehat B}={P_G}Q$, including singular ${G_\Sigma}$; the converse is immediate. Cauchy--Schwarz gives $|x^\top {H_\Sigma}x|^2=4|x^\top {\widehat B}^\top x|^2\le4x^\top {\widehat B}{\widehat B}^\top x$. Writing ${\widehat B}={H_\Sigma}/2+K$ yields $\tr {G_\Sigma}=2\|{\widehat B}\|_F^2\ge\|{H_\Sigma}\|_F^2/2$. Finally ${\widehat B}{\widehat B}^\top={G_\Sigma}/2$ forces kernel-indexed rows to vanish; block multiplication proves the stated singular criterion.
\end{proof}

\begin{proposition}[Optimization, amplitude, and orientation]
\label{prop:anis-approx}\label{prop:amplitude}\label{prop:orientation}\label{sec:amplitude}
For ${\widehat B}={H_\Sigma}/2+K$, $K$ skew, the surrogate $\frac12\|M\|_F^2$, $M=2{\widehat B}{\widehat B}^\top-{G_\Sigma}$, is coercive and attains a minimum; stationarity is $M{\widehat B}-{\widehat B}^\top M=0$. For fixed $S=S^\top,K$ and target ${\Sigma_\star}$, define $D_0=(I+hS)^2-{\Sigma_\star}$, $D_1=h^2[K,S]$, $D_2=-h^2K^2$. The amplitude objective $\frac12\|D_0+\alpha D_1+\alpha^2D_2\|_F^2+\lambda\alpha^2/2$ equals $F(0)+c_1\alpha+c_2\alpha^2+c_4\alpha^4$, where $c_1=\langle D_0,D_1\rangle$, $c_2=\langle D_0,D_2\rangle+\|D_1\|^2/2+\lambda/2$, $c_4=\|D_2\|^2/2$. For $K\ne0$, endpoints and roots of $4c_4\alpha^3+2c_2\alpha+c_1=0$ give its global minimum on $[0,\alpha_{\max}]$; if $[{\Sigma_\star},S]=0$, $\alpha_*^2=\clip(-c_2/(2c_4),[0,\alpha_{\max}^2])$.

For a covariance-normalized spherical source, $S=\beta I$ and $K^2=-\omega^2I$ make every orthogonal conjugate of $K$ produce the same law. If $S=\beta I+\epsilon {\Delta S}$ and $K_1^2=K_2^2=-\omega^2I$, their scatter difference is $h^2\epsilon[K_1-K_2,{\Delta S}]$. If both scatters are bounded below by $\mu^2I$ for $\mu>0$, their law distance is at most $|\epsilon|h^2\|{\Delta S}\|_{\op}\|K_1-K_2\|_F/\mu$.
\end{proposition}
\begin{proof}
Orthogonality gives $\|{\widehat B}\|_F^2=\|{H_\Sigma}/2\|_F^2+\|K\|_F^2$, while $\|{\widehat B}{\widehat B}^\top\|_F\ge\|{\widehat B}\|_F^2/\sqrt d$, proving coercivity. For skew $\delta K$, $\delta M=2(\delta K{\widehat B}^\top-{\widehat B}\delta K)$ and $\delta F=2\langle M{\widehat B}-{\widehat B}^\top M,\delta K\rangle$; its coefficient is skew, proving stationarity. Expanding the amplitude square gives the coefficients; its cubic term vanishes since $\tr([K,S]K^2)=0$. A compact-interval polynomial minimum occurs at an endpoint or derivative root. Commutation makes $c_1=0$, reducing to a quadratic in $\alpha^2$.

The orientation identities follow from Euler's Gram formula. The commutator norm is at most $2\|{\Delta S}\|_{\op}\|K_1-K_2\|_F$. Coupling principal square-root factors gives $W_2\le\|{\Sigma_a}^{1/2}-{\Sigma_b}^{1/2}\|_F$. For $X={\Sigma_a}^{1/2}-{\Sigma_b}^{1/2}$, the Sylvester identity ${\Sigma_a}^{1/2}X+X{\Sigma_b}^{1/2}={\Sigma_a}-{\Sigma_b}$ gives $2\mu\|X\|_F\le\|{\Sigma_a}-{\Sigma_b}\|_F$, completing the estimate.
\end{proof}
Covariance surrogates do not guarantee FID; spectral conjugation can change score energy: $\E|{M_{\rm sk}}s_\theta|^2=\tr({M_{\rm sk}}^\top {M_{\rm sk}}\,\E[s_\theta s_\theta^\top])$ depends on orientation unless the score covariance has suitable symmetry.

\begin{proposition}[Kinetic selection differs from numerical selection]\label{prop:separation}
For the two-dimensional Gaussian scale ${\varsigma}=e^{\beta t}$, $\beta>0$, the feasible fields $(\beta I+\omega J)x$ have kinetic action $2e^{2\beta t}(\beta^2+\omega^2)$ and projected acceleration $(\beta^2-\omega^2)x$. The kinetic minimum is $\omega=0$, the leading numerical zeros are $|\omega|=\beta$, and positive-step exactness uses $|\omega_h|=\sqrt{e^{2\beta h}-(1+\beta h)^2}/h=\beta+O(h)$.
\end{proposition}
\begin{proof}
The flow is a radial dilation times a rotation. Squaring its generator gives $(\beta^2-\omega^2)I+2\beta\omega J$, whose skew term is Gaussian-solenoidal. The action follows by tracing $B^\top B\Sigma$. Euler's scalar Gram is $(1+\beta h)^2+h^2\omega^2$; match the target and Taylor-expand.
\end{proof}

\input{supplement_compact/checks/reachability}

%% file: supplement_compact/checks/reachability.tex
\FloatBarrier
\subsection{Controlled reachability verification}\label{app:controlled-reachability}\label{app:e1-obstruction-exp}
The checks separate exact feasibility, a sharp residual, and numerical agreement with the constrained optimum; a small floating-point residual should not be interpreted as an empirical sampling floor.
\begin{table}[H]\centering\small
\caption{Controlled affine reachability results.}\label{tab:controlled-reachability}
\begin{tabularx}{\linewidth}{@{}lX@{}}\toprule Check & Retained result and boundary\\\midrule
E0: parity & Even-dimensional $W_2\leq 1.2\!\times\!10^{-7}$; odd-formula error $<4\!\times\!10^{-14}$.\\
E1: spectra, anisotropy & Pairing-cost median/max gap $2.2\!\times\!10^{-11}/1.4\!\times\!10^{-7}$; at directional violation 0.724, optimized cost 0.0791 exceeds lower bound 0.0656.\\
\bottomrule\end{tabularx}\end{table}

\ControlledObstructionFigure
The analytical suite also includes 384 fixed-generator Gaussian/Student rows ($\nu=3,5,10$; $d=2,3,8,9,32,33$; expanding/overexpanded chords; eight seeds of 32,768 particles). Population Bures costs are generator-independent in this setting, with even exactness, odd residuals and unrepairable overexpansion. Student 3 lacks a fourth moment; replicate ranges are descriptive, and empirical sliced-distance is not exact population $W_2$.
\FloatBarrier

%% file: supplement_compact/theory/03_solvers.tex
\section{Solver Classification and Transformation Boundaries}\label{app:solver-strata}
\label{app:normalforms}\label{app:block-solver}\label{app:integration-atlas}
These proofs support \cref{sec:solver}. Exact latent and fitted neural transfer tests appear in \cref{fig:digits-latent-robustness,app:neural-solver-rank}.
\subsection{Tangent response and whole-path restrictions}
\begin{proposition}[Tangent normal form and whole-path restriction]\label{prop:symmetry-strata}
\label{prop:normalform-app}\label{lem:commuting-necessary}\label{prop:relative-orientation}
At a covariance-normalized spherical source, a symmetric tangent $2S$ forces $B=S+K$, $K^\top=-K$. Euler's scatter and its derivative at $K_0$ are
\begin{equation}\label{eq:normalform-app}
 \Gamma(K)=I+2hS+h^2(S^2+[K,S]-K^2),\quad
 D\Gamma(K_0)[{\Delta K}]=h^2([{\Delta K},S]-K_0{\Delta K}-{\Delta K}K_0).
\end{equation}
At zero the kernel consists of skew rotations within repeated eigenspaces of $S$, of dimension $\sum_a m_a(m_a-1)/2$. An autonomous $S+K$ preserves the entire exponential elliptical path $\law(e^{tS}Z)$ iff $[K,S]=0$. Outside this commuting family, the separate spectra of $S$ and $K$ do not determine the response.
\end{proposition}
\begin{proof}
The covariance tangent is $B+B^\top=2S$; multiplication and differentiation give~\eqref{eq:normalform-app}. In an eigenbasis, $[{\Delta K},S]_{ij}=(\lambda_j-\lambda_i){\Delta K}_{ij}$, giving the kernel and dimension. If $[K,S]=0$, $e^{t(S+K)}=e^{tS}e^{tK}$ preserves the target by spherical symmetry. Conversely the autonomous scatter has second derivative $4S^2+2[K,S]$ at zero, while the target has $4S^2$; the factor lemma forces commutation. For the last claim take $S=\diag(1,2,3)$ and unit rotations $K_{12},K_{13}$. Both have skew spectrum $\{0,i,-i\}$, but $\|[K_{12},S]\|_F^2=2$ and $\|[K_{13},S]\|_F^2=8$, invariants of simultaneous conjugation.
\end{proof}

For $S=Q\diag(\lambda_aI_{m_a})Q^\top$, with orthogonal $Q$ and distinct eigenvalues $\lambda_a$ of multiplicity $m_a$, the kernel is
\begin{equation}\label{eq:centralizer}
 \mathfrak g_S=Q\Bigl(\bigoplus_a\mathfrak{so}(m_a)\Bigr)Q^\top.
\end{equation}
This kernel describes first-order sensitivity, not exact reachability.

\input{supplement_compact/checks/whole_path}

\subsection{Exact solver response and bounded branches}
\begin{theorem}[Commuting autonomous response strata]
\label{thm:block-classification}\label{thm:stability-classification}\label{thm:stability-app}\label{cor:odd-stability}
Let $S=Q\diag(\lambda_aI_{m_a})Q^\top$ and $h>0$. Let $R$ be a real polynomial or rational stability function without poles on evaluated matrices. Every autonomous lift of the whole path is $K=Q\diag(K_a)Q^\top$, $K_a$ skew. Put
\begin{equation}\label{eq:omega-set}
 \Omega_R(a)=\{b>0:R(a+ib)\text{ exists},\ |R(a+ib)|=e^a\}.
\end{equation}
Exactness of one $R$ step requires and suffices that every nonzero planar frequency in block $a$ satisfies $h|\omega|\in\Omega_R(h\lambda_a)$, and each zero direction satisfies $|R(h\lambda_a)|=e^{h\lambda_a}$. For fixed distinct frequencies and multiplicities, the full stratum is
\begin{equation}\label{eq:block-orbits}
 \prod_a O(m_a)\big/\left(O(m_{a,0})\times\prod_\ell U(m_{a,\ell})\right),
 \qquad m_{a,0}+2\sum_\ell m_{a,\ell}=m_a.
\end{equation}
Odd-dimensional strain blocks force a zero mode and its scalar equality. This is a blockwise obstruction within the commuting autonomous family, even if total dimension is even. Root sets may be infinite.
\end{theorem}
\begin{proof}
Commutation restricts $K$ to strain blocks. Reduce each skew block to planar rotations and a zero subspace. Identifying $xI+yJ$ with $x+iy$ makes each step's planar Gram $|R(h\lambda_a+ih\omega)|^2I$; zero directions have multiplier $|R(h\lambda_a)|^2$. The target factor is $e^{2h\lambda_a}$, so the factor lemma proves necessity and sufficiency. Orthogonal conjugation exhausts the fixed-frequency skew type. Its stabilizer is orthogonal on the zero space and unitary on each repeated nonzero-frequency space; distinct frequencies cannot mix. This gives~\eqref{eq:block-orbits} and the parity condition.
\end{proof}

For $a=h\lambda$ and ${\tau_{\rm freq}}=b^2$, Euler gives ${\tau_{\rm freq}}=e^{2a}-(1+a)^2$; implicit Euler gives ${\tau_{\rm freq}}=e^{-2a}-(1-a)^2$. Heun gives ${\tau_{\rm freq}}^2/4+(a+a^2/2){\tau_{\rm freq}}+(1+a+a^2/2)^2-e^{2a}=0$. Midpoint gives $(e^{2a}-1){\tau_{\rm freq}}=4[(1+a/2)^2-e^{2a}(1-a/2)^2]$, excluding poles. Substitute $a+ib$ and take squared moduli. At $a=0$, Euler and Heun have only zero as an exact rotational frequency, since their squared moduli are $1+b^2$ and $1+b^4/4$; midpoint allows every frequency. RK4 has $|R(ib)|^2=1+b^6(b^2-8)/576$, hence the additional exact frequency $\sqrt8/h$. This frequency is budget-specific; explicit calls and implicit solves have different costs.

\begin{corollary}[Bounded exact-frequency branches]\label{cor:bounded-frequency-branches}
Let $R$ be holomorphic near zero with real Taylor coefficients and
\[
 R(z)=e^z+c_Rz^{p_{\rm lin}+1}+O(z^{p_{\rm lin}+2}),\qquad
 c_R\ne0,\quad p_{\rm lin}\in\{1,2,\ldots\}.
\]
Here $p_{\rm lin}$ is the scalar linear order, not necessarily the nonlinear method order $p_{\rm meth}$. For $\lambda\ne0$ and small positive $h$, there are $\lceil p_{\rm lin}/2\rceil$ distinct positive real-analytic branches satisfying $|R(h(\lambda+i\omega_j(h)))|=e^{h\lambda}$, with
\begin{equation}\label{eq:bounded-frequency-limit}
 \omega_j(h)=|\lambda|\tan\frac{(2j+1)\pi}{2(p_{\rm lin}+1)}+O(h),\quad
 0\le j<\lceil p_{\rm lin}/2\rceil.
\end{equation}
\end{corollary}
\begin{proof}
The analytic quotient
\[
 \frac{R(h(\lambda+i\omega))R(h(\lambda-i\omega))-e^{2h\lambda}}
 {h^{p_{\rm lin}+1}}
\]
extends to $h=0$ with value $2c_R\operatorname{Re}(\lambda+i\omega)^{p_{\rm lin}+1}$. Its positive roots are the limits in \cref{eq:bounded-frequency-limit}. They are simple: in polar coordinates the angular factor is $\cos((p_{\rm lin}+1)\theta)$, with nonzero derivative at every listed root and $\dd\theta/\dd\omega\ne0$. Reversing $\lambda$ multiplies its real part by $(-1)^{p_{\rm lin}+1}$ and leaves the roots unchanged. The real-analytic implicit-function theorem supplies the branches; positivity and absence of nearby poles persist for small $h$.
\end{proof}
Euler and implicit Euler have $\omega/|\lambda|\to1$; Heun and midpoint have limit $1/\sqrt3$; RK4 has two limits $\tan(\pi/10)$ and $\tan(3\pi/10)$. These $h$-dependent branches give exact commuting-block steps (\cref{thm:block-classification}), not fixed-field order elevation. The argument excludes $\lambda=0$ and unbounded branches.

\FloatBarrier
\paragraph{Controlled root check.}\label{app:controlled-solver}
The recorded frequencies below test the solver-dependent branches directly.
\ControlledSolverRootFigure
\FloatBarrier

\subsection{Response collisions and transformation boundaries}
\begin{theorem}[First-response collision kernel]\label{thm:response-kernel}
Let
\begin{equation}\label{eq:solver-response-main}
 \Gamma_{R,h,S}(K)=R(h(S+K))R(h(S+K))^\top.
\end{equation}
For real analytic $R$ near the spectrum of $hS$ (or rational $R$ without poles), write $r_i=R(h\lambda_i)$. For skew ${\Delta K}$, $D\Gamma_{R,h,S}(0)[{\Delta K}]$ has zero diagonal and
\begin{equation}\label{eq:response-formula}
 [D\Gamma_{R,h,S}(0)[{\Delta K}]]_{ij}=
 \begin{cases}-(r_i-r_j)^2{\Delta K}_{ij}/(\lambda_i-\lambda_j),&\lambda_i\ne\lambda_j,\\0,&\lambda_i=\lambda_j.\end{cases}
\end{equation}
Thus its kernel is $\{{\Delta K}:[{\Delta K},R(hS)]=0\}$, with dimension $\sum_c n_c(n_c-1)/2$ for response-value multiplicities $n_c$.
\end{theorem}
\begin{proof}
Differentiating a matrix power gives $\sum_jD^jED^{k-1-j}$, whose entries in a diagonal basis are divided differences. Power series, or the inverse derivative for rational functions, therefore give $[DR(hS)[h{\Delta K}]]_{ij}=hR^{[1]}(h\lambda_i,h\lambda_j){\Delta K}_{ij}$ \citep{higham2008functions}. This derivative is skew since the divided difference is symmetric. Differentiating the Gram product gives $hR^{[1]}(h\lambda_i,h\lambda_j)(r_j-r_i){\Delta K}_{ij}$, which is~\eqref{eq:response-formula}. Vanishing entries are exactly equal-response blocks. Euler has no additional real collisions; Heun has collisions for distinct strains with $h(\lambda_i+\lambda_j)=-2$. Only the derivative cancels; higher response and whole-path restrictions remain.
\end{proof}

\begin{proposition}[What transformations preserve]
\label{prop:poly-no-go}\label{prop:splitting-boundary}\label{lem:affine-rk}\label{prop:gauge-conjugacy}\label{prop:clock}
A nonconstant real polynomial $R$ cannot have $|R(ib)|=1$ for every real $b$, whereas midpoint's Cayley map is orthogonal for every skew generator. If ${\mathcal G}_t$ preserves the source law, ${\mathcal G}_+$ preserves the target, and ${\mathcal G}_+$ and ${\mathcal G}_+^{-1}$ have uniformly bounded Lipschitz constants $L_+,L_-$, then
\begin{equation}
 L_-^{-1}W_2({\Psi}_{h\#}{\rho}_t,{\rho}_{t+h})\le
 W_2(({\mathcal G}_+\circ {\Psi}_h\circ {\mathcal G}_t)_\#{\rho}_t,{\rho}_{t+h})
 \le L_+W_2({\Psi}_{h\#}{\rho}_t,{\rho}_{t+h}).
\end{equation}
Fixed affine changes conjugate Runge--Kutta steps (assuming corresponding unique implicit stages). For a nonlinear diffeomorphism $T$, a gauge $a$ of $\nu$ becomes $u(Tz)=DT(z)a(z)$ for $T_\#\nu$, and exact maps conjugate, but ordinary Euler does not. Under a $C^2$ monotone clock $t=\tau(s)$, the field becomes $\widetilde v=\tau'v$ and acceleration becomes $\tau''v+(\tau')^2a_v$.
\end{proposition}
\begin{proof}
The polynomial identity $R(ib)R(-ib)-1=0$ would have a nonzero leading coefficient of degree $2\deg R$, a contradiction. For midpoint, transposition exchanges the commuting factors of $(I-hK/2)^{-1}(I+hK/2)$, proving orthogonality \citep{hairer2006geometric}. The source rearrangement disappears after pushforward; apply the Lipschitz pushforward inequality to ${\mathcal G}_+$ and its inverse. Multiplying every RK stage by the fixed affine map gives the transformed stage equations and update. For nonlinear $T$, testing against $\phi\circ T$ proves weighted conservation, while $T(z+ha)=Tz+hDTa+h^2D^2T[a,a]/2+O(h^3)$ exhibits Euler's mismatch. Finally the chain rule proves both the time-reparameterized continuity equation and acceleration formula.
\end{proof}
These uniformly bounded rearrangements preserve nonzero local order. Nonlinear coordinates require transformed costs; clock changes generally change the labeled tangent. For a given orientation-preserving diffeomorphism ${\Psi}_h$, exactness is precisely ${\rho}_t(x)={\rho}_{t+h}({\Psi}_hx)\det D{\Psi}_h(x)$, a change-of-variables criterion. Nonautonomous RK requires actual stage times and fields, not frozen $R(hB_t)$.

\input{supplement_compact/checks/digits}

%% file: supplement_compact/checks/whole_path.tex
\FloatBarrier
\paragraph{Controlled whole-path check.}
Repeated/simple strain has kernel dimension 2/0; commuting repeated-block whole-path error $1.93\!\times\!10^{-15}$ contrasts with repeated-cross/simple-block 0.428968/0.216728. Initial tangency alone does not settle preservation of the whole prescribed path.
\begin{figure}[!htbp]\centering
\includegraphics[width=\linewidth]{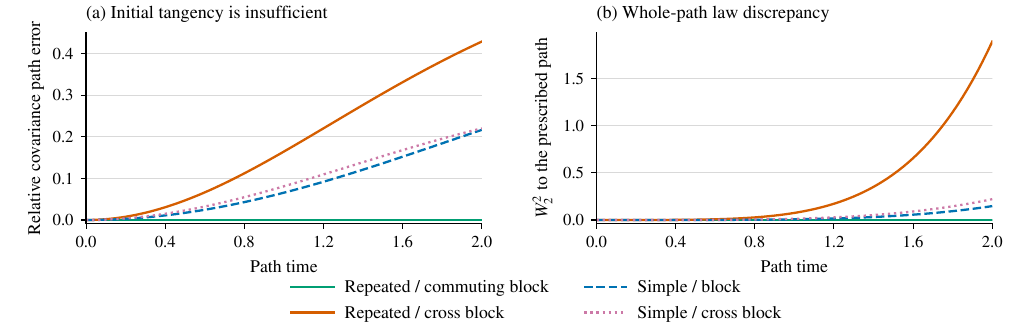}
\caption{All four repeated/simple-strain controls at 81 recorded times. These deterministic covariance and law-discrepancy curves verify the commuting whole-path boundary in \cref{prop:symmetry-strata}, without confidence intervals.}\label{fig:controlled-whole-path}
\end{figure}
\FloatBarrier

%% file: supplement_compact/checks/digits.tex
\FloatBarrier
\noindent\begin{minipage}{\linewidth}
\subsection{Controlled native fitting and solver/grid transfer}\label{app:controlled-digits}
\begin{table}[H]\centering\small
\caption{Exact latent-model integration and transfer results. Target-model sampling and solver-transfer errors are distinct.}\label{tab:controlled-digits}
\begin{tabularx}{\linewidth}{@{}lX@{}}\toprule Check & Retained result and boundary\\\midrule
E3: Digits, $d=16$, N4 & Class-scale RMSE Base/LO/FH/target: 0.46115/0.17415/0.02383/0.02623; split-bootstrap Base--FH improvement 0.43732 $[0.43374,0.44137]$, positive in all ten refits.\\
E3: solver/grid transfer & Two-step Heun scale error Base/Euler-FH/Heun-FH: 0.13366/0.77290/0.02522. Jittered Euler Base/transferred/recomputed: 0.65885/0.05061/0.02735. This matches the number of partitions, not the network-call budget.\\
\bottomrule\end{tabularx}\end{table}

\ControlledDigitsFigure
\end{minipage}\par
\input{supplement_compact/galleries/digits}

%% file: supplement_compact/galleries/digits.tex
\begin{figure}[H]\centering
\begin{minipage}{\linewidth}\raggedright\normalsize
Digits shuffling preserves class frequencies but mismatches scale response. Exact-model--held-out-code sliced-distance is about 0.0605, versus FH 0.0607 at N4: model error persists at the integration sampling baseline (decoded examples in \cref{fig:samples-digits}). Classification accuracy is not an integration metric.
\end{minipage}\par\smallskip
\includegraphics[width=\linewidth]{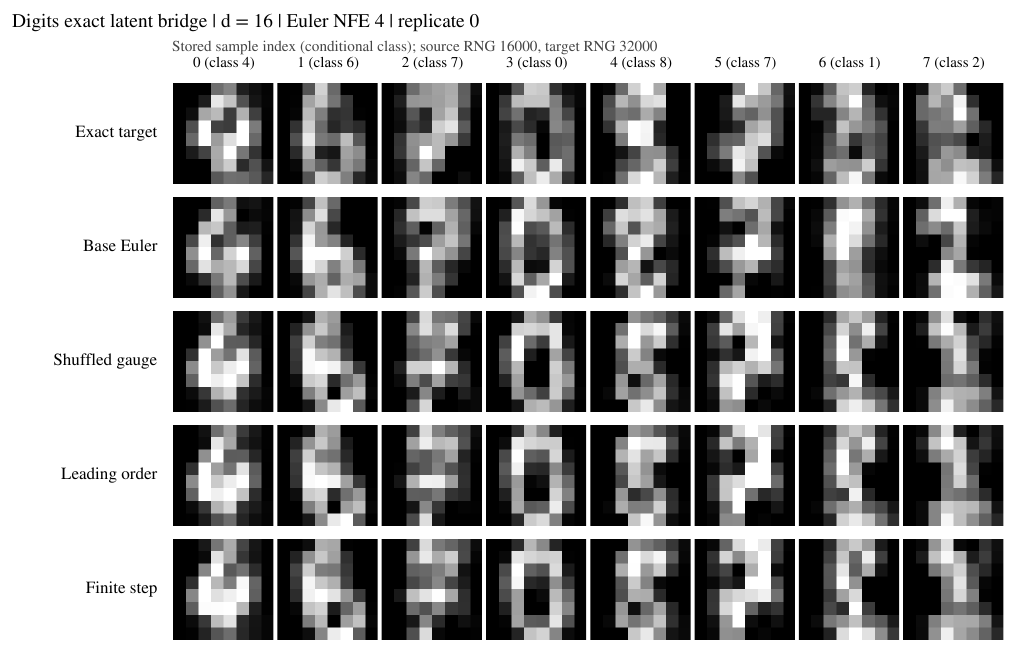}
\caption{Cached exact-latent Digits, $d=16$, Euler N4, replicate 0, source RNG seed 16000. Base/shuffled/LO/FH share initial latents, while the exact target is an independent draw with the same labels (RNG seed 32000). Linear PCA decoding illustrates the surrogate model in \cref{app:controlled-digits}; it is separate from pretrained neural sampling.}\label{fig:samples-digits}
\end{figure}

%% file: supplement_compact/theory/04_nonlinear.tex
\section{Euclidean Obstructions, Currents, Radial Laws, and Mixtures}
\label{app:atlas}\label{app:practical-classification}\label{app:moment-budget}\label{app:nonlinear-atlas}
Radial classifications (\cref{app:radial-finite}) require shell ordering and regularity; current and jet constructions (\cref{app:antisymmetric-potential}) do not classify all nonlinear numerical maps. \Cref{app:examples} pairs the mixture theory with its controlled numerical checks (\cref{fig:controlled-mixture-response}).

\begin{theorem}[Moment obstruction and required energy]
\label{lem:gauge-moments}\label{thm:moment-no-go}\label{thm:moment-budget}\label{prop:pure-visibility}
Let ${\rho}\in\mathcal P_2(\R^d)$, $u\in L^2({\rho})$ and $\Div({\rho}u)=0$. Then $\E u=0$, $\E[xu^\top+ux^\top]=0$. For $m=\E x$, ${\mathcal V_\rho}=\E|x-m|^2$, ${\mathcal E_u}=\E|u|^2$, the Euler law $\mu_h=(I+hu)_\#{\rho}$ satisfies
\begin{equation}\label{eq:moment-cov}
 \Cov\mu_h=\Cov {\rho}+h^2\E[uu^\top],\qquad
 W_2(\mu_h,{\rho})\ge\sqrt{{\mathcal V_\rho}+h^2{\mathcal E_u}}-\sqrt {\mathcal V_\rho}.
\end{equation}
\label{eq:main-moment-gate}
Hence a nonzero finite-energy pure gauge cannot be Euler-exact for $h\ne0$. Under a third-order weak Taylor bound, $\mu_h={\rho}+\frac12h^2\partial_i\partial_j({\rho}u_iu_j)+O(h^3)$; for smooth $u$ this second-order source is $\Div({\rho}(u\cdot\nabla)u)$.

For a differentiable finite-moment path with raw second moment $\mathcal M_2(t)=\E_{\rho_t}|X|^2$ and continuity-equation moment differentiation justified by cutoff limits, a feasible lift $w=v^{\rm W}+u$ with minimum-$L^2$ lift $v^{\rm W}$ can be Euler-exact at $h\ne0$ only if
\begin{equation}\label{eq:energy-budget}
 m(t+h)=m+h\dot m,\qquad
 \|u\|_{L^2({\rho}_t)}^2=\frac{\mathcal M_2(t+h)-\mathcal M_2(t)-h\dot{\mathcal M}_2(t)}{h^2}-\|v^{\rm W}\|_{L^2({\rho}_t)}^2\ge0.
\end{equation}
These conditions are not sufficient for full-law exactness.
\end{theorem}
\begin{proof}
Test $\Div({\rho}u)=0$ with $x_i\chi_R$ and $x_ix_j\chi_R$. The cutoff errors are bounded respectively by $C\int_{|x|>R}|u|{\rho}$ and $C\int_{|x|>R}|x||u|{\rho}$, which vanish by Cauchy--Schwarz and finite second moments. Expansion of $x+hu$ gives the covariance identity; reverse Minkowski on any centered coupling gives its $W_2$ bound. Taylor expansion of a test function gives the weak source, since its linear term vanishes. The remainder is bounded by $|h|^3\|D^3\phi\|_\infty\E|u|^3/6$, and $\partial_j({\rho}u_iu_j)={\rho}(u\cdot\nabla)u_i$. For a moving path, continuity gives $\dot m=\E w$ and $\dot{\mathcal M}_2=2\E(x\cdot w)$. Expanding $\E|x+hw|^2$ and using Hodge orthogonality proves~\eqref{eq:energy-budget}.
\end{proof}
For ${\mathcal V_\rho},{\mathcal E_u}>0$ the lower bound is $h^2{\mathcal E_u}/(2\sqrt {\mathcal V_\rho})+O(h^4)$, sharp for isotropic planar rotations. The obstruction is specific to stationary Euclidean laws.

\begin{proposition}[Moving-path moment barrier]\label{prop:acceleration-moment-barrier}
In the Euclidean setting of \cref{thm:moment-budget}, assume twice-differentiable moments and justified continuity-equation/cutoff identities. Let $w=v^{\rm W}+u$, $v^{\rm W}\in\mathcal T_{\rho_t}$ be the minimum-$L^2$ lift, and $w,a_w\in L^2(\rho_t)$. The projected norms below use $L^2(\rho_t)$. Write $m=\E X$, $\mathcal V_\rho=\E|X-m|^2>0$, $\mathcal K_{\rm W}=\E|v^{\rm W}-\dot m|^2$, and $\mathcal E_u=\E|u|^2$. Then
\begin{align}
 \|\Pi_{\rho_t}a_w\|^2&\ge|\ddot m|^2+
 \frac{(\ddot{\mathcal V}_\rho/2-\mathcal K_{\rm W}-\mathcal E_u)^2}{\mathcal V_\rho},\label{eq:acceleration-moment-barrier}\\
 \inf_u\|\Pi_{\rho_t}a_w\|^2&\ge|\ddot m|^2+
 \frac{[\mathcal K_{\rm W}-\ddot{\mathcal V}_\rho/2]_+^2}{\mathcal V_\rho}.
\end{align}
In particular full cancellation requires $\ddot m=0$ and $\mathcal E_u=\ddot{\mathcal V}_\rho/2-\mathcal K_{\rm W}\ge0$. For $\mathsf r_{\rm rms}=\sqrt{\mathcal V_\rho}$ it also requires $\ddot{\mathsf r}_{\rm rms}\ge0$.
At finite $h$, put $A_h=\mathcal V_\rho+h\dot{\mathcal V}_\rho+h^2\mathcal K_{\rm W}\ge0$ and $\mathcal V_+=\mathcal V_{\rho_{t+h}}$. Then
\begin{equation}\label{eq:moving-moment-finite}
\begin{split}
 W_2^2((I+hw)_\#\rho_t,\rho_{t+h})&\ge |m+h\dot m-m_{t+h}|^2\\[-2pt]
 &\quad+(\sqrt{A_h+h^2\mathcal E_u}-\sqrt{\mathcal V_+})^2.
\end{split}
\end{equation}
Minimizing its right side over nonnegative energy yields the same bound with the last term replaced by $[\sqrt{A_h}-\sqrt{\mathcal V_+}]_+^2$.
\end{proposition}
\begin{proof}
Moment differentiation gives $\ddot m=\E a_w$ and $\ddot{\mathcal V}_\rho/2=\E|w-\dot m|^2+\E[(X-m)\cdot a_w]$. Constant fields and $x-m$ belong to $\mathcal T_{\rho_t}$ by cutoff approximation. Gauge orthogonality therefore gives $\E u=0$ and $\E|w-\dot m|^2=\mathcal K_{\rm W}+\mathcal E_u$. With $z=\Pi_{\rho_t}a_w$, its inner products against the orthonormal fields $e_1,\ldots,e_d,(x-m)/\sqrt{\mathcal V_\rho}$ are the mean acceleration components and $(\ddot{\mathcal V}_\rho/2-\mathcal K_{\rm W}-\mathcal E_u)/\sqrt{\mathcal V_\rho}$. Bessel's inequality proves the first bound; minimizing over all energies $\mathcal E_u\ge0$ proves the second, even if not every energy is feasible. Cauchy--Schwarz gives $\dot{\mathcal V}_\rho^2\le4\mathcal V_\rho\mathcal K_{\rm W}$, hence the radius-convexity condition. Finally Euler's mean and centered second moment are exactly $m+h\dot m$ and $A_h+h^2\mathcal E_u$. Splitting any coupling into its mean displacement and centered displacement, then applying reverse Minkowski, proves \cref{eq:moving-moment-finite}.
\end{proof}
The projected bound is sharp: for $\rho_t=\law(m+\varsigma Z)$ with covariance-normalized spherical $Z$, $v^{\rm W}=\dot m+(\dot\varsigma/\varsigma)(x-m)$. If $\ddot\varsigma<0$, $u=0$ attains $|\ddot m|^2+d\ddot\varsigma^2$. If $\ddot\varsigma>0$ and $d$ is even, $u=\sqrt{\ddot\varsigma/\varsigma}\,J(x-m)$ attains $|\ddot m|^2$; its remaining rotational acceleration is invisible. Zeros of the radicand require smoothness checks.

\begin{proposition}[Fixed-field Euler barrier and sharp example]\label{prop:fixed-euler-barrier}
Let $\mathcal M_2(t)=\E_{\rho_t}|X|^2>0$ be $C^3$ and let $w_t\in L^2(\rho_t)$ be a feasible lift independent of $h$, with $\dot{\mathcal M}_2=2\E(X\cdot w_t)$. If $\mathcal M_2'''(t)\ne0$, its local Euler error cannot be $O(h^4)$. If $\E|w_t|^2=\ddot{\mathcal M}_2/2$, more precisely
\begin{equation}\label{eq:fixed-euler-barrier}
 \liminf_{h\downarrow0}\frac{W_2((I+hw_t)_\#\rho_t,\rho_{t+h})}{h^3}
 \ge\frac{|\mathcal M_2'''(t)|}{12\sqrt{\mathcal M_2(t)}}.
\end{equation}
For $\rho_t=N(0,\varsigma_0^2e^{2\beta t}I)$, $\beta>0$, and even $d$, the fixed field $w=\beta(I+J)x$ attains this bound:
\begin{equation}\label{eq:sharp-exponential-local}
 W_2((I+hw)_\#\rho_t,\rho_{t+h})
 =\tfrac23\beta^3\sqrt{\mathcal M_2(t)}\,h^3+O(h^4).
\end{equation}
\end{proposition}
\begin{proof}
Euler's raw second moment is exactly $\mathcal M_2+h\dot{\mathcal M}_2+h^2\E|w_t|^2$. For any two $\mathcal P_2$ laws, reverse Minkowski bounds $W_2$ below by the absolute difference of their RMS norms. Taylor-expand $\mathcal M_2(t+h)$. A mismatch in the quadratic coefficient gives a positive $h^2$ lower coefficient; when it vanishes, division by the sum of RMS norms gives \cref{eq:fixed-euler-barrier}. In the example Euler's spherical radius factor is $\sqrt{1+2\beta h+2\beta^2h^2}$; \cref{eq:elliptical-w2} and its expansion against $e^{\beta h}$ give \cref{eq:sharp-exponential-local}.
\end{proof}
Local $h^3$ is therefore optimal among \emph{all} fixed finite-energy feasible fields. For this particular rotating field, $Nh=T$ gives terminal error $\tfrac23\beta^3T\sqrt{\mathcal M_2(T)}h^2+O(h^3)$ by expanding $[1+2\beta h+2\beta^2h^2]^{T/(2h)}$. Other time-dependent fields may cancel terminal error; the $h$-dependent frequency in \cref{eq:finite-exact-gauge} supplies the cubic moment term. This clock-dependent barrier excludes the torus and exact EDM/RF paths, whose second moments are quadratic in $\sigma$ and $t$, respectively.

\input{supplement_compact/checks/moments}

\subsection{Current potentials and finite jets}\label{app:antisymmetric-potential}\label{sec:jets}
\begin{proposition}[Antisymmetric and periodic currents]
\label{prop:antisymmetric-current}\label{prop:spatial-score}\label{prop:torus-complete}
For smooth ${\rho}>0$, an antisymmetric potential $\Xi$ gives $u_i={\rho}^{-1}\sum_j\partial_j\Xi_{ij}$ with $\Div({\rho}u)=0$. In particular a spatial skew mobility requires
\begin{equation}\label{eq:antisymmetric-current}
 u=\operatorname{Div}{M_{\rm sk}}+{M_{\rm sk}}\nabla\log {\rho}={\rho}^{-1}\operatorname{Div}({\rho}{M_{\rm sk}}),\qquad {M_{\rm sk}}^\top=-{M_{\rm sk}}.
\end{equation}
Replacing the score by $s+e$ adds residual $\Div({\rho}{M_{\rm sk}}e)$ of negative norm at most $\|{M_{\rm sk}}e\|_{L^2({\rho})}$. On a torus every smooth gauge current has the complete representation ${\rho}u=\overline{{\rho}u}+\operatorname{Div}\Xi$, with periodic antisymmetric $\Xi$; its constant current must be retained.
\end{proposition}
\begin{proof}
Double divergence of an antisymmetric tensor vanishes by exchanging indices and mixed derivatives. The product rule gives~\eqref{eq:antisymmetric-current}; subtraction and the projection identity bound the learned-score residual. For the periodic converse, let ${\mathbf j}={\rho}u$ and use Fourier coefficients satisfying $k\cdot\widehat {\mathbf j}(k)=0$. For $k\ne0$ set $\widehat\Xi_{ij}(k)=(\widehat {\mathbf j}_i(k)k_j-k_i\widehat {\mathbf j}_j(k))/(i|k|^2)$. Its row divergence recovers the nonzero modes; it has the reality condition and smooth decay. The zero mode is the constant current, whose nonzero integral prohibits a periodic divergence representation.
\end{proof}
This conservation construction is classical \citep{ma2015recipe,horvat2024gauge}. For uniform density, the torus shear $(x,y)\mapsto(x+hf(y),y)$ is invertible with unit determinant, hence exactly Euler-preserving.

\begin{theorem}[Compatible finite jets]\label{thm:jets}\label{thm:all-jets}\label{cor:planar-jets}\label{prop:restricted-jet}
For $d\ge2$, smooth positive ${\rho}$ near zero and a degree-${\ell_{\rm jet}}$ polynomial vector jet ${\mathsf P_u}$, a smooth compactly supported weighted-solenoidal field realizes $j_0^{\ell_{\rm jet}}u={\mathsf P_u}$ iff
\begin{equation}\label{eq:jetcondition}
 j_0^{{\ell_{\rm jet}}-1}\Div({\rho}{\mathsf P_u})=0
\end{equation}
(the condition is empty for ${\ell_{\rm jet}}=0$). First jets $({u_0},{J_u})$ obey $\tr {J_u}+s\cdot {u_0}=0$, $s=\nabla\log {\rho}(0)$; second jets additionally obey $\sum_i(H_u)_{iik}+\sum_i s_i(J_u)_{ik}+\sum_i(\partial_ks_i){u_0}_i=0$. At an equilibrium every trace-free ${J_u}$ is realizable. In dimension two these are elliptic, saddle, nonzero nilpotent shear, or zero according to the sign of $-\det {J_u}$. By contrast, if $s=0$ and ${H_\rho}=-\nabla^2\log {\rho}(0)\succ0$, restricting the velocity to exactly ${J_u}x$ and requiring its residual to vanish through quadratic degree gives ${J_u}^\top {H_\rho}+{H_\rho}{J_u}=0$.
\end{theorem}
\begin{proof}
Necessity follows by differentiating conservation. For sufficiency write $P=j_0^{\ell_{\rm jet}}({\rho}{\mathsf P_u})=\sum_kP_k$ into homogeneous pieces; the condition says each is divergence-free. Set $\Xi^{(k)}_{ij}=(P_{k,i}x_j-P_{k,j}x_i)/(k+d-1)$. Euler's homogeneous identity gives $\sum_j\partial_j\Xi^{(k)}_{ij}=P_{k,i}$. Multiply the sum of potentials by a cutoff equal to one near zero, supported where ${\rho}>0$, and use the preceding current construction. Its jet is ${\mathsf P_u}$, and extension by zero is smooth, compactly supported and complete. Differentiating $\Div u+s(x)\cdot u=0$ gives the first and second conditions. At ${u_0}=0$ the first is $\tr {J_u}=0$. In two dimensions Cayley--Hamilton gives ${J_u}^2=-\det({J_u})I$, hence the four real-similarity types, all realizable by this construction. For an exactly linear field the residual is $\tr {J_u}-x^\top {J_u}^\top {H_\rho}x+O(|x|^3)$. Its quadratic coefficient vanishes iff the metric-skew equation holds; this also forces trace zero. Nonlinear higher velocity terms can instead cancel the quadratic divergence, so this restriction cannot classify all feasible first jets.
\end{proof}
Finite jets do not classify smooth germs or phase portraits: flat terms can have all jets zero. Near a Gaussian mode, the feasible current ${\rho} u=(x_1,-x_2)$ gives $u\propto e^{|x|^2/2}(x_1,-x_2)$, a saddle first jet.

\input{supplement_compact/checks/jets}

\subsection{Radial finite-step response}\label{app:radial-finite}\label{sec:radial}
\begin{theorem}[Shellwise exactness and regularity]
\label{thm:radial-finite}\label{thm:origin}\label{thm:topological-parity}
Let an atomless radial path have strictly increasing radial CDFs on their supports, radial velocity $b_t({\varrho})e_{\varrho}$, and smooth increasing quantile map $T=F_{t+h}^{-1}\circ F_t$ for $h\ne0$. In even dimension fix $J^\top=-J$, $J^2=-I$. The gauge $u=\gamma({\varrho})Jx$ makes Euler transport every source shell to its target shell iff
\begin{equation}\label{eq:radial-chord-condition}
 T({\varrho})^2\ge[{\varrho}+hb_t({\varrho})]^2,\qquad
 \gamma({\varrho})^2=\frac{T({\varrho})^2-[{\varrho}+hb_t({\varrho})]^2}{h^2{\varrho}^2}.
\end{equation}
A sign branch yielding a smooth complete field gives exact full-law transport; conversely exactness in this family with increasing output radius forces this formula. If near zero $\gamma^2={\varrho}^{2k}{\vartheta_{\rm reg}}({\varrho}^2)$, $k\in\{0,1,\ldots\}$, ${\vartheta_{\rm reg}}$ smooth and ${\vartheta_{\rm reg}}(0)>0$, the field is smooth for even $k$, and $C^k$ but not $C^{k+1}$ for odd $k$. Flat or other degenerate zeros need separate analysis. Any continuous tangential field with deterministic output radius and a strictly positive shell gap requires even ambient dimension, not only linear angular fields.
\end{theorem}
\begin{proof}
Radial amplitudes times $Jx$ are weighted-solenoidal by tangency and $\tr J=0$. Euler's squared radius is $({\varrho}+hb)^2+h^2{\varrho}^2\gamma^2$. Under the formula, the normalized angular map is orthogonal and preserves conditional uniform angles, while $T$ transports radial mass. Increasing radial transport is uniquely the quantile map, proving the converse. A continuous square-root branch is $\pm {\varrho}^k\sqrt{{\vartheta_{\rm reg}}({\varrho}^2)}$ near zero. For even $k$ this gives smooth Cartesian factors. For odd $k$, derivatives through order $k$ of ${\varrho}^kJx$ vanish continuously at zero; along a line a nonzero component is proportional to $s|s|^k$, with unequal one-sided derivatives of order $k+1$.

For the final claim a positive gap forces $u({\varrho}n)$ to be a nowhere-zero tangent field of constant norm on $S^{d-1}$. Normalize it to $\tau(n)$. The map $\cos(\pi s)n+\sin(\pi s)\tau(n)$ is a sphere-valued homotopy from identity to antipodal; their degrees are $1$ and $(-1)^d$, impossible for odd $d$ \citep{adams1962vector}. For $d=1$ tangency directly forces zero. A zero gap forces $u=0$ by the radius identity; even-dimensional $Jn$ supplies the converse construction.
\end{proof}
Angle-dependent redistribution remains possible. As $h\to0$, ${\varrho}\gamma^2=\partial_tb+b\partial_{\varrho}b+O(h)$, giving leading centrifugal compensation.

\begin{corollary}[Odd-power radius family and best constant rotation]
\label{cor:odd-power}
For ${\ell_{\rm rad}}=2k+1\ge3$, $t\ge0$, $h>0$, take ${\varrho}={\mathsf r}_t({\upsilon})={\upsilon}+t^2{\upsilon}^{\ell_{\rm rad}}$, $b_t({\varrho})=2t{\upsilon}^{\ell_{\rm rad}}$. Then
\begin{equation}\label{eq:radialgeneral}
 \gamma^2=\frac{{\upsilon}^{2k}[2+(2t^2+4th+h^2){\upsilon}^{2k}]}{(1+t^2{\upsilon}^{2k})^2}.
\end{equation}
The gauge is smooth at zero exactly for ${\ell_{\rm rad}}\equiv1\pmod4$, otherwise $C^{({\ell_{\rm rad}}-1)/2}$ but not $C^{({\ell_{\rm rad}}+1)/2}$. At $t=0$ let $Q>0$ be the source radius, $T(Q)=Q+h^2Q^{\ell_{\rm rad}}$, and assume $0<M_2=\E Q^2<\infty$ and $M_T=\E T(Q)^2<\infty$, and set $C_T=\E[QT(Q)]$. For constant rotations with output radius $cQ$, $c\ge1$,
\begin{equation}\label{eq:radialopt}
 c_{\rm cov}=\sqrt{M_T/M_2},\quad E_{\rm cov}=2M_T-2c_{\rm cov}C_T,\qquad
 c_{\rm opt}=C_T/M_2,\quad E_{\rm opt}=M_T-C_T^2/M_2.
\end{equation}
\end{corollary}
\begin{proof}
The target radius exceeds its Euler radial chord by $h^2{\upsilon}^{\ell_{\rm rad}}$; expanding their squared difference proves~\eqref{eq:radialgeneral}. The implicit function theorem gives ${\upsilon}={\varrho}\eta({\varrho}^{2k})$ with $\eta$ smooth positive, so the origin result applies. On a fixed finite interval ${\mathsf r}_t({\upsilon})$ remains finite, and the angular equation is regular along its radius, giving complete trajectories for smooth branches. Any vector coupling costs at least its radial coupling; shared uniform angles and increasing quantiles attain this bound. Thus the full cost is $\E(cQ-T(Q))^2$. Matching the second moment gives $c_{\rm cov}$; minimizing the quadratic over $c\ge1$ gives $c_{\rm opt}$, which is at least one. Strict Cauchy--Schwarz gives positive residual for a genuinely nonconstant $T(Q)/Q$.
\end{proof}
For $t>0$ the transformed law is in $\mathcal P_2$ iff $Q$ has a finite $2{\ell_{\rm rad}}$th moment; changing radial shape can therefore invalidate a heavy-tail $W_2$ premise.

\input{supplement_compact/checks/radial}

\subsection{Mixture response and counterexamples}\label{app:examples}\label{sec:mixture}
\begin{theorem}[Responsibility currents and barycentric defect]
\label{prop:mixture-gauge}\label{thm:mixture}
Let ${\rho}=\sum_m\pi_m{\rho}_m$ have fixed positive weights and responsibilities ${\chi}_m=\pi_m{\rho}_m/{\rho}$. Component gauges $u_m$ give the exact gauge $\sum_m{\chi}_mu_m$. For square-integrable steps $T_m=x+hw_m$ let $\bar w=\sum_m{\chi}_mw_m$, $T=x+h\bar w$, $\mu=T_\#{\rho}$, $\widetilde\mu=\sum_m\pi_mT_{m\#}{\rho}_m$, and ${C_{\rm post}}(x)=\sum_m{\chi}_m(w_m-\bar w)(w_m-\bar w)^\top$. The two output means coincide,
\begin{equation}\label{eq:mixcov}
 \Cov\widetilde\mu-\Cov\mu=h^2\E_{\rho}{C_{\rm post}}\succeq0,\qquad \mu\preceq_{\rm cx}\widetilde\mu.
\end{equation}
For $h\ne0$, the laws coincide iff ${C_{\rm post}}=0$ almost surely. Thus componentwise exact steps give a deterministic exact mixture step iff their velocities agree on posterior overlap. With $\mathcal E_{\rm post}=\E\tr {C_{\rm post}}$, $V_{\rm out}=\tr\Cov\widetilde\mu$,
\begin{equation}\label{eq:mixbounds}
 \sqrt{V_{\rm out}}-\sqrt{V_{\rm out}-h^2\mathcal E_{\rm post}}\le W_2(\mu,\widetilde\mu)\le|h|\sqrt{\mathcal E_{\rm post}}.
\end{equation}
\end{theorem}
\begin{proof}
The current is ${\rho}\sum {\chi}_mu_m=\sum\pi_m{\rho}_mu_m$, whose divergence is zero. Draw a component label then its state $X$, and let $\widetilde Y=T_M(X)$. Its conditional mean is $Y=T(X)$. Total covariance and conditional Jensen give~\eqref{eq:mixcov}. Equal laws force equal second moments, hence $\E\tr {C_{\rm post}}=0$; positive semidefiniteness forces ${C_{\rm post}}=0$ almost surely, which conversely gives $Y=\widetilde Y$. A weighted covariance vanishes exactly when its positive-weight entries agree. Their coupling has squared displacement $h^2\mathcal E_{\rm post}$, proving the upper bound; reverse Minkowski on centered couplings proves the lower one.
\end{proof}
Add component-target error by triangle inequality. Randomizing the posterior label attains $\widetilde\mu$ but changes the algorithm. The covariance deficit is an on-path local identity, not off-path monotonicity.

\input{supplement_compact/checks/mixtures}

\subsection{Leading constructions and curvature counterexamples}
\begin{proposition}[Leading constructions and falsification examples]
\label{prop:conformal-aniso}\label{prop:radial-cancel}\label{prop:harmful-gauge}\label{prop:curvature-distractor}
For ${\rho}_t=N(0,{\varsigma}^2\diag(a^2,b^2))$, the gauge ${K_{\rm conf}}=cJ\diag(a^2,b^2)^{-1}$ has projected acceleration $(\ddot {\varsigma}/{\varsigma}-c^2/(a^2b^2))x$. For a planar radial lift $b({\varrho})e_{\varrho}+c({\varrho})e_\theta$, acceleration has radial part $\partial_tb+b\partial_{\varrho}b-c^2/{\varrho}$ and angular part $\partial_tc+b\partial_{\varrho}c+bc/{\varrho}$; the latter is density-invisible. These cancel leading error when the corresponding radicands are nonnegative and the field is smooth.

A continuum gauge can also harm Euler: for uniform $\mathbb T^2$, $u=(-\sin x\cos y,\cos x\sin y)$ is solenoidal but $\Div((u\cdot\nabla)u)=\cos2x+\cos2y$. Conversely full acceleration can grow arbitrarily while its projection vanishes: in dimension four, for $\kappa=\ddot {\varsigma}/{\varsigma}>0$ take $B=\beta I+\sqrt\kappa J(t)$, with $J(t)$ an orthogonal complex structure; then $\Pi_{\rho}a_w=0$ for ${\rho}_t=N(0,{\varsigma}^2I)$ even with rapidly varying $J$.
\end{proposition}
\begin{proof}
For the conformal example ${K_{\rm conf}}\Sigma+\Sigma {K_{\rm conf}}^\top=0$, ${K_{\rm conf}}^2=-c^2/(a^2b^2)I$, and the remaining $\dot {K_{\rm conf}}+2\beta {K_{\rm conf}}$ is weighted-solenoidal. Polar differentiation gives the two radial components; every radial-amplitude angular field has zero weighted divergence. Direct differentiation gives the harmful example's acceleration $(\sin2x,\sin2y)/2$. Finally $\dot B+B^2$ has canceling radial terms $\dot\beta+\beta^2-\kappa=0$; its remaining terms are multiples of skew $J$ and $\dot J$, hence Gaussian-solenoidal. Choose $J(t)=e^{MtH}J_0e^{-MtH}$ with skew $H$ not commuting with $J_0$. Then $\|\dot J\|_F=M\|[H,J_0]\|_F$, and nondegenerate covariance makes full acceleration arbitrarily large on a fixed subinterval while its projection remains zero.
\end{proof}

\FloatBarrier
\paragraph{Controlled curvature check.}\label{app:controlled-curvature}
The rotating-complex-structure fixture checks the last counterexample with its radial path fixed.
\ControlledCurvatureFigure
\FloatBarrier

%% file: supplement_compact/checks/moments.tex
\paragraph{Controlled Euclidean/torus contrast.}\label{app:controlled-moments}
Stationary Euclidean compact-gauge energy 0.02349038 produces the $h^2\E\|u\|^2$ moment defect (residual $2.62\!\times\!10^{-16}$); a periodic torus shear instead has determinant 1 and inverse residual $1.11\!\times\!10^{-16}$. These checks illustrate the boundary of the result; they do not extend the Euclidean obstruction to periodic laws.

%% file: supplement_compact/checks/jets.tex
\paragraph{Controlled current-jet check.}
Compact elliptic/saddle/nilpotent/zero current jets have divergence residual $\leq 1.12\!\times\!10^{-15}$ and finite-difference Jacobian error $\leq 5.01\!\times\!10^{-11}$: local jets, not global portraits.

%% file: supplement_compact/checks/radial.tex
\FloatBarrier
\paragraph{Controlled shellwise check.}\label{app:controlled-radial}
For radial quantiles $\mathsf r_t(\upsilon)=\upsilon+t^2\upsilon^{\ell_{\rm rad}}$, $d=2,t=0,h=0.2$, costs Base/covariance-only/best-constant/shellwise are 0.0768/0.02566059/0.0256/0 for $\ell_{\rm rad}=3$ and 6.144/4.77680735/4.3008/0 for $\ell_{\rm rad}=5$ (\cref{fig:analytical-boundary-cases}). The cubic field is $C^1$, not $C^2$, at the origin; zero shellwise cost does not waive smoothness.
\begin{figure}[!htbp]\centering
\includegraphics[width=\linewidth]{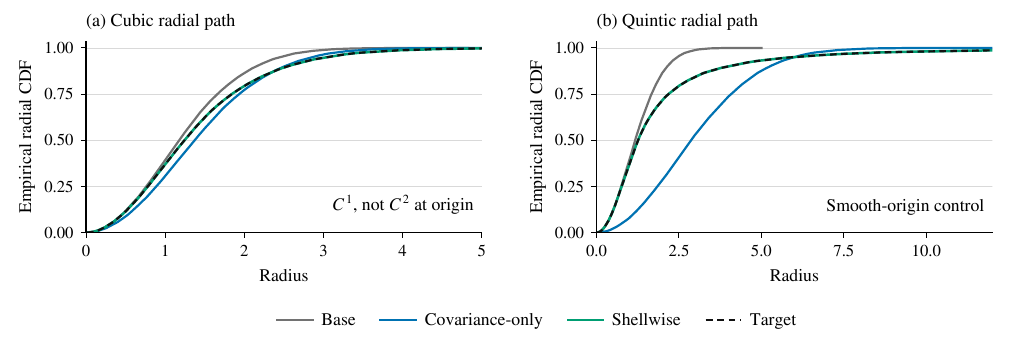}
\caption{Archived radial CDF curves for the cubic and quintic full-law fixtures. The cubic field is $C^1$, not $C^2$, at the origin; the quintic counterpart is smooth. Covariance matching alone does not settle full-law transport. These are point diagnostics without confidence intervals.}\label{fig:analytical-boundary-cases}\label{fig:controlled-radial}
\end{figure}
\FloatBarrier

%% file: supplement_compact/checks/mixtures.tex
\FloatBarrier
\phantomsection\label{app:controlled-mixtures}
\noindent\begin{minipage}{\linewidth}
\textbf{Controlled mixture check.}\par\smallskip
\begin{table}[H]\centering\small
\caption{Separated and overlapping mixture outcomes. All negative and competing-control results are retained.}\label{tab:controlled-mixtures}
\begin{tabularx}{\linewidth}{@{}lX@{}}\toprule Check & Retained result and boundary\\\midrule
E2: hard GMM, N4 & Component-scale RMSE Base/energy-random/analytic FH/target: 0.393/1.096/0.059/0.056; mode-mass errors remain comparable to target sampling.\\
E2: overlap stress, N4 & Analytic FH 0.0875 $\pm$0.0033; calibrated FH 0.0887 $\pm$0.0033; curvature-energy 0.0929 $\pm$0.0021 (eight seeds). Calibration does not improve the analytic field; curvature is competitive.\\
\bottomrule\end{tabularx}\end{table}

\ControlledMixtureFigure
\end{minipage}\par
\FloatBarrier

%% file: supplement_compact/controlled_checks.tex
\section{Controlled-Check Protocols and Regression Scope}\label{app:controlled-checks}\label{app:toy-experiments}
The controlled checks use known laws or exact latent surrogates. Their results follow the relevant proof blocks; \cref{tab:compact-controlled} maps claims to checks. Approximate-score experiments remain separate in \cref{app:experiments}.

\subsection{Shared interpretation}
E0--E2 use analytic Gaussian/mixture paths; E3 uses the data-derived exact latent Digits surrogate. Floating-point residuals, held-out empirical errors and independent target-model sampling baselines are distinct. Deterministic checks have no sampling intervals; seed/refit ranges and supplied standard errors retain the meanings stated beside each figure. Exact sampling settings, compact inputs and portable figure scripts accompany the anonymous code supplement.

\subsection{Additional analytical regressions}\label{app:classification-experiments}\label{app:extension-checks}
The 1,064-record CPU suite includes the 384 Gaussian/Student rows described in \cref{app:controlled-reachability}. Separate algebraic checks retain 348 exhaustive rank/norm cases (gap $\leq 1.43\!\times\!10^{-14}$), overlap quadrature (residual $<5.9\!\times\!10^{-15}$), six radial powers, 15 current identities, 50 anisotropic/singular and 80 amplitude/orientation cases, symbolic score/KID counterexamples and 51 original regressions. These are numerical checks of specific theorem instances, not evidence of neural-model generalization.
\FloatBarrier

%% file: supplement_compact/theory/05_models.tex
\section{Neural Adapters, Validation Boundaries, and Calibration Limits}
\label{app:diffusion}
The established skew-score current \citep{ma2015recipe,ma2022preconditioned,ma2025preconditioned} supports \cref{sec:application,app:neural-protocol}. Here we bound learned-score residuals, derive the independent-path adapter (\cref{app:rf-gauge,app:neural-rf}), and delimit own-error and feature diagnostics (\cref{app:neural-own,app:neural-local}).
\begin{proposition}[Score gauges, leakage, and acceleration]
\label{prop:diffusion-skew}\label{cor:score-lowrank}\label{prop:score-leak}\label{cor:score-leak-prop}\label{lem:diff-acc}
For a state-independent diffusion matrix ${D_{\rm diff}}={G_{\rm noise}}{G_{\rm noise}}^\top$, the Fokker--Planck law has probability-flow velocity $v=f-{D_{\rm diff}}s/2$, $s=\nabla\log {\rho}$. Any spatially constant time-skew ${M_{\rm sk}}$ gives the same path under $v+{M_{\rm sk}}s$. One may use ${M_{\rm sk}}=g^2A/2$ for ${D_{\rm diff}}=g^2I$, or ${M_{\rm sk}}={D_{\rm diff}}^{1/2}A{D_{\rm diff}}^{1/2}/2$, $A$ skew. A low-rank form ${M_{\rm sk}}=UV^\top-VU^\top$ has rank at most $2r$ and applies to a cached score as $U(V^\top s)-V(U^\top s)$ in $O(dr)$ arithmetic without another network evaluation.

For $s_\theta=s+e$, a learned correction $\alpha A s_\theta$ changes the nominal continuity residual by
\begin{equation}\label{eq:app-leak-exact}
 \alpha\Div({\rho}Ae),\qquad
 \|\alpha\Div({\rho}Ae)\|_{\dot H^{-1}({\rho})}\le|\alpha|\|A\|_{\op}\|e\|_{L^2({\rho})}.
\end{equation}
If a linearized source response has a bound with nonnegative ${\mathcal K_{\rm prop}}$, $\|\eta_T\|_X\le\int {\mathcal K_{\rm prop}}(T,t)\|{\mathfrak f}_t\|_{\dot H^{-1}({\rho}_t)}\dd t$, its incremental learned-gauge contribution is bounded by the integral of ${\mathcal K_{\rm prop}}|\alpha|\|A\|_{\op}\|e\|$. For $w=v+\alpha As$, its acceleration is
\begin{equation}\label{eq:app-diff-acc}
 a_w=a_v+\dot\alpha As+\alpha\dot A s+\alpha A\partial_ts
 +\alpha(\nabla v)As+\alpha A(\nabla s)v+\alpha^2A(\nabla s)As.
\end{equation}
\end{proposition}
\begin{proof}
The added current is ${M_{\rm sk}}\nabla {\rho}$, whose divergence is $\tr({M_{\rm sk}}\nabla^2{\rho})=0$. Expanding the Fokker--Planck equation gives the stated probability flow, and the specified mobilities are skew. The low-rank expression is matrix associativity. Subtract the exact-score current to get~\eqref{eq:app-leak-exact}; the negative-norm projection identity and Cauchy--Schwarz give its bound. Substitution into the assumed response inequality proves its propagated version. Expanding $a_{v+{u_{\rm sk}}}=a_v+\partial_t{u_{\rm sk}}+(\nabla v){u_{\rm sk}}+(\nabla {u_{\rm sk}})v+(\nabla {u_{\rm sk}}){u_{\rm sk}}$ with ${u_{\rm sk}}=\alpha As$ proves~\eqref{eq:app-diff-acc}.
\end{proof}
For VP, $f=-\beta_{\rm VP} x/2$, $g^2=\beta_{\rm VP}$, so $w=-\beta_{\rm VP}[x+(I-A)s]/2$; VE uses $f=0$. In the EDM noise coordinate,
\begin{equation}\label{eq:app-edm-base}
 v_\theta=(x-D_\theta(x,\sigma))/\sigma,\qquad
 s_\theta=(D_\theta-x)/\sigma^2,\qquad w_\theta=(I-A_\sigma)v_\theta,
\end{equation}
corresponding to score mobility $M_\sigma^{\rm score}=\sigma A_\sigma$. The residual factor $L_\sigma^{\rm res}=A_\sigma/\sigma$ (stored as $B_\sigma^{\rm res}$ in the immutable experiment records) instead multiplies $x-D_\theta$; operator caps and norms require these distinct factors. Use signed negative noise increments and the checkpoint's $\sigma=0$ rule. Score reuse preserves the number of explicit sampler evaluations, but not arithmetic cost, wall-clock time, or the cost profile of other solvers.

The population leading calibration objective is $\int\|\Pi_{\rho}a_w\|^2+\lambda\E|\alpha As|^2\,\dd t$. Finite Galerkin and characteristic objectives are projection proxies. Even a conservative approximate score can leak: for ${\rho}=N(0,I_2)$ and $s_\theta=(-x_1+\epsilon x_1,-x_2)$, ${\rho}^{-1}\Div({\rho}Js_\theta)=-\epsilon x_1x_2$. When comparing two learned flows the relevant source is
\begin{equation}
 \Div({\rho}_t^{\rm B}{M_{\rm sk}}_ts_{\theta,t})
 =\Div\!\left({\rho}_t^{\rm B}{M_{\rm sk}}_t[s_{\theta,t}-\nabla\log {\rho}_t^{\rm B}]\right),
\end{equation}
for the actual learned-base law ${\rho}_t^{\rm B}$, not a nominal forward-score mismatch. Stability bounds retain their regularity assumptions. The pretrained Euclidean flows and piecewise-linear log-noise profiles have no verified torus positivity or uniform smooth modified expansion. Matched-seed displacement also measures particle recoupling, so it does not certify law error.

\paragraph{Empirical counterpart.}
The frozen EDM own-field and receiving-solver tests are in \cref{app:neural-own,app:neural-solver-rank}; score reuse and approximate scores do not make their measured gains exact-path certificates.

\subsection{A valid independent-Gaussian path adapter}\label{app:rf-gauge}
\begin{proposition}[Score recovery and division-free skew correction]\label{prop:rf-score-gauge}
Let $Z\sim N(0,I)$ be independent of $Y$, $X_t=(1-t)Z+tY$, and $v_t(x)=\E[Y-Z\mid X_t=x]$, $0<t<1$, with the differentiability and integrability needed below. Then
\begin{equation}\label{eq:rf-score-recovery}
 \nabla\log {\rho}_t=(tv_t-x)/(1-t),\qquad
 u_t=t(1-t){K^{\rm RF}}_t(tv_t-x)=t(1-t)^2{K^{\rm RF}}_t\nabla\log {\rho}_t,
\end{equation}
is weighted-solenoidal for spatially constant time-skew ${K^{\rm RF}}_t$. A neural velocity $v_\theta$ adds nominal residual $t^2(1-t)\Div({\rho}_t{K^{\rm RF}}_t(v_\theta-v_t))$, bounded in negative norm by $t^2(1-t)\|{K^{\rm RF}}_t\|_{\op}\|v_\theta-v_t\|_{L^2({\rho}_t)}$.
\end{proposition}
\begin{proof}
Differentiating the Gaussian convolution gives $\E[Z\mid X_t=x]=-(1-t)\nabla\log {\rho}_t(x)$. The conditional identity $x=t\E[Y\mid X_t]+(1-t)\E[Z\mid X_t]$ yields $tv_t-x=-\E[Z\mid X_t]$, proving score recovery. Skewness annihilates the Hessian current. Subtracting the exact velocity gives the learned residual, and projection/Cauchy--Schwarz give its bound.
\end{proof}
For bounded ${K^{\rm RF}}$, state and velocity, the correction vanishes at endpoints and is implemented as zero there. The derivation requires the independent Gaussian path, not generic, correlated reflow or distilled velocities. The score-proxy pilot starts a Gaussian at $\epsilon=0.001$, approximating ${\rho}_\epsilon$; actual-base score mismatch remains separate.

\paragraph{Empirical counterpart.}
The independent-path RF pilot in \cref{app:neural-rf} retains both native budgets, their own references and null quality outcomes under this score-proxy adapter.

\subsection{Numerical benefit and finite-feature blind spots}
\label{app:benefit}\label{app:feature-limits}
Let ${\mathcal E}_{\rm B}=W_2(\mu_N^{\rm B},{\rho}_T^{\rm B})$, ${\mathcal E}_{\rm G}=W_2(\mu_N^{\rm G},{\rho}_T^{\rm G})$ be numerical errors to each representative's own exact law, and $\Lambda=W_2({\rho}_T^{\rm B},{\rho}_T^{\rm G})$. Triangle and reverse triangle inequalities give
\begin{equation}
 {\mathcal E}_{\rm B}-{\mathcal E}_{\rm G}-\Lambda\le {\mathcal E}_{\rm B}-W_2(\mu_N^{\rm G},{\rho}_T^{\rm B})
 \le {\mathcal E}_{\rm B}-{\mathcal E}_{\rm G}+\Lambda.
\end{equation}
Thus ${\mathcal E}_{\rm B}-{\mathcal E}_{\rm G}>\Lambda$ suffices for improvement against the common base target. FID/KID and dense-reference refinement do not supply these population quantities or a continuum-leakage certificate.
For any metric or pseudometric $\mathfrak d$ satisfying the triangle inequality, approximate own-reference laws $\widehat\rho_T^{\rm B},\widehat\rho_T^{\rm G}$ obey
\begin{equation}\label{eq:reference-error-budget}
\begin{aligned}
&\left|[\mathfrak d(\mu_N^{\rm B},\widehat\rho_T^{\rm B})-\mathfrak d(\mu_N^{\rm G},\widehat\rho_T^{\rm G})]
-[\mathfrak d(\mu_N^{\rm B},\rho_T^{\rm B})-\mathfrak d(\mu_N^{\rm G},\rho_T^{\rm G})]\right|\\
&\hspace{5em}\le \mathfrak d(\widehat\rho_T^{\rm B},\rho_T^{\rm B})+\mathfrak d(\widehat\rho_T^{\rm G},\rho_T^{\rm G}).
\end{aligned}
\end{equation}
Apply reverse triangle inequality to each distance. Small own-error gains require both reference errors, sampling and projection control; 256/512-step differences supply neither reference bound. This applies to $W_2$ and fixed-feature, fixed-projection sliced-$W_2$ pseudometrics, not squared FID or KID. No image-space bound on $\Lambda$ or population reference errors is claimed.

The scalar degree-three kernel has $\mathrm{MMD}_{(1+xy)^3}^2({\mu_a},{\mu_b})=\sum_{j=0}^3\binom3j(\E_{\mu_a}X^j-\E_{\mu_b}X^j)^2$, by kernel expansion and independence. Yet ${\mu_a}=N(0,1)$ and ${\mu_b}=\frac12N(a,1-a^2)+\frac12N(-a,1-a^2)$, $0<a<1$, have identical moments through degree three and zero Gaussian-moment FID, while $\E_{\mu_b}X^4=3-2a^4\ne3$. They are distinct smooth $\mathcal P_2$ laws, so $W_2>0$. Polynomial features therefore have population blind spots \citep{binkowski2018demystifying}; characteristic kernels also cannot invert a noninjective feature map. Numerical convergence, feature-level discrepancy, real-data quality, and marginal preservation are distinct claims.

\paragraph{Empirical counterpart.}
Own-reference, local-response and real-image-quality diagnostics are kept separate in \cref{app:neural-own,app:neural-local}. The next subsection analyzes the calibration estimator itself; the implemented endpoint/objective and fitting families are specified in \cref{app:calibration-algorithm}.
\FloatBarrier
\input{supplement_compact/calibration_limits}

%% file: supplement_compact/calibration_limits.tex
\subsection{Limits of paired characteristic calibration}\label{app:calibration-limits}
The squared difference of \emph{paired} feature means in \cref{eq:calibration-characteristic}, population energy penalization and zero initialization have distinct limitations. These controlled results establish neither population optimality of the unchanged archived neural fits nor neural FID gains.

\paragraph{Fixed features, fixed parameters, and the sampling target.}
Fix $\Theta$, both endpoint maps, feature directions and scales. Concatenate sine/cosine pairs with factor $M_{\rm feat}^{-1/2}$ into $\Phi$. For $n_{\rm b}\ge1$ pairs, put
\[
 \Delta_\Theta(X)=\Phi(\widehat T_\Theta(X))-\Phi(T^B(X)),\qquad
 V_{\rm batch}=\left\|n_{\rm b}^{-1}\sum_{i=1}^{n_{\rm b}}\Delta_\Theta(X_i)\right\|^2.
\]
Dependence between the two outputs sharing $X_i$ is allowed. \begin{samepage}
\begin{lemma}[Paired calibration mean]\label{lem:paired-calibration}
For iid input pairs of finite second moment, the paired mean objective satisfies
\begin{equation}\label{eq:paired-v-population}
 \E V_{\rm batch}=\|\E\Delta_\Theta\|^2+\frac{1}{n_{\rm b}}\E\|\Delta_\Theta-\E\Delta_\Theta\|^2.
\end{equation}
\end{lemma}
\begin{proof} Expand the square and separate equal from distinct pair indices.\end{proof}
\end{samepage}
Only the first term depends solely on output feature means; the second generally depends on the common-input coupling and parameters, even for equal output laws. This is the V-/U-statistic distinction \citep{gretton2012kernel} for paired calibration.

For a fixed bank of $n_{\rm c}\ge2$ vectors $\Delta_i$, define its mean $\overline\Delta_{\rm bank}$, $C_{\rm bank}=\|\overline\Delta_{\rm bank}\|^2$, $D_{\rm bank}=n_{\rm c}^{-1}\sum_i\|\Delta_i\|^2$ and $S_{\rm bank}^2=(n_{\rm c}-1)^{-1}\sum_i\|\Delta_i-\overline\Delta_{\rm bank}\|^2$. Conditional on that bank,
\begin{equation}\label{eq:paired-v-finite-bank}
 \E_{\rm batch}V_{\rm batch}=
 \begin{cases}
 C_{\rm bank}+(n_{\rm b}^{-1}-n_{\rm c}^{-1})S_{\rm bank}^2,&\text{uniformly without replacement},\\
 C_{\rm bank}+(D_{\rm bank}-C_{\rm bank})/n_{\rm b},&\text{iid indices with replacement}.
 \end{cases}
\end{equation}
The first identity follows from one- and two-index inclusion probabilities. Full-cloud evaluation ($n_{\rm b}=n_{\rm c}$, without replacement) adds no resampling term. For an iid bank and independently fixed $\Phi$, $\E C_{\rm bank}$ still has the population variance term in \cref{eq:paired-v-population} with $n_{\rm b}=n_{\rm c}$.

For $n_{\rm b}\ge2$, write $D_{\rm batch}=n_{\rm b}^{-1}\sum_i\|\Delta_i\|^2$. The paired U-statistic and a correction targeting the complete empirical bank are
\begin{equation}\label{eq:paired-u-correction}
 U_{\rm batch}^{\rm pair}=\frac{n_{\rm b}V_{\rm batch}-D_{\rm batch}}{n_{\rm b}-1}
 =\frac{\sum_{i\ne j}\langle\Delta_i,\Delta_j\rangle}{n_{\rm b}(n_{\rm b}-1)},
 \qquad
 \widehat C_{\rm bank}=\frac{n_{\rm c}-1}{n_{\rm c}}U_{\rm batch}^{\rm pair}+\frac{1}{n_{\rm c}} D_{\rm batch}.
\end{equation}
For iid fresh pairs, $\E U_{\rm batch}^{\rm pair}=\|\E\Delta_\Theta\|^2$. For replacement sampling from a fixed bank, $\E_{\rm batch}U_{\rm batch}^{\rm pair}=C_{\rm bank}$. Without replacement it instead equals $C_{\rm bank}-S_{\rm bank}^2/n_{\rm c}$, whereas $\E_{\rm batch}\widehat C_{\rm bank}=C_{\rm bank}$. Indeed $S_{\rm bank}^2=n_{\rm c}(D_{\rm bank}-C_{\rm bank})/(n_{\rm c}-1)$ and $\E D_{\rm batch}=D_{\rm bank}$, which prove both claims. Excluding $i=j$ must also exclude the cross terms between the two outputs within a pair. Both corrected estimates can be negative; clipping at zero reintroduces bias.

These identities are pointwise in $\Theta$. Conditional bank identities allow a feature map fitted once on that bank and frozen, without population or post-selection guarantees. Batch-dependent scales invalidate the displayed iid formula: population interpretation requires independent features/scales or explicit dependence analysis. Unbiased estimators guarantee neither smaller gradient variance nor better finite optimization.

\paragraph{Energy control.}
\begin{samepage}
\begin{proposition}[Energy control of a returned correction]\label{prop:coercive-calibration}
Suppose $\rho\in\mathcal P_2$, $X\sim\rho$, $v\in L^2(\rho)$, and the Base reference $F_h$ satisfies $\|\Id+hv-F_h\|_{L^2(\rho)}\le C_{\rm ref}h^2$. Let $\Phi$ be uniformly Lipschitz with constant $L_\Phi$, define
$\mathfrak C_h(u)=\|\E\Phi(X+h(v+u)(X))-\E\Phi(F_h(X))\|^2$, and assume $\mathcal R_{{\rm pen},h}(0)=0$, $\mathcal R_{{\rm pen},h}(u)\ge c\|u\|_{L^2(\rho)}^2$ uniformly, $c>0$. Let $\lambda_h>0$ and $\varepsilon_h\ge0$. Any returned correction satisfying
$\mathfrak C_h(u_h)+\lambda_h\mathcal R_{{\rm pen},h}(u_h)\le\mathfrak C_h(0)+\varepsilon_h$ obeys
\begin{equation}\label{eq:energy-calibration-coercivity}
 \|u_h\|_{L^2(\rho)}^2\le
 \frac{L_\Phi^2C_{\rm ref}^2h^4+\varepsilon_h}{c\lambda_h}.
\end{equation}
With fixed $\lambda_h=\lambda>0$ and $\varepsilon_h=O(h^4)$, the \emph{norm} $\|u_h\|_{L^2}=O(h^2)$, so the common-input coupling gives
\begin{equation}\label{eq:energy-calibration-leading-defect}
 W_2((\Id+h(v+u_h))_\#\rho,(\Id+hv)_\#\rho)
 \le |h|\|u_h\|_{L^2}=O(|h|^3).
\end{equation}
Consequently a Base local law error $c_0h^2+o(h^2)$, $c_0>0$, retains its leading coefficient by the reverse triangle inequality. This conclusion needs no derivative bound on $u_h$. \end{proposition}
\begin{proof} Jensen and Lipschitz continuity bound $\mathfrak C_h(0)\le L_\Phi^2C_{\rm ref}^2h^4$; nonnegativity of the returned loss gives \cref{eq:energy-calibration-coercivity}. The common-input coupling and reverse triangle inequality give the transport and leading-coefficient claims.\end{proof}
\end{samepage}
The population denominator $\max\{\E\|v\|^2,d\,10^{-6}\}$ in \cref{eq:calibration-penalty} supplies coercivity when uniformly upper bounded; check this near singular endpoints. On a finite bank the argument controls empirical energy, requiring a generalization estimate for population transport. Node averaging controls mean energy, not its supremum.

\paragraph{Expanding Gaussian example with the same sine/cosine features.}
For the mechanism in \cref{prop:separation}, let $h>0$, $X\sim N(0,I_2)$, $\rho_t=N(0,e^{2t}I_2)$, $v(x)=x$, and $u_\omega(x)=\omega Jx$, where $J^\top=-J$ and $J^2=-I$. Every $v+u_\omega$ transports this density path. At $t=0$ use exact Base reference $e^hX$ and fixed unit-scale pair $\Phi(x)=(\cos x_1,\sin x_1)$. Set
\[
 a_{\rm E}=1+h,\quad r_{\rm ref}=e^h,\quad
 c_{\rm E}=e^{-a_{\rm E}^2/2},\quad c_{\rm ref}=e^{-r_{\rm ref}^2/2},\quad
 c_{\rm pair}=e^{-(a_{\rm E}-r_{\rm ref})^2/2},\quad \chi=h^2\omega^2.
\]
The first Euler coordinate is $a_{\rm E}X_1-h\omega X_2$. The Gaussian characteristic function therefore gives the population feature discrepancy and expected iid minibatch objective
\begin{equation}\label{eq:gaussian-calibration-objectives}
 \begin{aligned}
 C_h(\chi)&=(c_{\rm E}e^{-\chi/2}-c_{\rm ref})^2,\\
 L_{{\rm batch},h}(\chi)&=(1-n_{\rm b}^{-1})C_h(\chi)+2n_{\rm b}^{-1}(1-c_{\rm pair}e^{-\chi/2}).
 \end{aligned}
\end{equation}
The population minimum $\chi_{\rm exact}=r_{\rm ref}^2-a_{\rm E}^2>0$ matches the \emph{entire} Gaussian law, with covariance $(a_{\rm E}^2+\chi)I_2$.

\paragraph{Joint sampling and energy suppression.}
The population normalization of \cref{eq:calibration-penalty} gives $\mathcal R_{\rm pen}(\omega Jx)=\omega^2$. For $\theta=\omega^2$, combine the expected iid minibatch objective with this energy term:
\begin{equation}\label{eq:gaussian-joint-objective}
 L_{{\rm joint},h}(\theta)
 =(1-n_{\rm b}^{-1})(c_{\rm E}e^{-h^2\theta/2}-c_{\rm ref})^2
 +\frac{2}{n_{\rm b}}(1-c_{\rm pair}e^{-h^2\theta/2})+\lambda\theta.
\end{equation}
This population penalty differs from the expected random minibatch energy ratio in \cref{eq:calibration-penalty}.
\begin{proposition}[Joint minibatch and energy threshold]
\label{prop:gaussian-joint-calibration}
In the fixed-scale Gaussian model above, add the population energy
$\lambda\theta$, $\theta=\omega^2$, to the expected iid minibatch objective.
Let $n_{\rm b}\in\{2,3,\ldots\}\cup\{\infty\}$, $\lambda\ge0$,
$0<\Omega<\infty$, and use $1/\infty=0$. Define
\[
 \beta_{\rm batch}=1-n_{\rm b}^{-1},\quad
 c_{\rm mix}=\beta_{\rm batch}c_{\rm E}c_{\rm ref}
             +c_{\rm pair}/n_{\rm b},\quad
 q_0=\frac{c_{\rm mix}+\sqrt{c_{\rm mix}^2+
           4\beta_{\rm batch}c_{\rm E}^2\lambda/h^2}}
          {2\beta_{\rm batch}c_{\rm E}^2}.
\]
The unique minimum over $0\le\theta\le\Omega^2$ is
\[
 q_* = \operatorname{clip}(q_0,[e^{-h^2\Omega^2/2},1]),\qquad
 \theta_*=-2h^{-2}\log q_*.
\]
It is zero if and only if
\[
 \lambda\ge\lambda_{{\rm crit},{\rm batch}}
 :=h^2\left[\beta_{\rm batch}c_{\rm E}(c_{\rm E}-c_{\rm ref})
                  -c_{\rm pair}/n_{\rm b}\right].
\]
For $n_{\rm b}=1$ or $\Omega=0$, the minimum is zero.
Uniqueness concerns $\theta$; the signs of a nonzero $\omega$ are equivalent.
\end{proposition}
\begin{proof}
Put $q=e^{-h^2\theta/2}$. Substituting into \cref{eq:gaussian-joint-objective} gives the derivative
$2\beta_{\rm batch}c_{\rm E}^2q-2c_{\rm mix}-2\lambda/(h^2q)$
and its second derivative is strictly positive.
Solving the derivative's quadratic and projecting its root proves the
minimum. The upper endpoint is optimal exactly when the derivative at
$q=1$ is nonpositive, giving the threshold. For $n_{\rm b}=1$ the
objective is strictly decreasing in $q$; a zero cap is immediate.
\end{proof}

\paragraph{Special cases.}
\Cref{prop:gaussian-joint-calibration} gives the batch-only minimum by setting $\lambda=0$: for finite $n_{\rm b}>1$,
$q_0=c_{\rm ref}/c_{\rm E}+c_{\rm pair}/((n_{\rm b}-1)c_{\rm E}^2)$,
with zero optimal exactly when $n_{\rm b}\le1+c_{\rm pair}/[c_{\rm E}(c_{\rm E}-c_{\rm ref})]$.
Its uncapped minimizer satisfies $\chi_{\rm batch}^*<\chi_{\rm exact}$ for every finite $n_{\rm b}$; isotropy makes the expected objective identical for every fixed unit direction, so averaging directions does not remove this effect.
For the population-penalty case $n_{\rm b}=\infty$,
$q_0=(c_{\rm ref}+\sqrt{c_{\rm ref}^2+4\lambda/h^2})/(2c_{\rm E})$ and
$\lambda_{\rm crit}=h^2c_{\rm E}(c_{\rm E}-c_{\rm ref})=h^4/(2e)+O(h^5)$.
Zero is the unique minimum for $\lambda\ge\lambda_{\rm crit}$; in both cases the proposition's clipping enforces the cap. \Cref{fig:calibration-limits} plots the separate objectives.

\paragraph{Fixed-batch small-step suppression.}
As $h\downarrow0$, $\chi_{\rm exact}=h^2+\frac43h^3+O(h^4)$, $L_{{\rm batch},h}(\chi_{\rm exact})=h^2/n_{\rm b}+O(h^3/n_{\rm b})$, and
\begin{equation}\label{eq:gaussian-calibration-small-h}
 L_{{\rm batch},h}(0)=\frac{h^4}{4}\bigl[(1-n_{\rm b}^{-1})e^{-1}+n_{\rm b}^{-1}\bigr]+O(h^5),
 \quad
 n_{{\rm b},{\rm cutoff}}=2e/h^2+O(h^{-1}).
\end{equation}
Fixed finite $n_{\rm b}$ therefore suppresses the useful frequency at small $h$; its $1/(n_{\rm b}h^2)\to\infty$ lies outside \cref{cor:joint-calibration-limit}. Recovery of $\omega_{\rm exact}\to1$ requires $n_{\rm b}h^2\to\infty$. Under $|\omega|\le\Omega$, truncate to $\chi\le h^2\Omega^2$: a positive population optimum can coexist with a zero minibatch optimum. At $h=.1$ these are $\omega=1.06784$ and zero for $n_{\rm b}=128$; under $\Omega=.5$ they become $.5$ and zero. An eight-uniform-substep Heun Base reference substitutes $r_{\rm ref}$ with $(1+h/8+h^2/128)^8$; exactness means matching that approximate reference. Other substep schedules require their actual product of Heun multipliers.

\begin{figure}[htbp]\centering
\includegraphics[width=.95\linewidth]{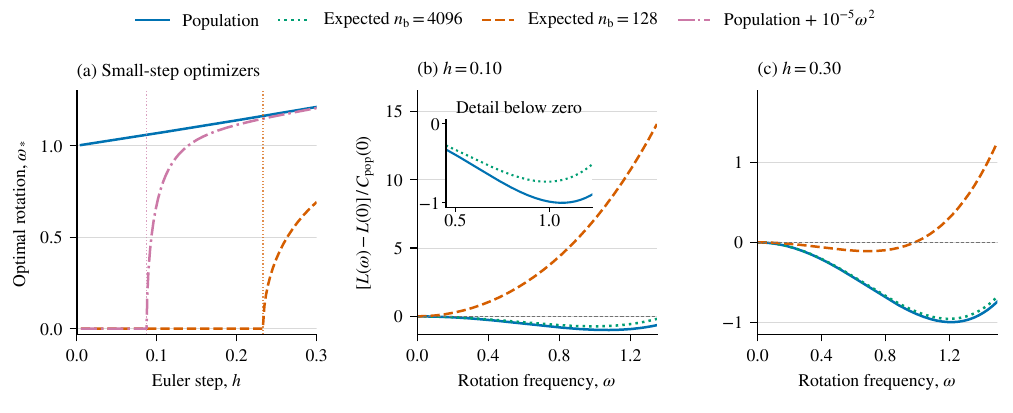}
\caption{Gaussian calibration optima. Left: optimal frequency for unpenalized population, expected iid minibatch ($n_{\rm b}=128$), and penalized population ($\lambda=10^{-5}$) objectives. Middle/right: unpenalized curves at $h=.10,.30$, subtracting each curve's value at zero and dividing by $C_h(0)$. Formula curves use fixed unit scales and an exact Base reference; they are not neural measurements.}\label{fig:calibration-limits}
\end{figure}

\paragraph{A separate result for scales recomputed on the batch.}
Take iid $X_i=(x_i,y_i)$ and $s_{\rm batch}=\max\{\operatorname{std}(x_1,\ldots,x_{n_{\rm b}}),.05\}$ in the same toy model, with divisor $n_{\rm b}-1$ and fixed $n_{\rm b}\ge2$. Define complex means $A_{\rm batch}=n_{\rm b}^{-1}\sum_i e^{ia_{\rm E}x_i/s_{\rm batch}}$ and $R_{\rm batch}=n_{\rm b}^{-1}\sum_i e^{ir_{\rm ref}x_i/s_{\rm batch}}$. Conditional on all $x_i$, $y_i$ remain iid; for $\theta=\omega^2$ and $q_\theta=e^{-h^2\theta/(2s_{\rm batch}^2)}$,
\begin{equation}\label{eq:batch-scale-toy}
 \E_y V_{\rm batch}=|q_\theta A_{\rm batch}-R_{\rm batch}|^2+(1-q_\theta^2)/n_{\rm b},
 \qquad
 \partial_\theta\E V_{\rm batch}=\frac{h^2}{n_{\rm b}}\E(s_{\rm batch}^{-2})+O(h^4).
\end{equation}
The last remainder is uniform on every fixed finite cap $0\le\theta\le\Omega^2$. For proof differentiate the first identity to obtain
$h^2q_\theta s_{\rm batch}^{-2}[\operatorname{Re}(A_{\rm batch}\overline R_{\rm batch})-q_\theta(|A_{\rm batch}|^2-n_{\rm b}^{-1})]$.
Use $|R_{\rm batch}-A_{\rm batch}|\le |r_{\rm ref}-a_{\rm E}|s_{\rm batch}^{-1}n_{\rm b}^{-1}\sum_i|x_i|$, $1-q_\theta\le h^2\Omega^2/(2s_{\rm batch}^2)$ and $s_{\rm batch}\ge.05$; the Gaussian first moment permits differentiation and expectation to commute. The derivative is positive for sufficiently small $h$, so zero is again the constrained minimum. This proof covers fixed directions in this toy model; PCA learned from the same bank does not automatically satisfy its conditional-independence step. Fixed-scale numerical cutoffs cannot be transferred to batch scales.

\paragraph{Penalty scaling and the compact joint limit.}
Supplementing \cref{prop:coercive-calibration}, on a compact parameter class, if $\mathfrak C_h/h^4\to\mathfrak C_4$ uniformly, $\mathcal R_{\rm pen}$ and $\mathfrak C_4$ are continuous, $\lambda_h/h^4\to\ell\ge0$, and returned parameters are $o(h^4)$-global minimizers, every cluster point minimizes $\mathfrak C_4+\ell\mathcal R_{\rm pen}$. Divide the minimizing inequality by $h^4$ and pass to a convergent subsequence to prove this. Thus $\lambda_h\asymp h^4$ retains an energy term in the leading problem and $\lambda_h=o(h^4)$ removes it under these hypotheses; neither scaling alone certifies order elevation.

\begin{corollary}[Compact joint calibration limit]
\label{cor:joint-calibration-limit}
Let $0\le\Omega<\infty$ be fixed, $n_{\rm b}(h)\in\{1,2,\ldots\}\cup\{\infty\}$ and $\lambda_h\ge0$. As $h\downarrow0$, suppose
$\tau_{{\rm batch},h}=1/(n_{\rm b}(h)h^2)\to\tau_{\rm batch}\ge0$
and $\ell_{{\rm pen},h}=\lambda_h/h^4\to\ell_{\rm pen}\ge0$.
Uniformly on $[0,\Omega^2]$, the joint objective divided by $h^4$ tends to
\begin{equation}\label{eq:joint-calibration-limit}
 F_{\rm joint}(\theta)=\frac{(1-\theta)^2}{4e}
             +(\tau_{\rm batch}+\ell_{\rm pen})\theta.
\end{equation}
Exact and $o(h^4)$-globally approximate minimizers converge to
$\operatorname{clip}(1-2e(\tau_{\rm batch}+\ell_{\rm pen}),[0,\Omega^2])$.
If $\Omega\ge1$, recovery of $|\omega_h|\to1$ is equivalent, under these
hypotheses, to $n_{\rm b}(h)h^2\to\infty$ and $\lambda_h/h^4\to0$.
\end{corollary}
\begin{proof}
Uniformly for bounded $\theta$,
$c_{\rm E}e^{-h^2\theta/2}-c_{\rm ref}
=e^{-1/2}h^2(1-\theta)/2+O(h^3)$ and
$1-c_{\rm pair}e^{-h^2\theta/2}=h^2\theta/2+O(h^4)$.
The finite limit of $\tau_{{\rm batch},h}$ implies
$n_{\rm b}(h)\to\infty$, proving uniform convergence after division by
$h^4$. The limit has second derivative $1/(2e)$ and the displayed unique
constrained minimum. If $\delta_h$ is the uniform error and
$\varepsilon_h=o(h^4)$ is the global optimization tolerance, comparison
to that minimum gives
$|\theta_h-\theta_*|^2\le4e(2\delta_h+\varepsilon_h/h^4)\to0$.
The final equivalence follows from nonnegativity of both limiting terms.
\end{proof}

\paragraph{Stationarity at zero.}
The gradient in $\omega$ vanishes at zero for every $\lambda$, with Hessian $2(\lambda-\lambda_{{\rm crit},{\rm batch}})$. Under a positive cap, $\lambda<\lambda_{{\rm crit},{\rm batch}}$ makes zero a local maximum and a nonzero minimum better. At threshold equality the minimum remains unique in $\theta$, despite a zero frequency Hessian.
More generally $(a_{\rm E}I+hK)(a_{\rm E}I+hK)^\top=a_{\rm E}^2I-h^2K^2$ has no linear skew response at zero. Thus the isotropic population parameterization $K=U_{\rm fit}\diag(b)V_{\rm fit}^\top-V_{\rm fit}\diag(b)U_{\rm fit}^\top$ has zero gradients in all parameters at $b=0$; Adam with exact gradients and zero moments stays there. Finite samples, anisotropy and score error can break symmetry, precluding a diagnosis of archived fits.

At $h=.1$, $\Omega=2$, batch size 1000 without penalty gives $|\omega_*|=.68280$; the penalized population objective with $\lambda=10^{-5}$ gives $.68214$. Together they give zero: the joint cutoff $6.9362\times10^{-6}$ is below the population cutoff $1.6953\times10^{-5}$. \Cref{fig:joint-calibration} contrasts this with \cref{cor:joint-calibration-limit}. These fixed-scale conditions on $n_{\rm b}$ and $\lambda_h$ certify neither neural order elevation nor optimization success.

FID's finite-sample, generator-dependent bias \citep{chong2020effectively} need not cancel with paired inputs and equal sample sizes. Correcting the characteristic estimator does not yield an unbiased FID difference.

\FloatBarrier

%% file: supplement_compact/experiments.tex
\section{Frozen-Model Experiments and Generated Samples}\label{app:experiments}
The image-model ledger contains 2,056 final records, including validation/reference datasets and eight reused aliases, not independent replications. The 112 local groups and 1,064 analytical classification records are separate. All reported favorable, adverse and null outcomes remain in the tables and code supplement. Its final reports retain the separate partial/cancelled inventories; these are excluded from the final-record count.

\subsection{Common frozen-model protocol and uncertainty}\label{app:neural-experiments}\label{app:cifar}\label{app:neural-protocol}
Official frozen EDM checkpoints cover conditional CIFAR VP/VE, unconditional CIFAR VP, FFHQ 64 VP and AFHQv2 VP \citep{karras2022edm,krizhevsky2009learning,karras2019style,choi2020starganv2}. With $v=(x-D_\theta)/\sigma$, the correction is $u=-A_\sigma v$, where $A_\sigma={P_{\rm PCA}}[{U_{\rm fit}}\diag({b}_\sigma){V_{\rm fit}}^\top-{V_{\rm fit}}\diag({b}_\sigma){U_{\rm fit}}^\top]{P_{\rm PCA}}^\top$ uses calibration PCA32. Width $r$ counts rank pairs, giving matrix rank $\leq2r$; original coefficients use eight global log-noise knots. These learned-score proxies have no exact marginal-preservation guarantee (\cref{app:diffusion,app:feature-limits}).

EDM grids use $\kappa_{\rm EDM}=7$, bounds $\max(.002,\texttt{net.sigma\_min})$ and $\min(80,\texttt{net.sigma\_max})$, and \texttt{net.round\_sigma} before appending zero. Comparisons share noise/labels. Euler $N$/same-partition Heun use $N$/$2N-1$ calls; \Cref{app:calibration-algorithm} specifies the endpoint, objective, fitting families and independent validation. Numerical precision and backend conventions are in the code supplement. Shared calibration/PCA is not an independent subspace audit.

FID uses compatible Inception 2048 and official moments \citep{heusel2017gans}; CIFAR also has polynomial KID \citep{binkowski2018demystifying}, while FFHQ/AFHQ lack raw-real-feature KID/PRDC. Feature SW2 uses 256 fixed projections, not population image-space $W_2$. Paired full-FID intervals ordinarily use 128 bootstrap replicates (64 for the original 32 profile/orientation controls in \cref{tab:controls-cost}) and are pointwise 95\%, conditional on fitted fields, excluding calibration uncertainty/familywise coverage. With only 128 resamples, tail estimates are noisy; the sign of a near-zero effect provides weak evidence. Generator-dependent finite-sample FID bias need not cancel in paired equal-size differences \citep{chong2020effectively}; no extrapolation or bias correction is applied. Refits vary calibration/PCA/optimizer/validation seeds and usually share each checkpoint's test bank; the eight complete N4 pipelines instead use separate fresh banks (\cref{tab:low-nfe}). Fixed-field $t$-intervals, KID subset std and PRDC draw ranges remain distinct; sample sizes are not pooled. All signs/zero amplitudes remain, and negative $\Delta\mathrm{FID}=$field--reference is better. The cost scope is in \cref{tab:recorded-cost}; timing logs and seed manifests accompany the code. These uncertainty conventions apply throughout this appendix unless a different type is stated.
\input{supplement_compact/calibration_algorithm}
\FloatBarrier

\subsection{Low-NFE benefits, calibration robustness and adverse checkpoints}\label{app:neural-cifar}\label{app:neural-checkpoints}
\Cref{tab:low-nfe} separates the initial conditional-CIFAR \emph{VP} checkpoint's fixed-field sampling blocks, independent calibrations and held-out damping tests. Full-amplitude FH can worsen N4 FID despite better KID; validation damping changes that outcome. Bold numerical entries indicate active improvements under the uncertainty stated in each table, not a universal guarantee.

\input{supplement_compact/result_tables/low_nfe}
The later checkpoint expansion uses conditional \emph{VE}, unconditional VP and FFHQ; \cref{tab:checkpoint-refits,tab:own-checkpoints} retain every favorable, adverse and null family. 

\input{supplement_compact/result_tables/checkpoint_refits}
\input{supplement_compact/visuals/calibration_refits}

\input{supplement_compact/galleries/afhq}
\paragraph{Projected LO, controls and efficiency.}
Projected LO uses full $a_w$ ($A'$ and $DvAv$), an empirical 384-gradient weighted Gram Galerkin projector, not population Hodge. It and matched FH fit 3072 states with shared 4096-state PCA. \Cref{tab:controls-cost} keeps the failed prespecified gate, every shuffle/orientation, the grid-specific comparison and the adverse one-call contrast. Random-control distributions and optimizer/grid diagnostics are in the code supplement.  Full-cost limits remain in \cref{tab:recorded-cost}.

\input{supplement_compact/result_tables/controls_cost}
\FloatBarrier
\subsection{Own-field integration, leakage and diversity}\label{app:neural-own}
\Cref{tab:own-checkpoints} puts FID beside distance to each field's \emph{own} Heun 256 flow.  Independent sampling references, coupled refinement, finite-map leakage tests and PRDC \citep{naeem2020reliable} diagnose different questions (\cref{fig:diversity-rank-profiles}).

\input{supplement_compact/result_tables/own_checkpoint}
\input{supplement_compact/visuals/checkpoint_own}
\input{supplement_compact/result_tables/reference_diversity}
\noindent\begin{minipage}{\linewidth}
Refinement alone does not certify the reference errors required by \cref{eq:reference-error-budget}. No score mismatch to the actual learned Base law was measured, leaving continuous-law drift and discretization contributions to FID unidentified. Non-rejection is not marginal equality; normalized finite-map MMD/mean bounds do not bound FID or SW2, and polynomial features have law-difference blind spots. Mixed PRDC diagnostics and reference moments do not prove absence of mode collapse.
\end{minipage}\par

\input{supplement_compact/visuals/diversity}

\FloatBarrier
\subsection{Redistributing a frozen temporal profile}\label{app:neural-mechanisms}\label{app:neural-temporal-profile}
Four cells retain 22 frozen-profile conditions each (Base/FH/reverse/flat/early/late, eight shuffles/eight orientations), yielding 88 pilot 20k datasets (\cref{fig:compact-temporal}) and 24 fresh 50k confirmations fixed regardless of pilot sign. Profiles match integrated correction energy at calibration nodes; orientations preserve spectra without energy rescaling.

\input{supplement_compact/result_tables/temporal_confirmation}
\begin{figure}[H]\centering
\includegraphics[width=\linewidth]{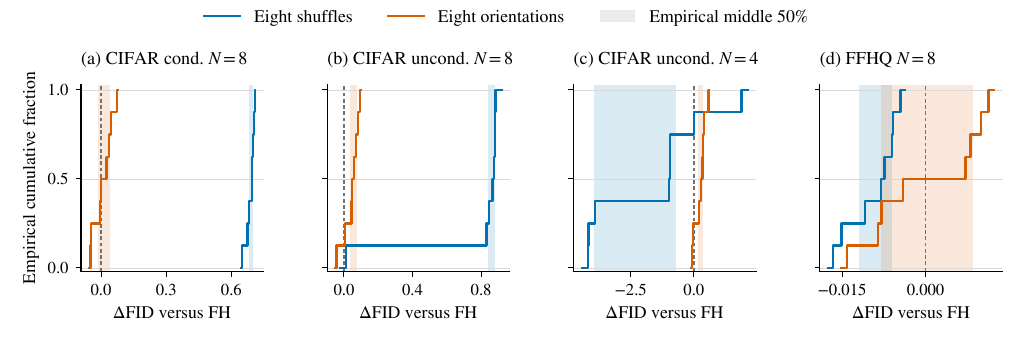}
\caption{Descriptive empirical distributions of eight time shuffles/eight orientations per cell. Shading is the descriptive middle 50\% of the eight draws, not a confidence interval. Curves are fixed-field pilot controls, not $p $-values; fresh 50k confirmations are tabulated in \cref{tab:temporal-confirmation}. Profile-energy matching and spectrum-preserving orientation controls are different interventions.}\label{fig:compact-temporal}\label{fig:neural-temporal}
\end{figure}

\input{supplement_compact/visuals/confirmation}

Fixed-seed CIFAR N8 examples are in \cref{fig:samples-cifar-eight}; the matched N4/FFHQ panels are in the code supplement\textquotesingle s extended gallery.
\FloatBarrier
\input{supplement_compact/galleries/temporal_cifar}

\noindent\begin{minipage}{\linewidth}
\subsection{Native window fitting and cross-transfer}\label{app:neural-window}
Native window fitting (\cref{fig:window-fitting}) comprises 40 fits: four cells $\times$ five independent refits $\times$ early/late, with the mask \emph{inside} the objective. Update, seed and shared-PCA rules are in \cref{tab:calibration-families,app:calibration-algorithm}. Actual-node coefficients are zero on inactive nodes, with piecewise-linear log-noise interpolation rather than smooth compact support. Chronological source-node $j$ transfers to receiver-node $j$ without reversal/rescaling/retraining. All 351 datasets and 132 paired contrasts complete, including 40 direct transfer--native contrasts. The original global-eight-knot comparison changes coefficient layout; equal rank/cap does not ensure equal realized norm, energy or spectra.

\begin{table}[H]\centering\small
\caption{All native-window cells, five refits per row. Native--Base and transferred--native ranges retain all point estimates (four-decimal rounding); zero selections are counted explicitly. The last column counts direct-transfer conditional 95\% intervals below/above zero, not simultaneous calibration coverage. Bold marks active all-five native gains or active all-five transfer gains; the latter has only three below-zero intervals. Zero-amplitude backend residuals are not learned effects.}\label{tab:compact-window}\label{tab:neural-window}
\begin{tabular}{llllrr}\toprule
Checkpoint/N & Window & Native--Base & Transfer--native & $\alpha=0$ & Below/above\\\midrule
CIFAR VE/8 & early & $[-0.0018,+0.0012]$ & $[+0.0000,+0.0020]$ & 2 & 0/0\\
CIFAR VE/8 & late & $[-0.0432,+0.0015]$ & $[+0.0003,+0.0511]$ & 0 & 0/1\\
CIFAR uncond./8 & early & $[-0.0047,+0.0012]$ & $[+0.0004,+0.0040]$ & 0 & 0/0\\
CIFAR uncond./8 & late & $[-0.0486,+0.0004]$ & $[-0.0179,+0.0430]$ & 1 & 0/0\\
CIFAR uncond./4 & early & $[-0.0079,+0.0026]$ & $[-0.0065,+0.0082]$ & 2 & 1/0\\
CIFAR uncond./4 & late & $\boldsymbol{[-4.3635,-1.7421]}$ & $[+0.5902,+2.4492]$ & 0 & 0/5\\
FFHQ/8 & early & $[-0.0042,-0.0000]$ & $[-0.0001,+0.0046]$ & 1 & 0/1\\
FFHQ/8 & late & $[-0.0046,+0.0012]$ & $\boldsymbol{[-0.0403,-0.0044]}$ & 0 & 3/0\\
\bottomrule\end{tabular}\end{table}

Table~\ref{tab:compact-window} retains the heterogeneous cells, including all five $\alpha=1$ unconditional-N4 late fits and FFHQ late-to-early transfer benefits in all five refits (three intervals below zero). Counts such as 5/5 on a common test bank imply neither near-one new-pipeline success nor a universal/pure causal temporal-position effect or simultaneous refit guarantee. Fixed-refit CIFAR N8 examples are in \cref{fig:samples-native-eight}; the matched N4/FFHQ panels remain in the extended gallery.

\end{minipage}\par

\input{supplement_compact/galleries/native_cifar}
\FloatBarrier
\subsection{Solver cross-fit and rank/cap specificity}\label{app:neural-solver-rank}
Matched solver cross-fits use new calibration/PCA, native validation and actual Heun second-call input gradients. \Cref{tab:solver-crossfit} retains native and transferred fields separately: CIFAR's Euler-native benefits do not transfer to Heun, whereas the unconditional Heun-native correction is disabled (\cref{fig:solver-crossfit}). Euler 8/same-partition Heun 8 require 8/15 calls and are not equal-cost comparisons.

\input{supplement_compact/result_tables/solver_crossfit}
\input{supplement_compact/visuals/crossfit}

\Cref{tab:rank-cap} retains all 30 rank/cap fits, including zero selections and nonmonotone FID behavior (\cref{fig:compact-rank}). All 36 dense datasets complete; lower point own-distance among active fields carries no population-SW2 significance. Diagnostics distinguish $A_\sigma$, $L_\sigma^{\mathrm{res}}=A_\sigma/\sigma$, cap hits/rank/energy and $\max|h/\sigma|\|A_\sigma\|$.

\input{supplement_compact/result_tables/rank_cap}
\begin{figure}[!htbp]\centering
\includegraphics[width=\linewidth]{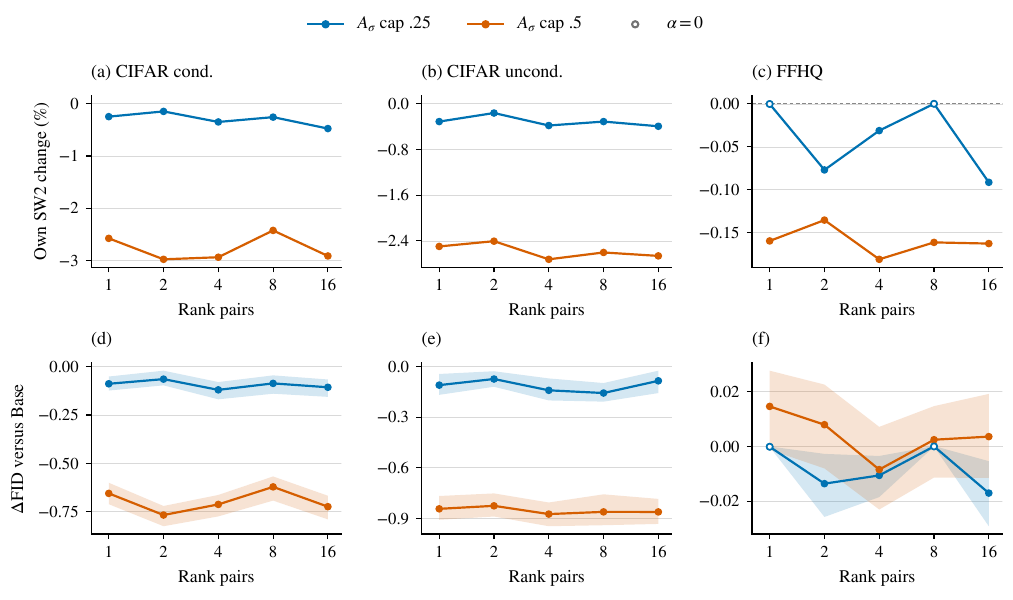}
\caption{Top: change in own-field SW2 across rank pairs and caps, with each field compared against its own Heun 256 reference; no population-SW2 interval is supplied. Bottom: FID change versus Base with conditional pointwise 95\% paired-bootstrap bands. Zero amplitudes and independent Base--Base distances remain separate. Rank does not yield monotone FID gains.}\label{fig:compact-rank}\label{fig:neural-rank-own}
\end{figure}

\FloatBarrier
\subsection{Local one-step response}\label{app:neural-local}
The 112 groups ($(8+8+4+8)$ intervals $\times$ Base/FH/early/late) use 4096 common initial states: one Euler step versus the \emph{same field's} 32-substep Heun interval, audited with 64 substeps on 1024 shared particles (\cref{fig:compact-local}). These FH/early/late fields are the original FH and its redistributed profiles, not Phase 8's native-window fits. \Cref{tab:local-response} retains all interval outcomes. A zero correction at the left Euler node does not imply an identically zero dense interval: late interpolation becomes active inside four such intervals. Entirely inactive intervals and nonbitwise residuals remain separate.

\input{supplement_compact/result_tables/local_response}
Positive-noise substeps are log-uniform; the terminal interval is linear with its final corrector omitted. The code supplement records integrator/backend settings and the additional characteristic/Taylor diagnostics.

\begin{figure}[!htbp]\centering
\includegraphics[width=\linewidth]{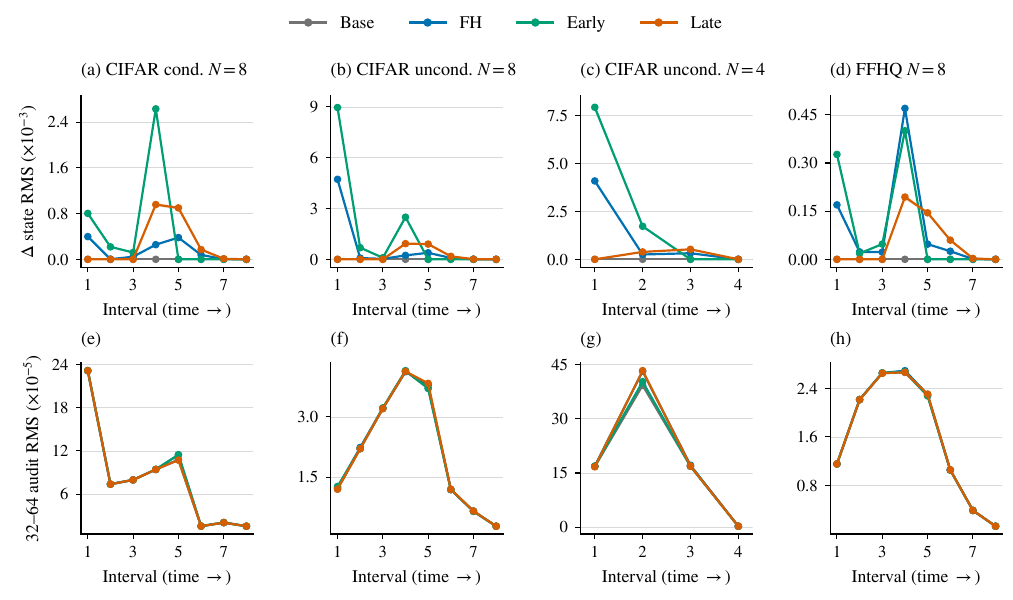}
\caption{Top: local coupled-state RMS error minus the corresponding Base error over all coarse intervals. Bands are 512-replicate conditional pointwise 95\% paired bootstraps; all 84 outcomes are counted in \cref{tab:local-response}. Bottom: own-field 32/64-substep RMS audit on the same bank, a numerical-convergence check without an independent floor or confidence interval. Denoised-feature proxy and characteristic errors remain separate.}\label{fig:compact-local}\label{fig:neural-local}
\end{figure}

Recorded denoised-feature and characteristic/Taylor responses are retained in the code supplement; these proxies do not certify full-state error improvement.

\FloatBarrier
\clearpage
\noindent\begin{minipage}{\linewidth}
\subsection{Approximately matched-call solver context}\label{app:neural-cost}
The 24 fresh 20k datasets and 18 fixed contrasts compare Euler 8 Base/FH, Euler 16 Base/frozen transferred FH8 and Base Heun 4/5/8/9. \Cref{tab:cost-context} retains the favorable low-budget bracket, its high-budget reversal, the same-Euler FFHQ adverse result and nulls (\cref{fig:compact-cost}). Approximate call brackets establish neither exactly equal cost nor call savings; the Euler 16 results concern transferred FH8, not a native N16 fit.

\input{supplement_compact/result_tables/cost_context}
\begin{figure}[H]\centering
\includegraphics[width=\linewidth]{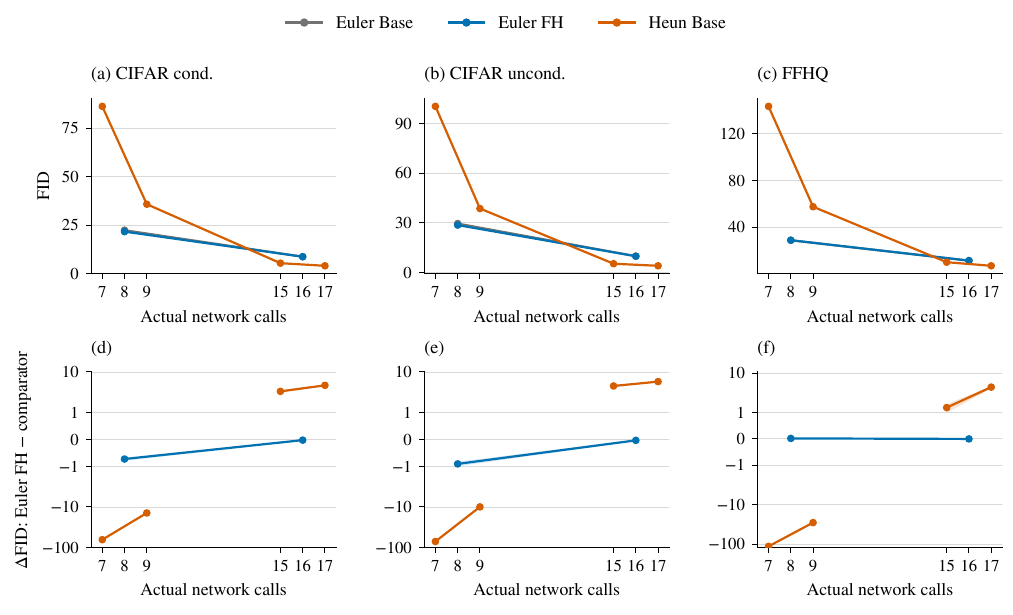}
\caption{Top: recorded FID versus actual network calls. Bottom: Euler FH minus the indicated comparator, with conditional pointwise 95\% paired-bootstrap bands; these difference intervals are not absolute-FID intervals. Competitiveness changes across budgets; Euler 16 carries the frozen N8 field. Connecting measured points does not imply interpolated call savings.}\label{fig:compact-cost}\label{fig:neural-cost}
\end{figure}

\end{minipage}\par
\FloatBarrier
\subsection{Cached objective--own-error associations}\label{app:neural-correlations}
A posthoc cached study of all 30 rank/cap fields finds heterogeneous objective--own-error associations across checkpoint/cap families (five fixed ranks each). The common-map selected-field objective and raw pre-update training histories are different axes, with penalty and amplitude mismatch. Zero fields have exact $\Delta\mathcal C_{\rm cal}=0$. These descriptive associations imply neither causality nor a universal objective--own-error law. All coefficients, the calibration-objective plot and the separate native-history analysis are retained in the code supplement, \texttt{docs/TECHNICAL\_REPRODUCIBILITY.md}.

\subsection{A null original independent-path RF pilot}\label{app:neural-rf}
The original independently coupled 1-RF checkpoint \citep{liu2023rectified} uses the path-specific adapter in \cref{app:rf-gauge}; its exact identity gives no neural preservation certificate or automatic extension to correlated reflow. Gaussian initialization at $\varepsilon=.001$ approximates the initial law; step $.999/N$ differs from the original $1/N$ sampler, so published-FID reproduction is not claimed. RF uses centered inputs with clipped uint8 map $127.5(x+1)$; EDM uses $127.5x+128$. Exact checkpoint restoration, time labels and architecture flags are in the code supplement. Two native N8/N16 fits follow \cref{tab:calibration-families}, share N8 node-3 PCA, select on independent 2048 validation and test on fresh 10k banks. All 20 datasets complete. \Cref{tab:rf-pilot} retains both paired FID nulls and all four own-error point measurements (\cref{fig:RF-pilot-response}). The tiny own-error changes have mixed signs. One checkpoint and two fits are insufficient to establish family-wide robustness or a convincing FID improvement; fixed-seed examples are in \cref{fig:samples-rf}.

\input{supplement_compact/result_tables/rf_pilot}

\input{supplement_compact/visuals/rf}

\input{supplement_compact/galleries/rf}

%% file: supplement_compact/calibration_algorithm.tex
\subsection{Calibration objective, fitting and selection}\label{app:calibration-algorithm}
\paragraph{Finite characteristic objective.}
Let $x_{ki},v_{ki}$ be cached Base states and velocities at coarse node $k$, $h_k$ its signed step, and $T_k^B(x_{ki})$ the Base endpoint from eight Heun substeps. EDM substeps are log-uniform at positive noise and linear on the terminal interval, omitting the terminal corrector; RF substeps are linear in time. These approximate Base targets differ from the frozen field's own-field references. Flattened states have dimension $d$; the uncentered velocity PCA basis $P_{\rm PCA}$ has 32 columns. Concatenate these with 96 seeded standard Gaussian directions, normalizing the $M_{\rm feat}=128$ columns to unit Euclidean norm to obtain $\widetilde d_j$. The earlier actual-Heun diagnostic instead uses only 64 Gaussian objective directions; its PCA32 defines the gauge subspace. Define
\begin{align}
 s_{kj}^{\rm feat}&=\max\{\operatorname{std}_{i\in\mathcal S_k}(x_{ki}^{\top}\widetilde d_j),\,0.05\},
 &d_{kj}^{\rm feat}&=\widetilde d_j/s_{kj}^{\rm feat},\label{eq:calibration-frequency}\\
 \phi_{kj}(z)&=\bigl(\cos((d_{kj}^{\rm feat})^{\top}z+\vartheta_j),\,
 \sin((d_{kj}^{\rm feat})^{\top}z+\vartheta_j)\bigr),\notag\\
 \mathcal C_k(\Theta;\mathcal B)&=\frac1{M_{\rm feat}}\sum_{j=1}^{M_{\rm feat}}
 \left\|\langle\phi_{kj}(\widehat x_k^\Theta)\rangle_{\mathcal B}
       -\langle\phi_{kj}(T_k^B(x_k))\rangle_{\mathcal B}\right\|_2^2,\label{eq:calibration-characteristic}\\
 \mathcal L_k(\Theta;\mathcal B)&=\mathcal C_k(\Theta;\mathcal B)
   +\lambda_{\rm pen}\frac{\langle\|c_k^\Theta\|_2^2/d\rangle_{\mathcal B}}
                  {\max\{\langle\|v_k\|_2^2/d\rangle_{\mathcal B},10^{-6}\}}.
 \label{eq:calibration-penalty}
\end{align}
Here $\operatorname{std}$ uses divisor $|\mathcal S_k|-1$, $\langle\cdot\rangle$ is the empirical particle mean, $\Theta=(U_{\rm fit},V_{\rm fit},b)$, $c_k^\Theta=A_{\sigma_k}v_{ki}$ for EDM and $c_k^\Theta=\delta v_{t_k}$ for RF. Thus $\widehat x=x+h(v-c)$ for EDM Euler, and $\widehat x=x+h(v+c)$ for RF. Full-cloud unpenalized $\mathcal C_{\rm cal}=N^{-1}\sum_k\mathcal C_k$ differs from the penalty-inclusive logged loss; cyclic logs cover only the selected node/batch. Original FH uses seeded $\vartheta_j\sim\operatorname{Unif}[0,2\pi)$; matched solver/window and RF fits use $\vartheta_j=0$. A common phase rotates the retained sine/cosine pair, preserving the ideal squared difference. This finite feature discrepancy is not image FID, exact Hodge projection or a whole-law norm. \Cref{app:calibration-limits} analyzes minibatch, penalty and initialization limitations; archived objectives remain unchanged.

\paragraph{Actual endpoint maps.}
Write $w_k^\Theta$ for the corrected field at node $k$ ($v-A v$ for EDM, $v+\delta v$ for RF). The implemented Euler predictor and Heun endpoint are
\begin{equation}\label{eq:calibration-endpoint}
\widehat x_{\rm E}=x+h_k w_k^\Theta(x),\qquad
\widehat x_{\rm H}=x+\frac{h_k}{2}\bigl[w_k^\Theta(x)+w_{k+1}^\Theta(\widehat x_{\rm E})\bigr].
\end{equation}
Here $h_k<0$ for EDM noise and $h_k>0$ for RF time. EDM uses $\widehat x_{\rm E}$ at receiving noise zero; RF retains the second call at $t=1$. The receiver evaluates the frozen network and gauge at the predictor, not a cached Base receiver. \Cref{eq:calibration-characteristic} uses the specified fitting solver's endpoint.

\paragraph{Fitted parameters and constraints.}
Cached trajectories, targets, PCA and feature maps remain fixed during fitting. Only $U_{\rm fit},V_{\rm fit},b$ are optimized; Heun fits differentiate through the actual receiving-network input in \cref{eq:calibration-endpoint}. Orthogonal latent factors and zero coefficients initialize the field. Adam updates are followed by column-norm and operator-cap projection; under-cap fields are not expanded. Linear interpolation preserves the cap. Native-window masks act \emph{inside} the endpoint objective, with inactive coefficients kept at zero. \Cref{tab:calibration-families} gives the full-cloud/minibatch rule and hyperparameters. Precision, autograd and projection implementation details are provided in the code supplement.

\begin{table}[H]\centering\small
\caption{Fitting families. ``Full'' uses all listed calibration particles/nodes per update; ``cyclic'' uses one node. Unless stated otherwise, $P_{\rm PCA}$ has 32 columns and the map has 128 directions. $r$ counts rank pairs; calibration sizes differ from test sizes.}\label{tab:calibration-families}
\begin{tabular}{p{.25\linewidth}p{.25\linewidth}p{.20\linewidth}p{.20\linewidth}}\toprule
Family & Update/scale bank & Adam updates, rate; $\lambda_{\rm pen}$ & Rank/cap/profile\\\midrule
Primary full-cloud FH; later checkpoint/refit FH; rank/cap & Full 4096, full-node mean $N^{-1}\sum_k\mathcal L_k$; full-node scale. Matched projected-LO FH uses 3072. & 800, .012; $10^{-5}$ & $r=2$ at N8, 8 at N4/N16/N20; cap .5. Rank study $r=1,2,4,8,16$, caps .25/.5; eight global log-noise knots.\\
Early actual-Heun diagnostic & Cyclic nodes, 128 without replacement from 4096/node; batch scale; 64 Gaussian directions only. & 400, .008; 0 & $r=2/8$ at N8/N16, cap .5; eight global knots.\\
Matched Euler/Heun cross-fit & Cyclic eight nodes, 128 without replacement from 4096; batch scale. & 800, .008; 0 & $r=2$, cap .5; eight global knots; identical fit seeds within solver pair.\\
Native early/late windows & Cyclic active half only, 128 without replacement from 4096; batch scale. & 800 active updates, .008; 0 & $r=8/2$ for N4/N8, cap .5; actual-node masked coefficients.\\
RF native N8/N16 & Cyclic nodes, 128 \emph{with} replacement from 4096; full-node scale fixed beforehand. & 800, .008; $10^{-5}$ & $r=2$, $\|K_t^{\rm RF}\|\leq.5$; eight linear-time knots, not the EDM cap.\\
\bottomrule\end{tabular}
\end{table}

\paragraph{Calibration, validation and transfer.}
Checkpoint, grid, random streams and evaluation scope are fixed before fitting. Cache the Base path, eight-substep endpoints and velocity PCA32; draw the feature directions once. Each update uses the family-specific full-cloud or cyclic-node batch in \cref{tab:calibration-families}, evaluates the corrected endpoint and loss, and projects the returned field onto the cap. Unconditional N4/N8 share calibration noise and N8 node-3 PCA within a refit; paired native windows share initialization and feature directions.

On independent validation noise/labels (4096 samples in Phase 2; 2048 in later pipelines), compare $\alpha\in\{0,.125,.25,.5,.75,1\}$ to BaseHeun256 by empirical Inception SW2 with 256 fixed projections. Select the first minimum and freeze the field before the prespecified test bank. Transfer amplitudes unchanged without receiver tuning; retain every negative outcome and zero selection. Exact RNG streams, seed formulas, checkpoint hashes and stage manifests are in the anonymous code supplement, \texttt{docs/TECHNICAL\_REPRODUCIBILITY.md}.

Weak LO replaces \cref{eq:calibration-characteristic} by the squared cached leading Euler characteristic defect with finite-difference Base acceleration. Projected LO instead fits $\frac18\sum_k|h_k|\,b_k^\top G_{k,\rm Gram}^{+}b_k$, normalized by its initial value: 3072 particles, 384 gradient features, relative eigenvalue cutoff $10^{-7}$, 600 Adam updates at .006 and cap .5. The remaining 1024 states share the 4096-state PCA and are not an independent subspace audit. Derivative-cache construction and JVP counts are documented in the code supplement.

\paragraph{Zero correction and cost scope.}
At $\alpha=0$ the correction is mathematically disabled; tiny nonbitwise AFHQ residuals are not learned-gauge effects, and their precise backend cause remains unresolved. Effects comparable to those residuals or dense-refinement distances remain unresolved.

\begin{table}[H]\centering\small
\caption{Cost accounting and its limits. Detailed recorded timings are in the code supplement; heterogeneous execution timings do not establish latency gains or amortization.}\label{tab:recorded-cost}
\begin{tabularx}{\linewidth}{>{\raggedright\arraybackslash}p{.25\linewidth}X}\toprule
Component & Recorded scope and limitation\\\midrule
Coarse sampling & Actual network calls are reported with each solver comparison. Worker generation time excludes fitting, feature extraction, I/O and scheduling.\\
Offline selection & Calibration, PCA, fitting, dense references and six validation candidates are additional costs. No controlled end-to-end deployment or break-even measurement is available.\\
Diagnostic timing & Elapsed local-shard hours include feature calls/I/O; calibration-cloud wall seconds and sampling worker-hours are distinct and cannot be added.\\
\bottomrule\end{tabularx}\end{table}

%% file: supplement_compact/result_tables/low_nfe.tex
\begin{table}[htbp]\centering\small
\setlength{\tabcolsep}{3pt}\renewcommand{\arraystretch}{1.12}
\caption{Initial conditional-CIFAR VP results (upper block), distinct from later conditional VE. Differences are FH--Base unless the comparator is explicitly named. Fixed-field means and $t$95 intervals concern 28 matched 50k sampling blocks, not 28 fits. Refit means/$t$95 concern eight independent calibration/PCA/optimizer runs on one common 20k bank. The eight N4 full pipelines instead use separate calibration, validation and fresh 50k test banks per pipeline; their range is descriptive and intervals conditional per field. All fixed-field methods and adverse/null outcomes remain. Bold denotes active improvements under the stated uncertainty, not calibration-population coverage.}\label{tab:low-nfe}
\begin{tabular}{@{}p{.275\linewidth}cp{.31\linewidth}p{.29\linewidth}@{}}\toprule
Study/scope & N & FID or $\Delta$FID & Uncertainty/interpretation\\\midrule
Initial fixed fields, 50k & 8 & \shortstack{Base: $22.4244$; Random: $21.9074$\\Weak LO: $22.0634$; FH: $21.6174$} & FID point values\\
Full amplitude $\alpha=1$, 50k & 4 & \shortstack{Base: $67.653$; FH: $78.470$\\Weak LO: $62.435$; Random: $70.647$} & FH worsens FID despite lower KID\\
28 fixed-field seed blocks & 8 & $\boldsymbol{-0.80831}$ & $t$95 $[-0.81772,-0.79890]$\\
28 fixed-field seed blocks & 16 & $\boldsymbol{-0.03682}$ & $t$95 $[-0.04009,-0.03354]$\\
28 fixed-field seed blocks & 20 & $\boldsymbol{-0.01890}$ & $t$95 $[-0.02195,-0.01586]$\\
Frozen validation choice, 50k & 4 & $\boldsymbol{-4.41279}$ ($\alpha=0.75$) & Paired95 $[-4.68740,-4.22669]$\\
Frozen validation choice, 50k & 8 & $\boldsymbol{-0.78505}$ ($\alpha=1$) & Paired95 $[-0.81966,-0.75101]$\\
Eight refits,common 20k bank & 8 & $\boldsymbol{-0.77578}$ & Refit $t$95 $[-0.81993,-0.73164]$\\
Eight refits,common 20k bank & 16 & $\boldsymbol{-0.03911}$ & Refit $t$95 $[-0.04484,-0.03337]$\\
Eight full pipelines,fresh 50k each & 4 & $\boldsymbol{[-4.4679,-1.3642]}$ & Descriptive range; 8/8 paired intervals below zero\\
\midrule\multicolumn{4}{l}{Original AFHQ checkpoint: separate frozen fields and fresh 50k banks}\\
AFHQ original FH--Base & 8 & $\boldsymbol{-0.01667}$ & Paired95 $[-0.02405,-0.01053]$\\
AFHQ original FH--Random & 8 & $-0.00633$ & Paired95 $[-0.01668,+0.00160]$\\
AFHQ original FH--Base & 16 & $-0.00037$ & Paired95 $[-0.00142,+0.00087]$\\
AFHQ original FH--Random & 16 & $-0.00001$ & Paired95 $[-0.00104,+0.00106]$\\
\bottomrule\end{tabular}\end{table}

%% file: supplement_compact/result_tables/checkpoint_refits.tex
\begin{table}[htbp]\centering\small
\setlength{\tabcolsep}{3pt}\renewcommand{\arraystretch}{1.12}
\caption{All 25 Phase5 refits plus all five AFHQ refits on common 20k test banks per checkpoint/NFE. Upper rows give five-refit $\Delta$FID ranges, point improve/worsen counts and counts of conditional pointwise 95\% intervals below/above zero (not refit-population coverage). Lower rows give individual $\Delta$FID and those intervals. Bold marks active all-five improvements with five below-zero intervals, or an individual active below-zero interval. $\alpha=0$ disables the gauge: its residual signs, including a positive interval, are backend observations rather than learned effects.}\label{tab:checkpoint-refits}
\begin{tabular}{@{}lclp{.235\linewidth}cc@{}}\toprule
Checkpoint/N & Refit scope & $\alpha$ & $\Delta$FID/range & Improve/worsen & Below/above\\\midrule
CIFAR VE/8 & five refits & $1$ & $\boldsymbol{[-0.7555,-0.4993]}$ & 5/0 & 5/0\\
CIFAR uncond./8 & five refits & $1$ & $\boldsymbol{[-0.8935,-0.6967]}$ & 5/0 & 5/0\\
CIFAR uncond./4 & five refits & $0.75$ & $[-1.0762,+0.5071]$ & 3/2 & 2/1\\
FFHQ/4 & five refits & $1$ & $\boldsymbol{[-0.2814,-0.2001]}$ & 5/0 & 5/0\\
FFHQ/8 & five refits & $1$ & $[+0.0117,+0.0161]$ & 0/5 & 0/4\\
\midrule\multicolumn{6}{l}{AFHQ/8: every independent refit, including the three disabled corrections}\\
AFHQ/8 & refit 0 & $1$ & $-0.00382$ & \multicolumn{2}{c}{$[-0.01499,+0.00831]$}\\
AFHQ/8 & refit 1 & $0$ (off) & $+0.00007$ & \multicolumn{2}{c}{$[-0.00001,+0.00016]$}\\
AFHQ/8 & refit 2 & $0$ (off) & $+0.00006$ & \multicolumn{2}{c}{$[-0.00001,+0.00013]$}\\
AFHQ/8 & refit 3 & $0$ (off) & $+0.00012$ & \multicolumn{2}{c}{$[+0.00002,+0.00020]$}\\
AFHQ/8 & refit 4 & $1$ & $\boldsymbol{-0.01199}$ & \multicolumn{2}{c}{$[-0.02340,-0.00279]$}\\
\midrule AFHQ random & 16 draws & -- & 1/16 beats original FH & \multicolumn{2}{c}{Descriptive; not a $p$-value}\\
\bottomrule\end{tabular}\end{table}

%% file: supplement_compact/visuals/calibration_refits.tex
\begin{figure}[!htbp]
\noindent Because the refits share a common test bank, all-fit success counts should not be extrapolated to new pipelines. All independent refits appear below, with checkpoint/NFE context in \cref{fig:checkpoint-nfe}; full random-draw distributions are retained in the code supplement.\par\smallskip
\centering
\includegraphics[width=\linewidth]{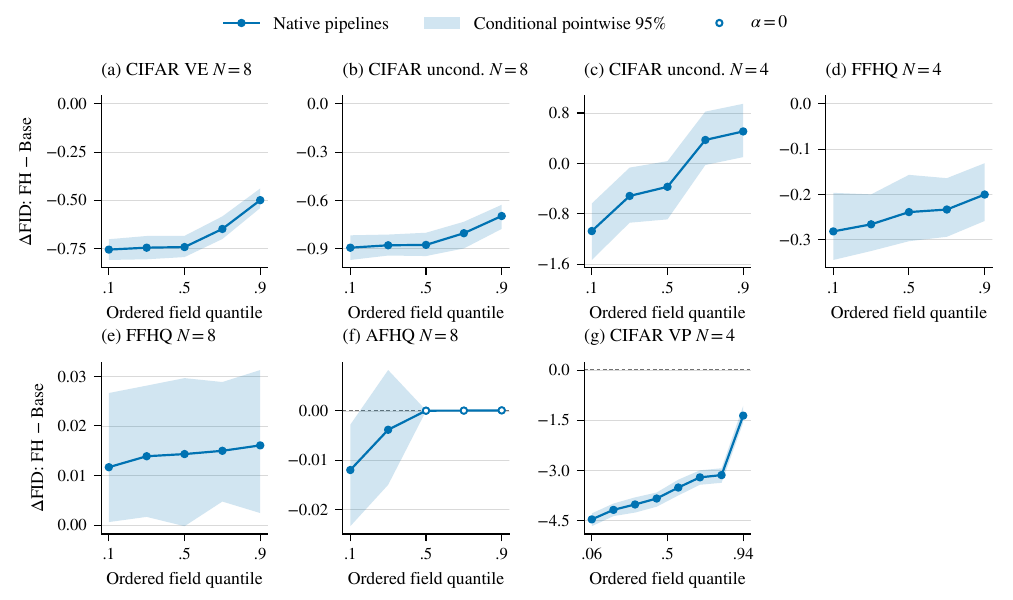}
\caption{All 30 Phase 5/6 independent refits, ordered within each prespecified family by observed FID effect; the separate panel retains all eight conditional-CIFAR N4 fresh-50k pipelines. Bands are the existing conditional intervals for each fitted field, not confidence bands for a refit population or quantile trend. Open symbols retain $\alpha=0$; their tiny nonbitwise residuals are not learned effects.}\label{fig:calibration-refits}
\end{figure}

%% file: supplement_compact/galleries/afhq.tex
\begin{figure}[!htbp]\centering
\includegraphics[width=\linewidth]{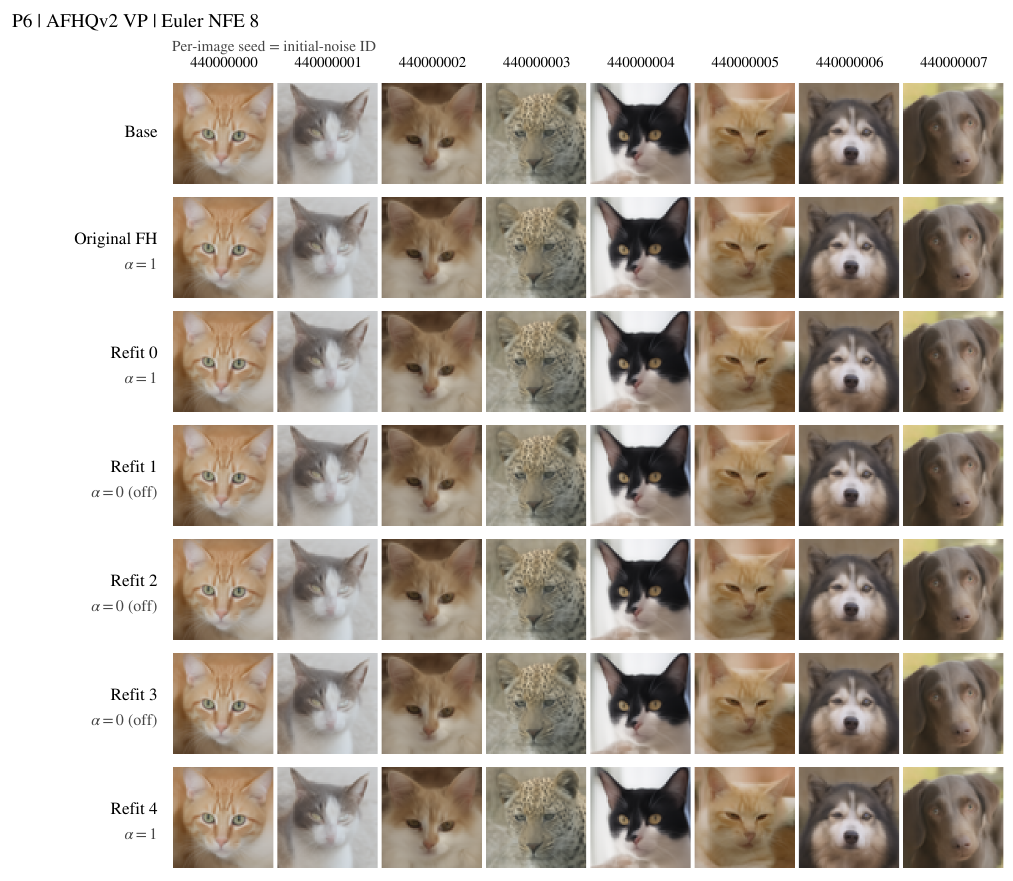}
\caption{AFHQv2 VP, Euler N8: Base, original FH and all five native refits on the same first eight seeds of the common 20k test bank. Refits 1--3 select $\alpha=0$; any tiny backend pixel residual is not a learned-gauge effect. The common appearance does not establish marginal equivalence.}\label{fig:samples-afhq}
\par\smallskip\noindent AFHQ's zero-amplitude corrections are mathematically disabled: their nonbitwise residuals, even an above-zero interval, are not learned harm; the exact backend cause remains unresolved. The active AFHQ examples are linked in \cref{fig:samples-afhq}.\par
\end{figure}

%% file: supplement_compact/result_tables/controls_cost.tex
\begin{table}[htbp]\centering\small
\setlength{\tabcolsep}{3pt}\renewcommand{\arraystretch}{1.12}
\caption{Original conditional-CIFAR VP projected-LO, profile/orientation, grid and call-context tests. Unless marked FID, numbers are method--named-reference $\Delta$FID. Pilot/draw ranges are descriptive; original 32 control intervals use 64 bootstrap resamples, distinct from the 128-replicate later contrasts. Orientation results preserve spectra without state-conditional energy matching. The screened fresh 50k one-call comparison is exploratory; it finds no measured integer-call saving.}\label{tab:controls-cost}
\begin{tabular}{@{}p{.31\linewidth}p{.235\linewidth}p{.38\linewidth}@{}}\toprule
Comparison & Recorded result & Scope/uncertainty\\\midrule
Projected LO--matched FH,8 calls & $+0.07929$ & Conditional paired95 $[+0.01858,+0.13887]$\\
Projected LO--matched FH,16 calls & $-0.00337$ & Conditional paired95 $[-0.03822,+0.03062]$\\
Prespecified advancement & $|\Delta\mathrm{FID}|\geq.10$:FAIL & Direction-agnostic pilot gate; no gated50k result\\
16 time-knot shuffles vs FH & $[+0.11726,+0.62612]$ & 0/16/0 intervals below/above/containing zero; 64 resamples\\
16 spectrum orientations vs FH & $[-0.08190,+0.02965]$ & 1/0/15 intervals below/above/containing zero; 64 resamples\\
Log-grid N8: transferred/native fit & $22.2376/22.3275$ FID & Matched20k point values; no direct interval\\
FH8--Base9,fresh 50k & $+3.30886$ & Conditional paired95 $[+3.23315,+3.36578]$\\
\bottomrule\end{tabular}\end{table}

%% file: supplement_compact/result_tables/own_checkpoint.tex
\begin{table}[htbp]\centering\small
\setlength{\tabcolsep}{3pt}\renewcommand{\arraystretch}{1.12}
\caption{Nine final checkpoint/NFE cells: FID uses fresh 50k banks and own-SW2 uses 20k banks. FH minus Base FID has conditional pointwise 95\% paired intervals, alongside feature SW2 to each field's own Heun 256 flow (point measurements). Bold denotes below-zero active FID intervals, not a confidence statement about SW2. All own distances decrease (0.01\%--14.04\%), including two adverse FID cells; the FFHQ16 FID interval includes zero.}\label{tab:own-checkpoints}
\begin{tabular}{@{}lccp{.255\linewidth}p{.30\linewidth}@{}}\toprule
Checkpoint/N & $\alpha$ & $\Delta$FID & Conditional 95\% & Own SW2: Base $\to$ FH\\\midrule
CIFAR VE/4 & $0.75$ & $\boldsymbol{-5.52007}$ & $[-5.68471,-5.29807]$ & $0.124881\to0.107342$\\
CIFAR VE/8 & $1$ & $\boldsymbol{-0.67896}$ & $[-0.72016,-0.64694]$ & $0.067337\to0.065565$\\
CIFAR VE/16 & $1$ & $\boldsymbol{-0.03520}$ & $[-0.04953,-0.01744]$ & $0.035883\to0.035726$\\
CIFAR uncond./4 & $0.75$ & $+0.40445$ & $[+0.11676,+0.65074]$ & $0.160136\to0.147643$\\
CIFAR uncond./8 & $1$ & $\boldsymbol{-0.83403}$ & $[-0.88015,-0.77928]$ & $0.084550\to0.082423$\\
CIFAR uncond./16 & $1$ & $\boldsymbol{-0.03262}$ & $[-0.04973,-0.01568]$ & $0.040564\to0.040390$\\
FFHQ/4 & $1$ & $\boldsymbol{-0.26932}$ & $[-0.31001,-0.23475]$ & $0.151296\to0.148716$\\
FFHQ/8 & $1$ & $+0.01241$ & $[+0.00257,+0.02168]$ & $0.081905\to0.081813$\\
FFHQ/16 & $1$ & $+0.00187$ & $[-0.00299,+0.00710]$ & $0.044911\to0.044905$\\
\bottomrule\end{tabular}\end{table}

%% file: supplement_compact/visuals/checkpoint_own.tex
\begin{figure}[!htbp]
\noindent Better empirical integration need not give better FID; the AFHQ refit counterexample in \cref{tab:reference-diversity} also shows the reverse pattern.\par\smallskip
\centering
\includegraphics[width=\linewidth]{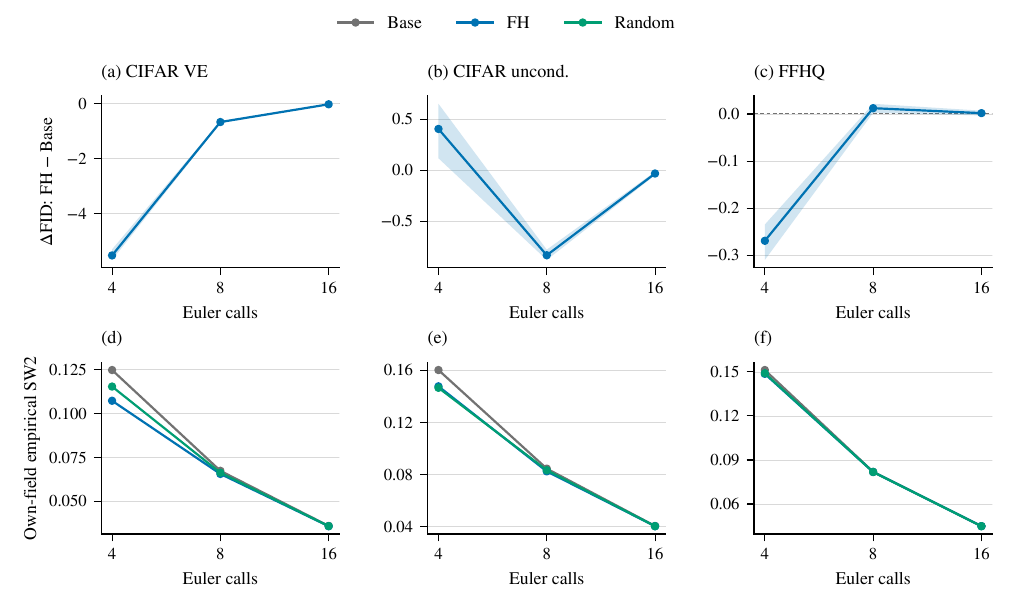}
\caption{All nine fresh-50k checkpoint/NFE cells. Real-data FID contrasts have the registered conditional pointwise 95\% paired intervals; empirical own-field distances have no population interval. Lower own-field integration error can coexist with worse FID, notably at unconditional CIFAR N4.}\label{fig:checkpoint-nfe}
\end{figure}

%% file: supplement_compact/result_tables/reference_diversity.tex
\begin{table}[htbp]\centering\small
\setlength{\tabcolsep}{3pt}\renewcommand{\arraystretch}{1.12}
\caption{Own-error counterexample and all reference/leakage/diversity diagnostic families. Refinement distances are neither independent floors nor certified reference errors, population intervals or equality tests. Non-rejection and finite-map bounds do not establish marginal equality; mixed PRDC changes do not prove absence of mode collapse. No row is highlighted as evidence of a universal benefit.}\label{tab:reference-diversity}
\begin{tabular}{@{}p{.25\linewidth}p{.36\linewidth}p{.30\linewidth}@{}}\toprule
Diagnostic & Recorded result & Scope/uncertainty\\\midrule
AFHQ refit 0 & Own SW2 $.04854216\to.04857189$ & Point only; lower point FID, higher own-error\\
Independent BaseA--BaseB & VE $\approx.005750$; uncond. $\approx.005641$; FFHQ $\approx.005279$ & Same-metric sampling references\\
Coupled Heun 256--512 & SW2 $.000415$--$.000709$ & Same-noise numerical refinement\\
Independent dense kernel test & All Holm-adjusted $p=1$ & Fixed 4096 subsets; 512 permutations; six amplitudes\\
Clipped 128-function mean bound & $\approx.0771$--$.0776$; empirical max $.0102$--$.0107$ & Simultaneous finite-map means, not FID/SW2\\
PRDC $k=5$: eight N4 pipelines & Recall $\uparrow$;precision/density/coverage $\downarrow$ in all eight & Three paired 5k draws;descriptive signs\\
Later diversity supplement & 52 datasets/156 draw records & All conditions retained; no raw-real FFHQ features\\
\bottomrule\end{tabular}\end{table}

%% file: supplement_compact/visuals/diversity.tex
\begin{figure}[!htbp]\centering
\includegraphics[width=\linewidth]{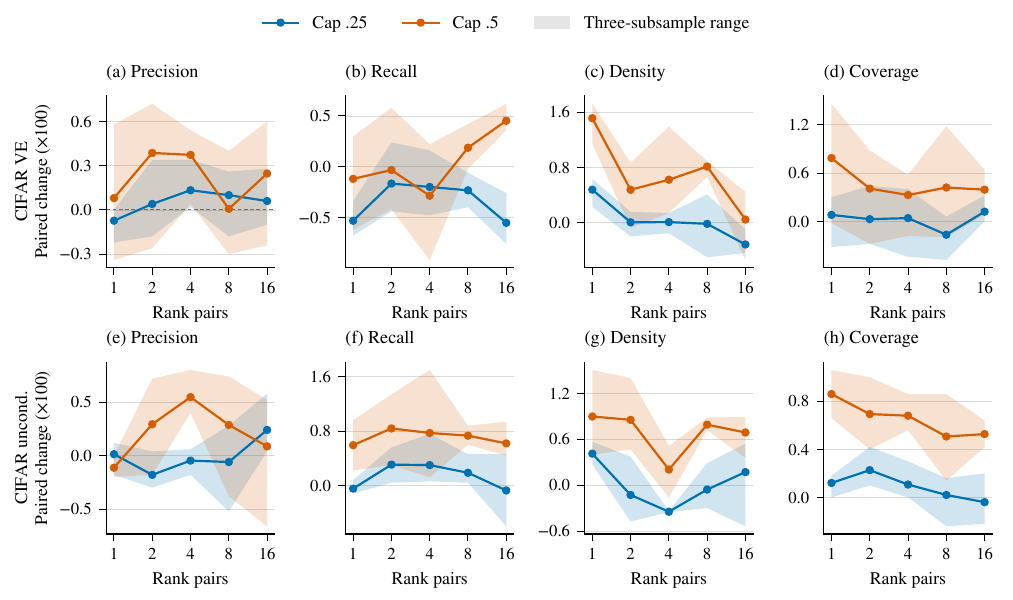}
\caption{All 20 CIFAR rank/cap fields, with PRDC precision, recall, density and coverage. Curves average three paired 5k subsamples and scale changes by 100 (percentage points for precision/recall/coverage, dimensionless density times 100); shading is their descriptive minimum--maximum range, not a confidence interval or independent refit uncertainty. Mixed diversity changes do not prove absence of mode collapse; FFHQ raw-real-feature PRDC is unavailable.}\label{fig:diversity-rank-profiles}
\end{figure}

%% file: supplement_compact/result_tables/temporal_confirmation.tex
\begin{table}[H]\centering\small
\setlength{\tabcolsep}{3pt}\renewcommand{\arraystretch}{1.12}
\caption{All 24 fresh 50k confirmations, fixed before viewing pilot signs. Each cell is condition minus original FH FID and its conditional pointwise 95\% paired interval. Base intervals negate the registered FH--Base intervals; original FH is the zero reference. Bold marks active profile changes with intervals wholly below zero, not a universal temporal-position or familywise claim. Energy-matched profiles and spectrum-preserving orientations are distinct controls.}\label{tab:temporal-confirmation}
\begin{tabular}{@{}p{.14\linewidth}p{.195\linewidth}p{.195\linewidth}p{.195\linewidth}p{.195\linewidth}@{}}\toprule
Condition & CIFAR VE/8 & CIFAR uncond./8 & CIFAR uncond./4 & FFHQ/8\\\midrule
Base & \shortstack{$+0.6866$\\$[+0.6505,+0.7138]$} & \shortstack{$+0.8838$\\$[+0.8404,+0.9389]$} & \shortstack{$-0.5127$\\$[-0.8195,-0.2143]$} & \shortstack{$-0.0107$\\$[-0.0190,-0.0029]$}\\
FH & $0$ (reference) & $0$ (reference) & $0$ (reference) & $0$ (reference)\\
Reverse & \shortstack{$+0.6778$\\$[+0.6501,+0.7047]$} & \shortstack{$+0.8555$\\$[+0.8244,+0.8917]$} & \shortstack{$\boldsymbol{-1.3877}$\\$[-1.6944,-1.0814]$} & \shortstack{$\boldsymbol{-0.0110}$\\$[-0.0209,-0.0004]$}\\
Flat & \shortstack{$+0.6777$\\$[+0.6467,+0.7056]$} & \shortstack{$+0.8580$\\$[+0.8238,+0.9015]$} & \shortstack{$\boldsymbol{-0.7105}$\\$[-1.0159,-0.4280]$} & \shortstack{$\boldsymbol{-0.0079}$\\$[-0.0158,-0.0005]$}\\
Early & \shortstack{$+0.6519$\\$[+0.6249,+0.6829]$} & \shortstack{$+0.8130$\\$[+0.7843,+0.8548]$} & \shortstack{$\boldsymbol{-1.0661}$\\$[-1.3705,-0.7692]$} & \shortstack{$\boldsymbol{-0.0128}$\\$[-0.0218,-0.0031]$}\\
Late & \shortstack{$\boldsymbol{-1.5950}$\\$[-1.6315,-1.5579]$} & \shortstack{$\boldsymbol{-1.5403}$\\$[-1.5815,-1.4895]$} & \shortstack{$+1.3962$\\$[+1.3100,+1.4784]$} & \shortstack{$+0.1014$\\$[+0.0889,+0.1173]$}\\
\bottomrule\end{tabular}\end{table}

%% file: supplement_compact/visuals/confirmation.tex
\begin{figure}[H]
\noindent \Cref{tab:temporal-confirmation} reports all confirmation comparisons to original FH: late helps both CIFAR N8 cells but harms unconditional N4 and FFHQ N8, while reverse/flat/early show the opposite preference (\cref{fig:temporal-confirmation}).\par\smallskip
\centering
\includegraphics[width=\linewidth]{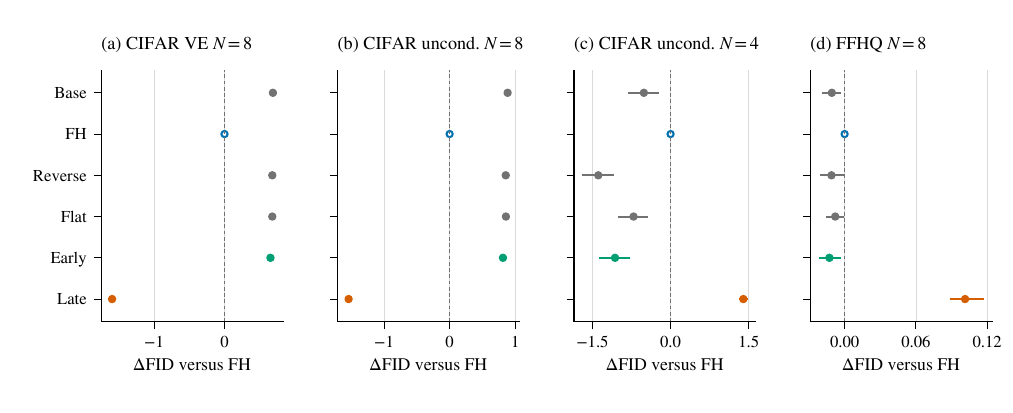}
\caption{All 24 fresh-50k frozen-profile confirmations, with the registered conditional pointwise paired-FID intervals. Control labels are categorical, so they are not joined as a time curve. Late redistribution helps the two CIFAR N8 cells but harms unconditional N4 and FFHQ N8. This frozen-field intervention is separate from native-window retraining in \cref{fig:window-fitting}.}\label{fig:temporal-confirmation}
\end{figure}

%% file: supplement_compact/galleries/temporal_cifar.tex
\begin{figure}[H]\centering
\includegraphics[width=\linewidth]{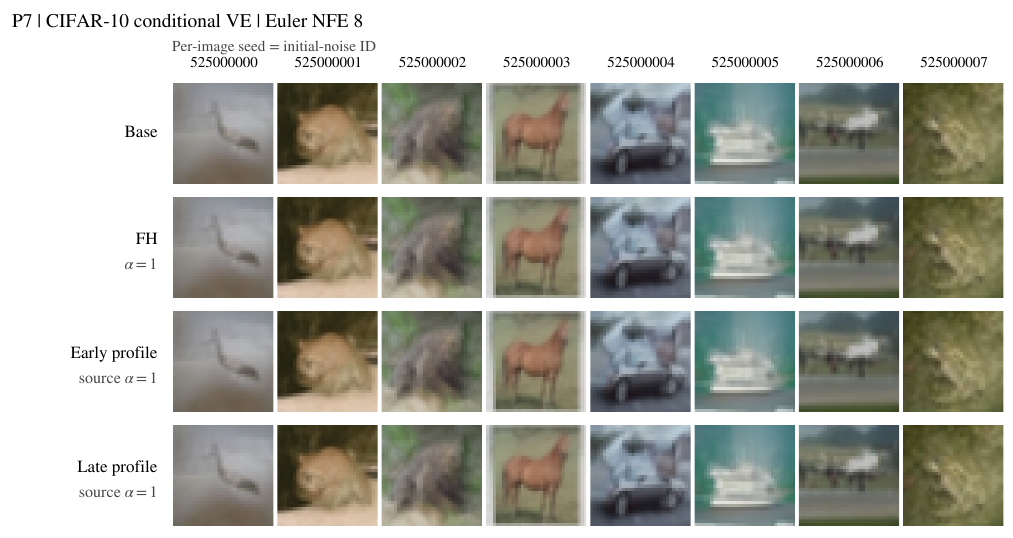}\par\smallskip
\includegraphics[width=\linewidth]{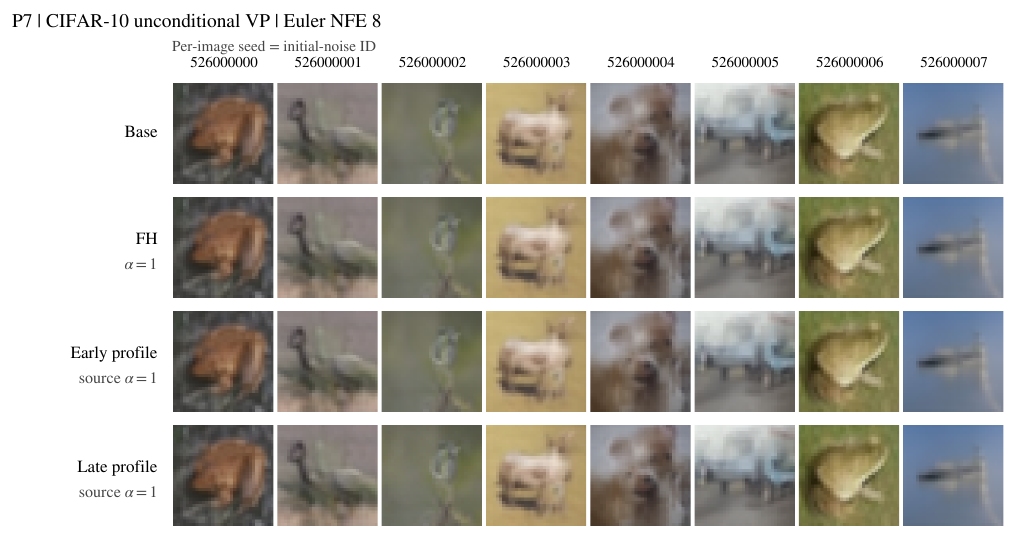}
\caption{Frozen-profile confirmation: conditional VE (top) and unconditional VP (bottom) CIFAR-10, Euler N8. Within each panel, Base/FH/early/late rows share the eight displayed per-image seeds and any conditional labels from the fresh 50k bank. Early/late redistribute the existing fitted profile; source $\alpha$ names its original selected strength, not a new native-window fit.}\label{fig:samples-cifar-eight}
\end{figure}

%% file: supplement_compact/galleries/native_cifar.tex
\begin{figure}[H]\centering
\includegraphics[width=.96\linewidth]{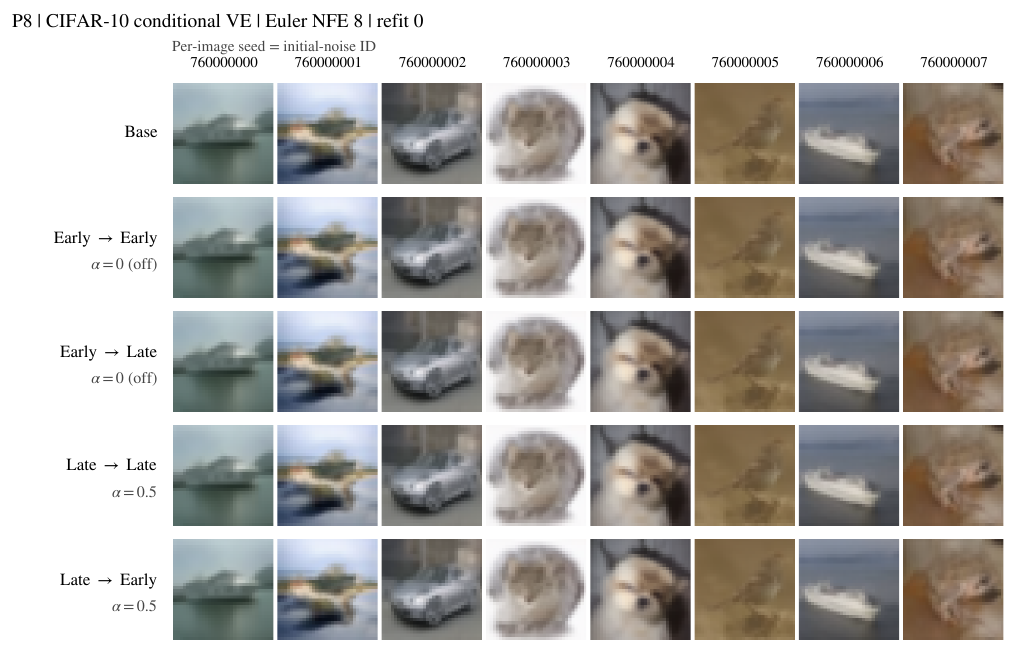}\par\smallskip
\includegraphics[width=.96\linewidth]{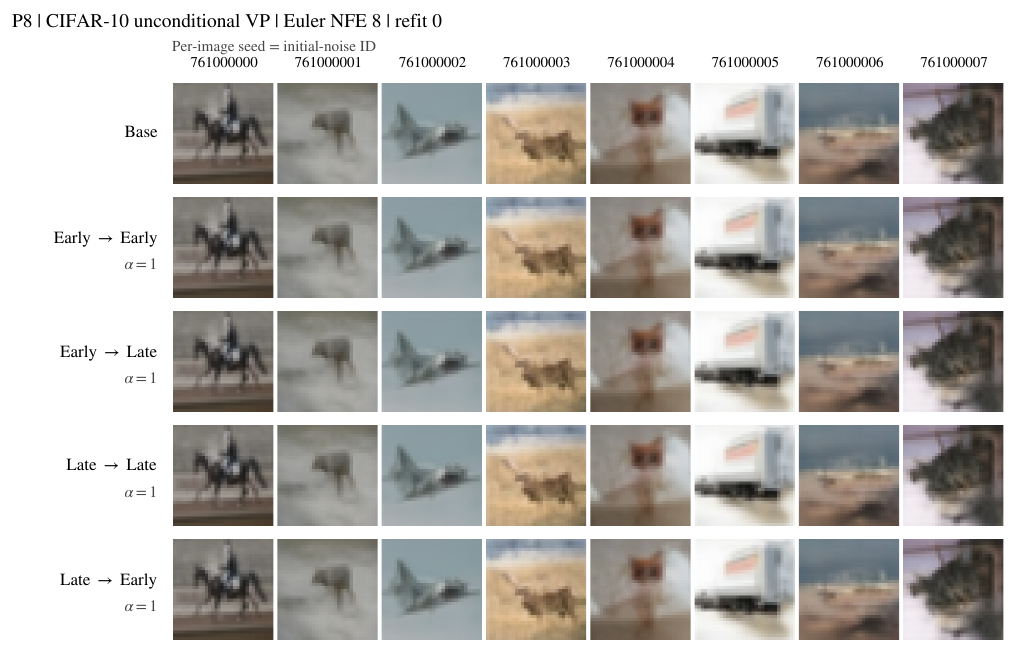}
\caption{Native-window refit 0: conditional VE (top) and unconditional VP (bottom) CIFAR N8. Arrows give fitting window $\to$ receiving window. Both native and transferred rows are retained; transfer inherits the native-selected amplitude without receiver tuning. Refit 0 is chosen by index uniformly across cells, not by appearance or FID.}\label{fig:samples-native-eight}
\end{figure}

%% file: supplement_compact/result_tables/solver_crossfit.tex
\begin{table}[htbp]\centering\small
\setlength{\tabcolsep}{3pt}\renewcommand{\arraystretch}{1.12}
\caption{All 12 fitted-field solver contrasts on common 20k banks. Each cell gives field minus the receiving solver's Base FID, its conditional pointwise 95\% paired-bootstrap interval, and the inherited native-selected amplitude. The six receiving Base rows have $\Delta=0$. Bold denotes active fields with intervals wholly below zero; $\alpha=0$ rows are recorded backend residuals. These are comparisons to receiving Base, not direct transfer--native intervals.}\label{tab:solver-crossfit}
\begin{tabular}{@{}p{.255\linewidth}p{.225\linewidth}p{.225\linewidth}p{.225\linewidth}@{}}\toprule
Fit $\to$ receiver & CIFAR VE & CIFAR uncond. & FFHQ\\\midrule
Euler $\to$ Euler (8 calls) & \shortstack{$\boldsymbol{-0.04242}$\\$[-0.06876,-0.02280]$\\$\alpha=1$} & \shortstack{$\boldsymbol{-0.03922}$\\$[-0.06648,-0.01572]$\\$\alpha=1$} & \shortstack{$-0.00169$\\$[-0.00477,+0.00194]$\\$\alpha=0.125$}\\
Euler $\to$ Heun (15 calls) & \shortstack{$+0.09914$\\$[+0.09076,+0.10550]$\\$\alpha=1$} & \shortstack{$+0.05597$\\$[+0.05007,+0.06246]$\\$\alpha=1$} & \shortstack{$+0.00079$\\$[-0.00129,+0.00369]$\\$\alpha=0.125$}\\
Heun $\to$ Euler (8 calls) & \shortstack{$-0.00050$\\$[-0.00813,+0.00690]$\\$\alpha=0.125$} & \shortstack{$+0.00015$\\$[-0.00015,+0.00045]$\\$\alpha=0$ (off)} & \shortstack{$+0.00213$\\$[-0.00269,+0.00807]$\\$\alpha=0.25$}\\
Heun $\to$ Heun (15 calls) & \shortstack{$-0.00161$\\$[-0.00424,+0.00057]$\\$\alpha=0.125$} & \shortstack{$+0.00000$\\$[-0.00008,+0.00033]$\\$\alpha=0$ (off)} & \shortstack{$\boldsymbol{-0.00311}$\\$[-0.00573,-0.00030]$\\$\alpha=0.25$}\\
\bottomrule\end{tabular}\end{table}

%% file: supplement_compact/visuals/crossfit.tex
\begin{figure}[!htbp]\centering
\includegraphics[width=\linewidth]{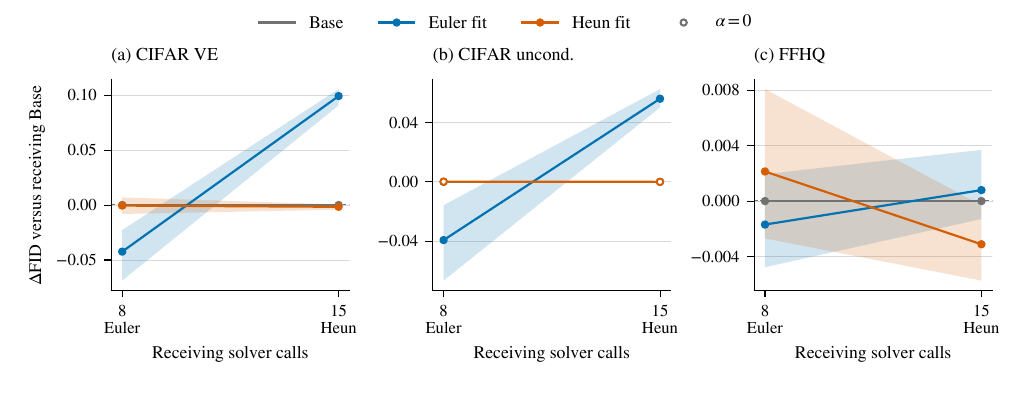}
\caption{All 18 native-fit/receiving-solver conditions, retaining disabled corrections. Existing conditional pointwise 95\% paired intervals compare each field with the receiving solver\textquotesingle s own Base. Transferring an Euler-native correction can worsen Heun. Receiver costs are 8 and 15 network calls, so this matrix does not make an equal-cost claim.}\label{fig:solver-crossfit}
\end{figure}

%% file: supplement_compact/result_tables/rank_cap.tex
\begin{table}[htbp]\centering\small
\setlength{\tabcolsep}{3pt}\renewcommand{\arraystretch}{1.12}
\caption{All 30 rank/cap fields: five fixed rank-pair settings $r=1,2,4,8,16$ per row, not five refit replications. FID ranges retain every point estimate; below/above/contains-zero counts classify conditional pointwise 95\% intervals. The last column counts active fields with smaller point own-SW2, each versus its own Heun 256 reference; no population-SW2 interval is available. Bold marks rows with five active FID improvements and five below-zero intervals. Zero-amplitude residuals remain in the ranges.}\label{tab:rank-cap}
\begin{tabular}{@{}lcp{.245\linewidth}ccc@{}}\toprule
Checkpoint & Cap & $\Delta$FID range & Active/off & \shortstack{Below/above/\\contains zero} & Own lower\\\midrule
CIFAR VE & $0.25$ & $\boldsymbol{[-0.1198,-0.0642]}$ & 5/0 & 5/0/0 & 5/5\\
CIFAR VE & $0.5$ & $\boldsymbol{[-0.7676,-0.6219]}$ & 5/0 & 5/0/0 & 5/5\\
CIFAR uncond. & $0.25$ & $\boldsymbol{[-0.1571,-0.0740]}$ & 5/0 & 5/0/0 & 5/5\\
CIFAR uncond. & $0.5$ & $\boldsymbol{[-0.8748,-0.8260]}$ & 5/0 & 5/0/0 & 5/5\\
FFHQ & $0.25$ & $[-0.0170,+0.0000]$ & 3/2 & 3/0/2 & 3/3\\
FFHQ & $0.5$ & $[-0.0084,+0.0147]$ & 5/0 & 0/0/5 & 5/5\\
\bottomrule\end{tabular}\end{table}

%% file: supplement_compact/result_tables/local_response.tex
\begin{table}[H]\centering\small
\setlength{\tabcolsep}{3pt}\renewcommand{\arraystretch}{1.12}
\caption{Local coupled-state RMS minus Base: counts of 512-replicate conditional pointwise 95\% intervals below/above/including zero. All 84 non-Base contrasts and 28 Base zero-reference rows are retained. Among non-Base rows, 24 intervals are exactly $[0,0]$. Four above-zero late rows have zero correction at the left Euler node but become active inside the dense interpolated interval; they are not globally inactive-field effects.}\label{tab:local-response}
\begin{tabular}{@{}p{.27\linewidth}cp{.16\linewidth}p{.16\linewidth}p{.16\linewidth}@{}}\toprule
Checkpoint/N & Intervals & FH & Frozen early & Frozen late\\\midrule
CIFAR VE/8 & 8 & 0/8/0 & 0/4/4 & 0/5/3\\
CIFAR uncond./8 & 8 & 0/7/1 & 0/4/4 & 0/4/4\\
CIFAR uncond./4 & 4 & 0/3/1 & 0/2/2 & 0/3/1\\
FFHQ/8 & 8 & 0/7/1 & 0/4/4 & 0/4/4\\
\midrule All non-Base & 28 $\times$ 3 & \multicolumn{3}{c}{$0$ below, $55$ above, $29$ include zero (84 total)}\\
\bottomrule\end{tabular}\end{table}

%% file: supplement_compact/result_tables/cost_context.tex
\begin{table}[H]\centering\small
\setlength{\tabcolsep}{3pt}\renewcommand{\arraystretch}{1.12}
\caption{All 18 fixed cost-context contrasts from 24 fresh 20k datasets. Each cell is field--comparator $\Delta$FID and its conditional pointwise 95\% paired-bootstrap interval; parentheses give actual field/comparator calls. Bold marks active below-zero intervals. FH8@16 carries the frozen N8 field to Euler16. The 8-versus-7/9 and 16-versus-15/17 brackets are approximate, not exactly equal cost or measured call savings.}\label{tab:cost-context}
\begin{tabular}{@{}p{.275\linewidth}p{.215\linewidth}p{.215\linewidth}p{.215\linewidth}@{}}\toprule
Field -- comparator (calls) & CIFAR VE & CIFAR uncond. & FFHQ\\\midrule
Euler FH8 -- Base8 (8/8) & \shortstack{$\boldsymbol{-0.7259}$\\$[-0.7904,-0.6797]$} & \shortstack{$\boldsymbol{-0.9069}$\\$[-0.9896,-0.8284]$} & \shortstack{$+0.0165$\\$[+0.0017,+0.0329]$}\\
Euler FH8@16 -- Base16 (16/16) & \shortstack{$-0.0259$\\$[-0.0523,+0.0027]$} & \shortstack{$\boldsymbol{-0.0346}$\\$[-0.0630,-0.0065]$} & \shortstack{$-0.0031$\\$[-0.0104,+0.0035]$}\\
Euler FH8 -- Heun4 (8/7) & \shortstack{$\boldsymbol{-64.6699}$\\$[-65.4678,-64.0654]$} & \shortstack{$\boldsymbol{-71.7132}$\\$[-72.4355,-71.0759]$} & \shortstack{$\boldsymbol{-114.3447}$\\$[-115.2294,-113.5988]$}\\
Euler FH8 -- Heun5 (8/9) & \shortstack{$\boldsymbol{-14.1552}$\\$[-14.7566,-13.7185]$} & \shortstack{$\boldsymbol{-10.0018}$\\$[-10.6906,-9.5506]$} & \shortstack{$\boldsymbol{-28.6022}$\\$[-29.2142,-28.1417]$}\\
Euler FH8@16 -- Heun8 (16/15) & \shortstack{$+3.3165$\\$[+3.1017,+3.5232]$} & \shortstack{$+4.5093$\\$[+4.2581,+4.6729]$} & \shortstack{$+1.3363$\\$[+0.9899,+1.6234]$}\\
Euler FH8@16 -- Heun9 (16/17) & \shortstack{$+4.6772$\\$[+4.4883,+4.8891]$} & \shortstack{$+5.7780$\\$[+5.5436,+5.9581]$} & \shortstack{$+4.4108$\\$[+4.0715,+4.6783]$}\\
\bottomrule\end{tabular}\end{table}

%% file: supplement_compact/result_tables/rf_pilot.tex
\begin{table}[htbp]\centering\small
\setlength{\tabcolsep}{3pt}\renewcommand{\arraystretch}{1.12}
\caption{Original independent-path RF pilot: all two held-out 10k FID contrasts and four own-SW2 point measurements. FID intervals are conditional pointwise 95\% paired bootstraps; both include zero. Own distances compare each field to its own Heun 256 reference, without an independent floor, population-SW2 interval or equivalence inference. No entry is bolded as an established FID improvement.}\label{tab:rf-pilot}
\begin{tabular}{@{}cccp{.25\linewidth}p{.33\linewidth}@{}}\toprule
N & $\alpha$ & $\Delta$FID & Conditional 95\% & Own SW2: Base $\to$ FH\\\midrule
8 & $0.75$ & $-0.00912$ & $[-0.03317,+0.01396]$ & $0.05707995 \to 0.05704312$\\
16 & $0.5$ & $-0.00836$ & $[-0.02428,+0.00342]$ & $0.03542936 \to 0.03543288$\\
\bottomrule\end{tabular}\end{table}

%% file: supplement_compact/visuals/rf.tex
\begin{figure}[!htbp]\centering
\includegraphics[width=\linewidth]{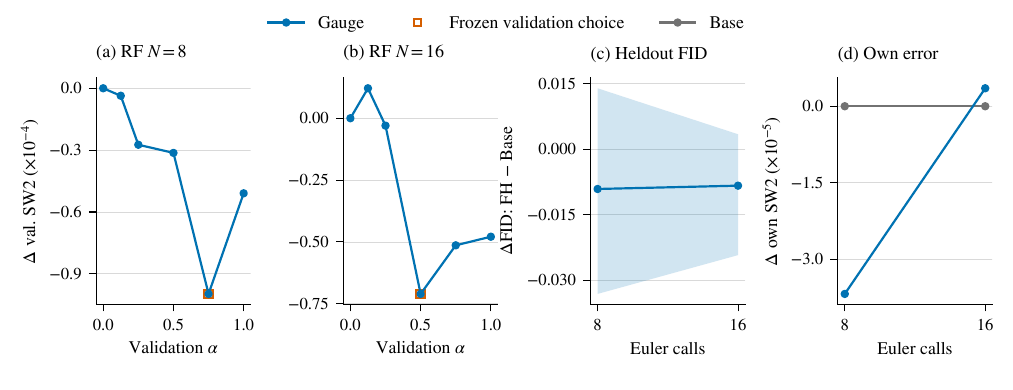}
\caption{Original independent-Gaussian RF pilot: all validation amplitudes, the frozen choices, both fresh-10k FID contrasts and the two own-field comparisons. The last panel shows $(\mathrm{SW2}_{\rm FH,own}-\mathrm{SW2}_{\rm Base,own})\times10^5$; all four original values remain in the scalar record. FID shading is the existing conditional pointwise 95\% paired interval, and both intervals span zero. Validation and own-SW2 points have no population interval. One checkpoint and two fits do not establish family-wide robustness or neural marginal preservation.}\label{fig:RF-pilot-response}
\end{figure}

%% file: supplement_compact/galleries/rf.tex
\begin{figure}[H]\centering
\includegraphics[width=\linewidth]{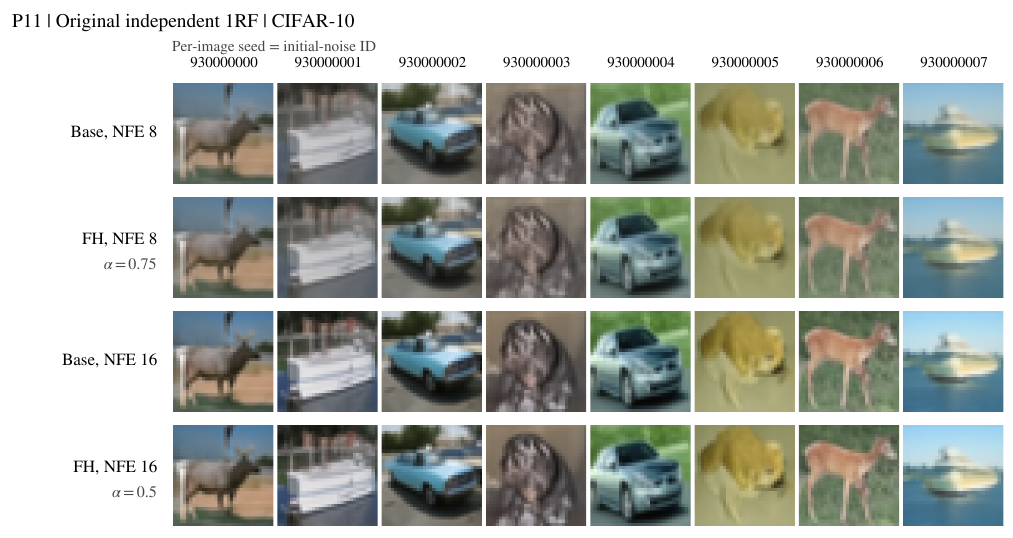}
\caption{The original independently coupled RF checkpoint, Base/FH at N8/N16, with validation-selected strengths .75/.5 and the first eight fresh-10k seeds. These matched examples accompany the null paired-FID result in \cref{app:neural-rf}; visual similarity does not establish marginal equivalence.}\label{fig:samples-rf}\label{fig:samples-rf-digits}
\end{figure}

%% file: supplement_compact/provenance.tex
\section{Attribution and scope}\label{app:novelty-boundary}
\paragraph{Continuous marginal freedom is inherited.}
The weighted gradient/solenoidal geometry is classical in Wasserstein analysis \citep{ambrosio2008gradient}. For diffusion marginals, we use the formulation of \citet{horvat2024gauge} and observable score-error bounds of \citet{khelifa2026gradients}. Antisymmetric invariant-law currents appear in \citet{ma2015recipe}. EDDY's revised Stein-operator analysis conditions its no-bias claim on batch independence and charges bias to dependence \citep{vinograd2026eddy}.

\paragraph{Practical skew-score sampling is also inherited.}
Preconditioned diffusion already tests skew-score terms $M_{\rm sk}\nabla\log\rho$, $M_{\rm sk}^\top=-M_{\rm sk}$, on pretrained models \citep{ma2022preconditioned,ma2025preconditioned}. ECCV Section~5/supplement and journal Section~8.4 report no clear acceleration from these solenoidal ablations, distinct from matrix-preconditioner gains. We fit low-rank temporal fields to a finite-step distributional proxy and test solver/window transfer.

\paragraph{Other uses of field freedom.}
Flux Matching fixes a time-dependent marginal path by matching reference weighted-flux divergence (their Appendix~E, Eqs.~(53)--(56)); their Appendix~C.4 proposes divergence-free score-model augmentation and their Appendix~C.2 discusses field-dependent finite-step behavior at constrained boundaries \citep{pao2026flux}.

Helmholtz-SDE optimizes marginal-compatible stochastic paths for variational inference \citep{smith2026helmholtz}; Gauge Flow Models learn a bundle architecture rather than impose our weighted-solenoidal constraint \citep{strunk2025gauge}. \Citet{vuong2025score} study learned diffusion fields as gradient flows.

\paragraph{Distributional numerical error is not a new subject.}
Finite-time weak expansions \citep{talay1990expansion}, weak modified equations and higher-order invariant-measure sampling \citep{debussche2012weak,abdulle2014invariant,laurent2022invariant,bronasco2025exotic} are classical. \Citet{hurault2026geometry} derive exact-score Euler--Maruyama weak/Fr\'echet errors, including the ODE limit; \citet{pierret2025gaussian} compute Gaussian diffusion solutions and Wasserstein discretization errors. Stationary versus time-dependent analysis alone is not our distinction.

\paragraph{Irreversible dynamics and numerical preservation.}
\citet{lelievre2013optimal} optimize Gaussian-invariant nonreversible linear drifts for convergence to equilibrium; \citet{duncan2017nonreversible} analyze splitting bias and asymptotic variance. Continuous-target preservation therefore need not preserve discrete sampling. The B-series volume-preservation impossibility, apart from the exact-flow series, quantifies over all divergence-free fields \citep{iserles2007volume}; our reachability question instead selects a field for one specified law/path.

\paragraph{The closest Gaussian finite-step comparison.}
\citet{claustre2026gaussian} study commuting Gaussian schedules, a necessary Euler drift bound and stable-convex-path exactness. Changing their variance path changes drift; factorization at a fixed path changes regression variance without changing drift. We vary skew motion within a prescribed affine fixed-shape path, characterizing its numerical orbit, target-law intersection and nearest errors under rank/norm budgets. Bures--Wasserstein identities are standard \citep{gelbrich1990formula,bhatia2019bures}.

\paragraph{Other numerical design variables.}
Transport, rectification and minibatch coupling change transport geometry \citep{khrulkov2023ddpm,lavenant2022flow,liu2023rectified,pooladian2023multisample,bansal2026straightness}. Strain/vorticity analysis links Jacobians to truncation error/stability without our fixed-path constraint \citep{tao2026strain}. Scheduling, learned timesteps and bespoke/differentiable solvers optimize complementary variables \citep{gao2023backward,sabour2024align,xue2024optimized,shaul2024bespoke,tong2025ld3,wang2025solver,li2026catflow,huang2026art}; our nearby-call Heun context neither benchmarks all of them nor claims an optimal solver.

\paragraph{Incremental contribution.}
The leading projected Euler objective is a baseline. Our contribution is the set of sufficient marginal-order conditions, affine orbit--slice exactness results, and constrained spectral projection formulas established under the stated hypotheses. We do not claim arbitrary-path solvability or priority over all related formulations.